%% file: main.tex
\documentclass[twoside,11pt]{article}

\usepackage[abbrvbib, nohyperref,preprint]{jmlr2e}

\usepackage{url} 
\usepackage{enumitem}% simple URL typesetting
\usepackage{booktabs}       % professional-quality tables
\usepackage{amssymb,amsmath,amsfonts}    % blackboard math symbolhs
\usepackage{mathtools}
\usepackage{nicefrac}       % compact symbols for 1/2, etc.

\usepackage{bm}
\usepackage{mathtools}
\usepackage[ruled, lined, linesnumbered, commentsnumbered, longend]{algorithm2e}
\usepackage{multirow}
\usepackage{physics}
\usepackage{bbding} 
\usepackage{array}
\usepackage{pifont}
\usepackage{fancyhdr}
\usepackage{amssymb}% http://ctan.org/pkg/amssymb
\usepackage{pifont}
\usepackage{xr}
\usepackage{wrapfig}
\usepackage{placeins}
\usepackage{subfigure}
\usepackage{setspace}
\usepackage{exscale}
\usepackage{relsize}
\usepackage{ragged2e}
\usepackage{microtype}      % microtypography
\usepackage{makecell}

\usepackage[table,xcdraw]{xcolor}
\definecolor{oxfordblue}{rgb}{0.0, 0.13, 0.28} % 深牛津蓝
\definecolor{carnelian}{rgb}{0.7, 0.11, 0.11}  % 深红色
\definecolor{darkslategray}{rgb}{0.18, 0.31, 0.31} % 暗石板灰
\definecolor{skyblue}{rgb}{0.53, 0.81, 0.92}    % 天蓝色
\definecolor{cadetblue}{rgb}{0.37, 0.62, 0.63} % 军校蓝
\definecolor{carnelian}{rgb}{0.7, 0.11, 0.11}  % 栗色

\usepackage[
    colorlinks=true,
    linkcolor=carnelian,        % 内部链接（如章节、公式、图表）使用军校蓝
    citecolor=carnelian,        % 引用链接使用栗色（carnelian）
    urlcolor=skyblue,     % 网址链接如果存在，仍使用暗石板灰
    bookmarks=true,
    bookmarksopen=true,
    bookmarksnumbered=true,
    bookmarkstype=toc
]{hyperref}
\usepackage[T1]{fontenc}
\usepackage{fix-cm}

\newcolumntype{P}[1]{>{\centering\arraybackslash}p{#1}}
\newtheorem{mytheorem}{Theorem}
\newtheorem{mylemma}{Lemma}

\newtheorem{myassumption}{Assumption}

\newtheorem{myprop}{Proposition}

\DeclareMathOperator{\E}{\mathop{\mathbb{E}}}
\DeclareMathOperator*{\argmin}{argmin}

\def\*#1{\mathbf{#1}}
\def\$#1{\mathcal{#1}}
\def\bb#1{\mathbb{#1}}

\def\pinv#1{\*{#1}^{\dagger}}

\newcommand{\blue}[1]{#1}

\newcommand{\inner}[1]{\left\langle #1 \right\rangle}
\newcommand{\cA}{{\mathcal{A}}}

\newcommand{\cE}{{\mathcal{E}}}
\newcommand{\bbP}{{\mathbb{P}}}

\newcommand{\myvec}{{\rm vec}}

\newcommand{\abbr}{{Opt-Laws}}
\newcommand{\ito}{{It\^{o}}}

\newcommand{\floor}[1]{\lfloor#1\rfloor}

\makeatletter
\newcommand*{\addFileDependency}[1]{% argument=file name and extension
  \typeout{(#1)}
  \@addtofilelist{#1}
  \IfFileExists{#1}{}{\typeout{No file #1.}}
}
\newcommand{\algorithmfootnote}[2][\footnotesize]{%
  \let\old@algocf@finish\@algocf@finish% Store algorithm finish macro
  \def\@algocf@finish{\old@algocf@finish% Update finish macro to insert "footnote"
    \leavevmode\rlap{\begin{minipage}{\linewidth}
    #1#2
    \end{minipage}}%
  }%
}

\usepackage{lastpage}
\jmlrheading{xx}{2024}{1-\pageref{LastPage}}{x/x; Revised x/x}{xx/xx}{xxx}{xxx}

\ShortHeadings{Opt-Laws}{Opt-Laws}
\firstpageno{1}

\begin{document}

\title{Optimization Hyper-parameter Laws for Large Language Models}

\author{\name Xingyu Xie \email xyxie@nus.edu.sg \\
        \addr Department of Mathematics,
       National University of Singapore, Singapore.
       \AND
       \name Kuangyu Ding \email kuangyud@u.nus.edu \\
       \addr Department of Mathematics,
       National University of
Singapore, Singapore.
\AND
Shuicheng Yan \email shuicheng.yan@gmail.com \\
\addr School of Computing, National University of Singapore.
\AND
Kim-Chuan Toh \email mattohkc@nus.edu.sg \\
\addr Department of Mathematics and Institute of Operations Research and Analytics, 
National University of
Singapore, Singapore.
\AND
Tianwen Wei \email wei.lille1@gmail.com \\
\addr Skywork AI, Beijing.
}

\editor{editor}

\maketitle

\begin{abstract}%   <- trailing '%' for backward compatibility of .sty file 
Large Language Models have driven significant AI advancements, yet their training is resource-intensive and highly sensitive to hyper-parameter selection. While scaling laws provide valuable guidance on model size and data requirements, they fall short in choosing dynamic hyper-parameters, such as learning-rate (LR) schedules, that evolve during training. To bridge this gap, we present Optimization Hyper-parameter Laws (\abbr{}), 
\blue{a framework that predicts final training loss as a function of LR schedule, model size, and data size. Grounded in SDE-based convergence and escape analyses, \abbr{} yield interpretable convergence and escape features that predict final training loss across model scales, enabling schedule pre-selection from small-scale experiments. Empirically, \abbr{} achieve a 94\% Top-2 hit rate for identifying near-optimal schedule candidates on held-out configurations, correctly identify the best-performing schedule family in all five evaluated out-of-family settings, and detect training divergence with F1\,=\,0.92.}
\end{abstract}


\section{Introduction}

Large Language Models (LLMs) have emerged as a leading paradigm in artificial intelligence, yet their training processes impose significant demands on computational resources and energy~\citep{achiam2023gpt,dubey2024llama,deepseekai2024deepseekv2,xie2026slow}. Given the immense scale and associated costs, training these models is typically a one-off endeavor, where crucial hyper-parameters, such as peak learning rate (LR), warmup steps, and LR schedules, must be predetermined and remain fixed throughout the training process. However, determining the optimal values for these hyper-parameters prior to experimentation is often challenging. Inadequate hyper-parameter selection can severely compromise the training process, potentially leading to failure and the consequent waste of computational resources and financial investment.

To address the complexities of hyper-parameter selection in large-scale models, researchers have increasingly relied on \textbf{Scaling Laws}~\citep{kaplan2020scaling,hoffmann2022training,achiam2023gpt}. These scaling laws, derived from extensive empirical studies on small-scale models and datasets, provide a practical framework for predicting the relationships between certain hyper-parameters, such as model size and data volume, and training outcomes. By fitting these observed relationships to the power-law formula, scaling laws offer a heuristic approach for extrapolating the expected performance of large-scale models. This enables the selection of model size and data requirements prior to full-scale training, thereby reducing the likelihood of inefficient resource utilization.

While existing research on scaling laws focuses on the relationship between model scale, dataset size and model performance, few studies have explored the scaling law on training hyper-parameters, particularly time-dependent ones like the learning rate schedule. These hyper-parameters, including peak learning rate, are crucial for optimal model performance. Consequently, practitioners often rely on heuristic methods that frequently fall short in addressing the complex challenges of diverse training scenarios.

This challenge extends beyond pre-training. Fine-tuning or continual training of advanced foundation models, such as LLaMA3~\citep{dubey2024llama} or DeepSeek-V2~\citep{deepseekai2024deepseekv2}, encounters similar difficulties. Although fine-tuning is less computationally expensive than pre-training, the distribution shift between the fine-tuning data and the original training data complicates the selection of optimal hyper-parameters. Furthermore, since the model size is fixed during fine-tuning, scaling laws are not easily applicable. Consequently, the identification of optimal training hyper-parameters, particularly peak LR and warmup steps, typically requires iterative cycles of fine-tuning and evaluation. These hyper-parameters play a crucial role in balancing the retention of the foundation model's inherent capabilities with its adaptation to new data~\citep{ibrahim2024simple}.

To address the challenge of hyper-parameter selection in large-scale LLM training, we propose a novel approach termed Optimization Hyper-parameter Laws, \textbf{\abbr{}}. \abbr{} exploits data gathered from small-scale models and datasets to establish mathematical relationships between training hyper-parameters and final training loss. This enables the pre-selection of suitable hyper-parameter configurations, such as the LR schedule, warmup steps, and peak LR, before commencing large-scale model training or fine-tuning.
We first apply stochastic differential equations (SDEs) to model the training dynamics of prevalent first-order optimization algorithms, such as SGD and Adam~\citep{kingma2014adam}, within a continuous-time framework, followed by a detailed analysis of convergence rates and the probability of escaping local minima. This analysis yields \blue{a statistical vector of 15 feature terms (plus one intercept)} encapsulating key training hyper-parameters, which is subsequently used in a linear regression model to obtain \abbr{}. Notably, to the best of our knowledge, this is among the first works to utilize SDEs in establishing convergence rates of gradient-based methods for general non-convex optimization and to apply time-inhomogeneous SDEs for evaluating escape probabilities in non-convex settings.

\blue{In addition to loss prediction, the law provides qualitative interpretations of observed training patterns, such as the diminishing effect of schedule choice at large data scales and the dependence of optimal peak LR on model size. We validate \abbr{} on MoE models with up to 4 billion trainable parameters and over 450 billion training tokens, covering pre-training, continual training, and fine-tuning.} The key contributions are summarized as follows:
\begin{itemize}
    \item \blue{We propose \abbr{}, a framework fitted on small-scale experiments that predicts final training loss for large-scale LLM training across model sizes, data budgets, and LR schedules. This allows pre-selection of LR schedules for pre-training, continual training, and fine-tuning before committing to full-scale runs.}
    \item \blue{We derive SDE-based convergence and escape bounds for both SGD and Adam under general time-varying learning rates in non-convex settings. Both optimizers yield similar functional forms, motivating a unified feature construction that decomposes the schedule's effect into convergence speed and escape dynamics.}
    \item \blue{We conduct extensive experiments on held-out prediction, divergence detection, out-of-family schedule generalization, and LR schedule ranking in continual training and fine-tuning. The generalized \abbr{} selects a near-optimal schedule in 94\% of held-out evaluations, detects divergence with F1${=}$0.92, and correctly identifies the best schedule family in all five out-of-family groups.}
\end{itemize}

\blue{\paragraph{Roadmap.} Sec.~2 reviews related work. Sec.~3 introduces the simplified \abbr{} at a fixed model size. Sec.~4 extends it to the generalized \abbr{} across model scales and reports the primary held-out benchmark. Sec.~5 evaluates cross-scale extrapolation, out-of-family schedule generalization, and ranking in continual training and fine-tuning. Sec.~6 provides the SDE-based theoretical justification for the feature construction.}

\section{Related Work}
\subsection{Scaling Laws}
The study of scaling laws in LLMs has been essential for understanding how model performance scales with increases in model size, data volume, and computational resources. While early research attempted to model multilayer neural network performance using power laws~\citep{rosenfeld2020a}, \cite{kaplan2020scaling} were the first to apply this approach systematically to LLM training. This foundational work had led to several variations, including the Chinchilla law~\citep{hoffmann2022training,besiroglu2024chinchilla}, Mosaic law~\citep{sardana2024beyond}, and models from DeepSeek~\citep{bi2024deepseek} and MiniCPM~\citep{hu2024minicpm}, all of which primarily utilize an empirical power law that defines the relationship between training loss, model size $N$, and data size $D$ in LLMs:
\[
\text{Loss} = \underbrace{\frac{A_1}{N^{\kappa_1}}}_{\text{Model-size dependent}} + \underbrace{\frac{A_2}{D^{\kappa_2}}}_{\text{Data-size dependent}} + \underbrace{A_3}_{\text{Irreducible}},
\]
where $\kappa_1, \kappa_2 > 0$ represent the scaling exponents, $A_1, A_2 > 0$ are coefficients, and $A_3 \in \mathbb{R}$ denotes the irreducible loss component.
Recent studies~\citep{muennighoff2024scaling,goyal2024scaling} have highlighted the impact of data repetition and quality on scaling behavior, suggesting the need for more frequent updates to existing scaling laws. 
Some researchers have also questioned the sufficiency of power law models, advocating for more complex parameterizations to better capture the relationship between model size and data volume~\citep{hernandez2021scaling,caballero2023broken}. Other studies have extended power laws to model the relationship between training loss and individual hyper-parameters, such as batch size~\citep{deepseekai2024deepseekv2}. 
Additionally, scaling laws have been applied in downstream task losses~\citep{dubey2024llama} or later stages of model alignment, such as fine-tuning~\citep{isik2024scaling} and RLHF~\citep{gao2023scaling}.
%, and there is growing interest in modeling downstream task losses relative to training FLOPs, which could generalize scaling laws across models. However, these approaches remain largely heuristic and lack rigorous mathematical foundations.
\blue{Recent works have begun incorporating LR schedules into scaling-law-style predictors. \citet{tissue2025scaling} combine cumulative learning-rate and annealing-area features with a model-size interaction to predict per-step loss, while \citet{luo2025multipower} develop a multi-power-law approach for the same task. \abbr{} differs in both construction and scope: the feature family is derived from SDE-based convergence and escape analyses rather than empirical curve fitting, the prediction target is the final training loss across model scales rather than the per-step loss curve, and the framework extends to continual training and fine-tuning through history-aware modifications.}

\subsection{Convergence Analysis via Dynamical Systems}
Dynamical systems are powerful tools in the convergence analysis of optimization algorithms. Over the past decade, numerous studies such as \cite{su2014differential,may2017asymptotic,attouch2018fast,muehlebach2019dynamical,attouch2024fast} have utilized ordinary differential equations (ODEs) to analyze algorithmic convergence properties, particularly employing high-order (with order $\geq2$) ODEs to intuitively interpret the mystery behind momentum acceleration, such as Polyak's heavy ball method \citep{polyak1987introduction} and Nesterov's acceleration \citep{nesterov1983method} in smooth convex optimization problems. In the realm of non-smooth non-convex optimization, differential inclusions, which are more general than ODEs, have recently been applied to establish the convergence of subgradient methods \citep{duchi2018stochastic,davis2020stochastic,xiao2024adam,ding2024stochastic}. In addition to ODE-based methods, recent studies~\citep{gess2023convergence,maulen2022sde,maulen2024sde} have started using SDEs to derive convergence rates for stochastic gradient methods in convex optimization problems or non-convex optimization problem under Polyak-\L{}ojasiewicz (PL) condition. However, to the best of our knowledge, no SDE-based approach has been established for determining convergence rates in the general non-convex settings. \emph{To the best of our knowledge, our work is among the first to apply SDEs in establishing convergence rates for general non-convex problems, with our derived bounds applicable to any general learning rate policy, without relying on specific patterns such as constant, sublinearly decaying, or linearly decaying learning rates.}
\blue{On the optimization side, \citet{li2021expcos} and \citet{wang2021stepdecay} are relevant to our schedule-dependent convergence discussion, while \citet{bergsma2025straightzero} is particularly relevant to our discussion of schedules that decay to zero and to the limitations of schedules with non-zero final learning rates.}

\subsection{Escaping analysis via SDEs}
A closely related research area involves analyzing SGD's escape dynamics using SDEs~\citep{nguyen2019first,xie2020diffusion,mori2022power,ibayashi2023does,battash2024revisiting}. These work focus on estimating the exit time, which measures how quickly an SGD sequence can move from a sharp local minimum to a flatter region, potentially improving generalization performance~\citep{keskar2016large,foret2020sharpness}. The classic Eyring-Kramers law~\citep{kramers1940brownian,berglund2013kramers}, based on large deviation theory, is typically used for such exit time estimations. However, this approach is not directly applicable in our case due to the time-inhomogeneity of the SDEs, which arises from the time-dependent learning rate schedule. To address this challenge, we employ a Gaussian approximation to make the escape analysis more tractable under mild conditions.

\section{Optimization Hyper-parameter Laws}\label{sec:opt-laws-main}
\blue{This section introduces the simplified \abbr{} at a fixed model size, capturing the relationship between LR schedule and final training loss without cross-scale variation. The generalized multi-scale extension and held-out evaluation are presented in Sec.~\ref{sec:general}.}
\subsection{\abbr{} with Fixed Model Size}\label{sec:opt-law1}
We \blue{propose} the \abbr{} in Eqn.~\eqref{eq:optlaw} to predict the training loss for LLMs. \blue{Because the training epoch for both pre-training and continual training is typically set to one, the training loss closely approximates the validation loss.}
In this section, we fix the model size and consider the LR schedule $\eta(t) \geq 0$ for all $t \in [0, S]$, where $\eta_{\operatorname{max}}$ denotes the peak LR, and $S$ represents the total training steps (proportional to data size $D$ such that $D = S \times \text{Batch size} \times \text{Token length}$). The function $\eta(t)$ is defined as $\eta(\Delta_t k) = \eta_k$, where $\eta_k$ is the LR used at the $k$-th training step, and $\Delta_t$ is a sufficiently small time step. This mapping effectively links each step-index to its corresponding LR.

For simplicity, we set the initial and final values of $\eta_t$ to zero, i.e., $\eta(0) = \eta(S) = 0$, though these values can be adjusted based on practical requirements. Under these conditions, the training loss and training hyper-parameters adhere closely to the following relationship:
\begin{equation}\label{eq:optlaw}
\begin{aligned}
   \log\qty(\text{Loss})  =   &\underbrace{c_1\qty(\int_0^a \!\!\!\eta(s) \dd s)^{-\alpha_1} \!\!\!\! + c_2\qty(\int_a^S \!\!\! \eta(s) \dd s)^{-\alpha_2}}_{\text{convergence speed}} + \underbrace{\frac{c_3}{S} + b}_{\text{bias}} + \\
   &\underbrace{c_4 \qty(\int_0^a \!\! \qty(\eta^\prime(s))^2 \dd s)^{\alpha_3} + c_5 \qty(\int_a^S \!\! \qty(\eta^\prime(s))^2 \dd s)^{\alpha_4}}_{\text{ease of escaping local region}} \\
\end{aligned}
\end{equation}
\blue{where $a$ denotes the warm-up duration (denoted by $a_1$ in the later generalized multi-phase notation),} $\eta^\prime(s)$ is the time derivative of $\eta(s)$, the constants $c_1, c_2, c_3, c_4, c_5, \alpha_1, \alpha_2, \alpha_3, \alpha_4>0$, and $b\in \bb R$ are all dependents on the model and training data. 
Analogous to scaling laws, \abbr{} is first parameterized by fitting its constants and power terms on a small-scale model with a similar network architecture and a small subset of the dataset. The parameterized \abbr{} can then be applied to choosing hyper-parameters in large-scale training. The details on the parameters fitting are clarified in Secs.~\ref{sec:fit1} and~\ref{sec:apply-opt-laws}.

The first component in Eqn.~\eqref{eq:optlaw} addresses the convergence speed, elucidating the impact of the LR schedule $\eta(t)$ on optimization convergence. Given the same computational budget, training dynamics with faster convergence tend to achieve lower training loss, thereby enhancing overall efficiency. Notably, the loss is inversely proportional to $\int_a^S \eta_t$, suggesting that, within certain bounds, a higher peak LR can be beneficial~\citep{xie2024adan}. However, an excessively high peak LR can also increase the integral $\int_a^S {\eta_t^\prime}^2$, indicating the need to carefully balance the benefits of a high peak LR against the potential drawbacks.
{Note that for convenience, we use notation such as $\int_a^S {\eta_t^\prime}^2$
to denote $\int_a^S (\eta^\prime(t))^2dt$ in the previous sentence and other parts of this paper.}

The escape terms in Eqn.~\eqref{eq:optlaw} reflect the schedule's influence on basin selection. A larger $\int (\eta'_t)^2$ increases the trapping probability near the current basin, so overly rapid warmup or cooldown can lock the optimizer into a suboptimal region. This is consistent with LLM training practice, where excessively aggressive schedule transitions degrade performance~\citep{hu2024minicpm}.
\begin{figure}[t]
    \centering
    \begin{minipage}{\textwidth}
        \subfigure[]{
            \includegraphics[width=0.48\linewidth]{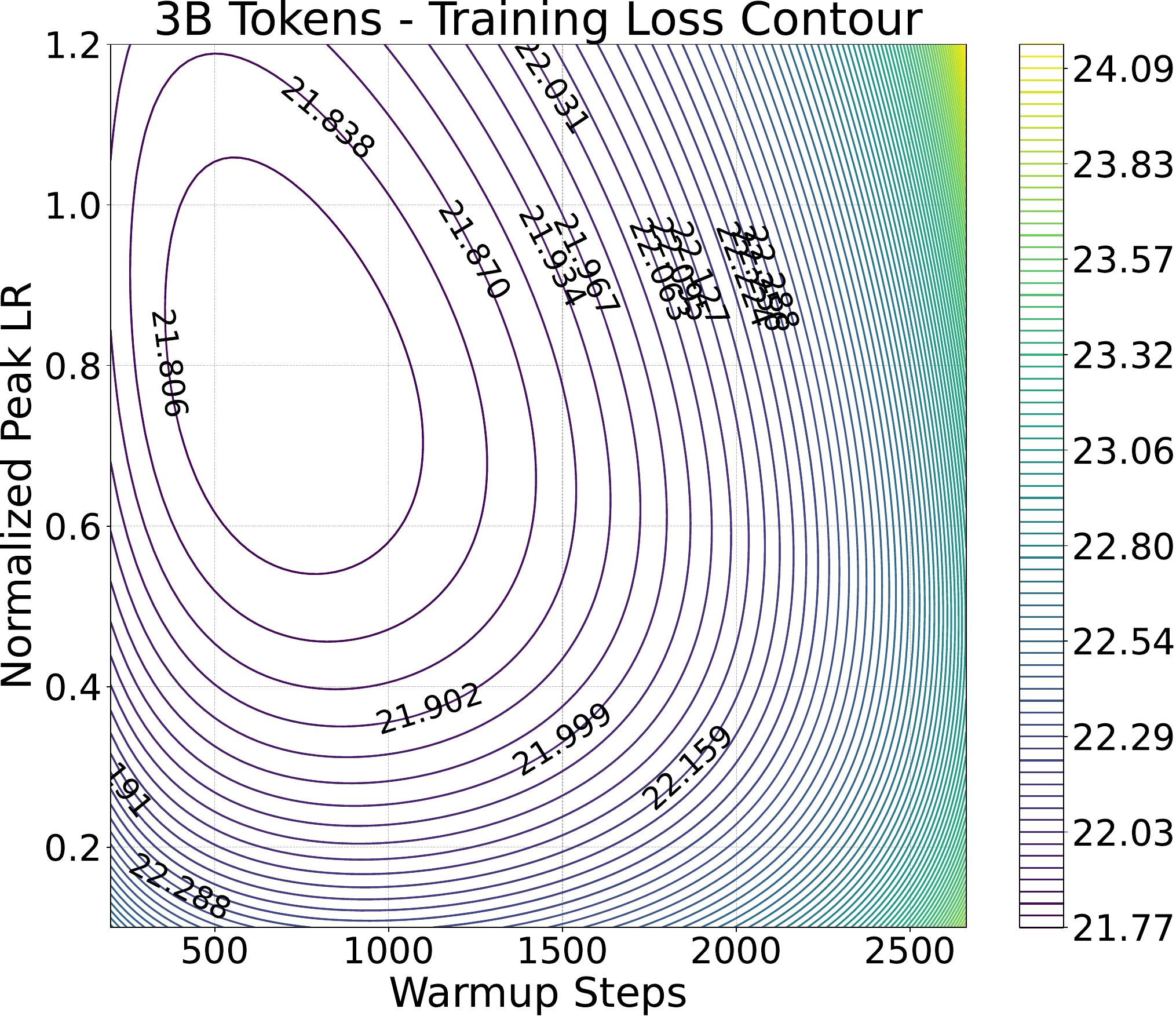}
        }
        \subfigure[]{
            \includegraphics[width=0.48\linewidth]{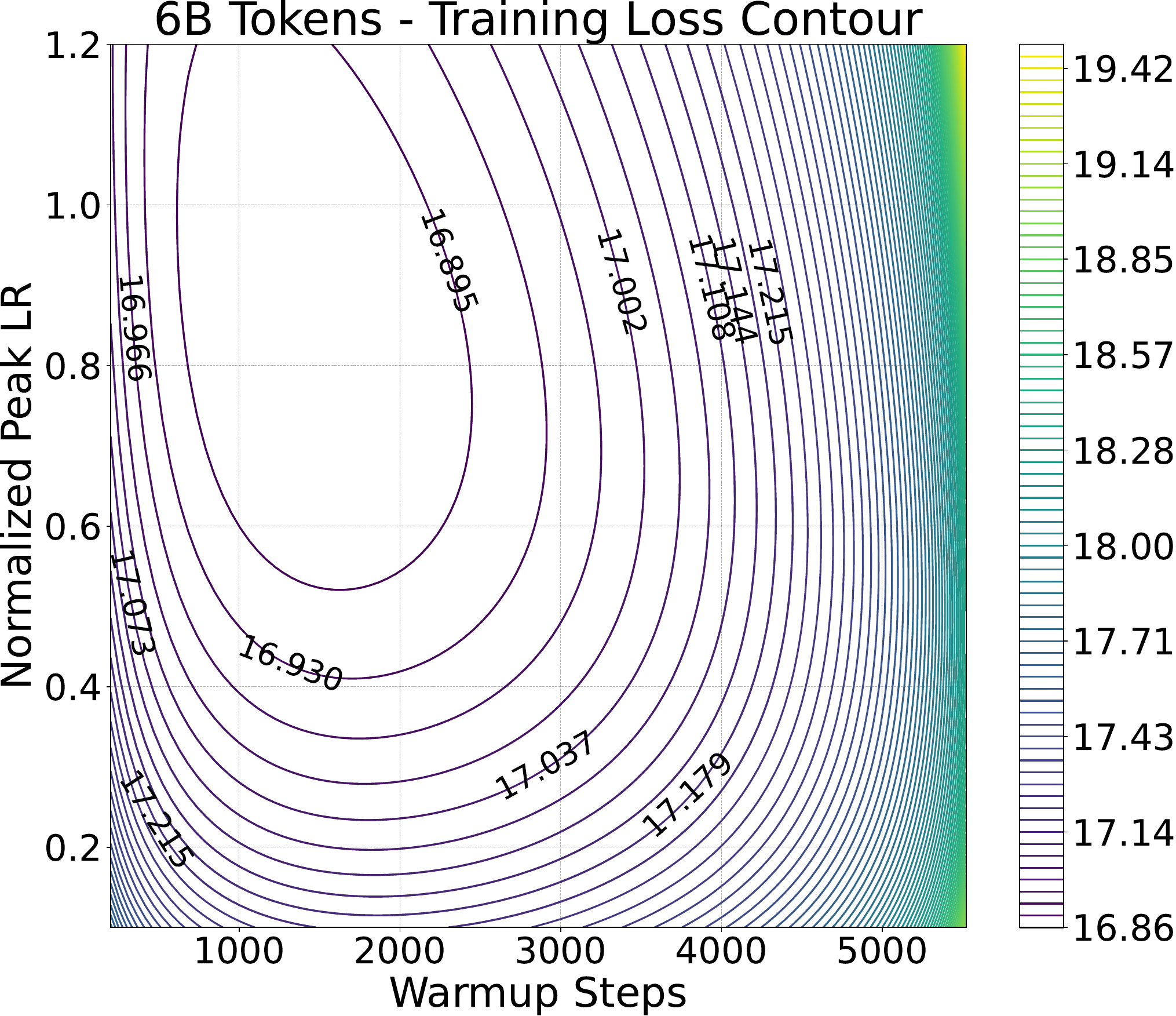}
        }
    \end{minipage}
    \vspace{-4mm}  % 调整这里的值来增加或减少距离
    \begin{minipage}{\textwidth}
        \subfigure[]{
            \includegraphics[width=0.48\linewidth]{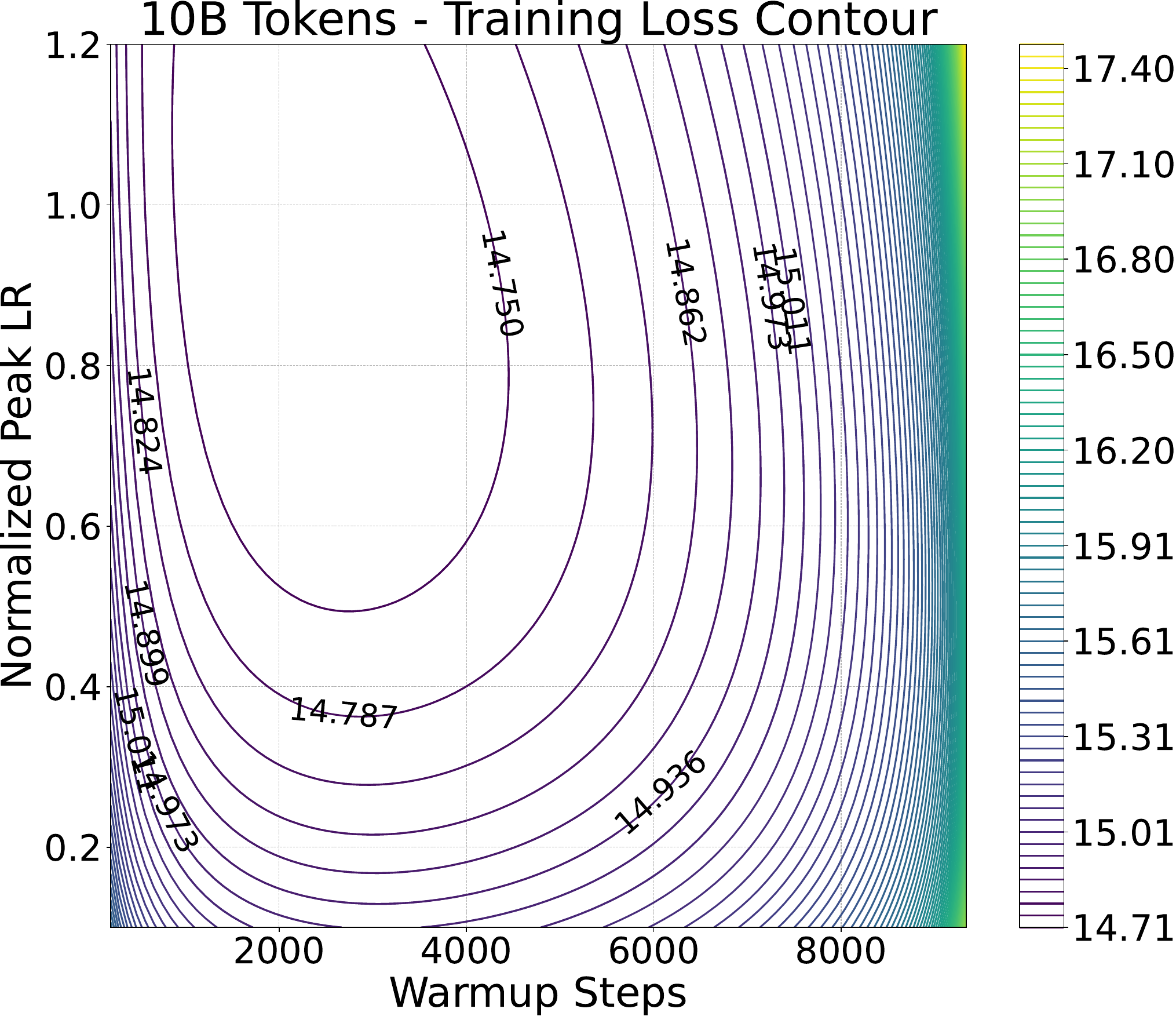}
        }
        \subfigure[]{
            \includegraphics[width=0.48\linewidth]{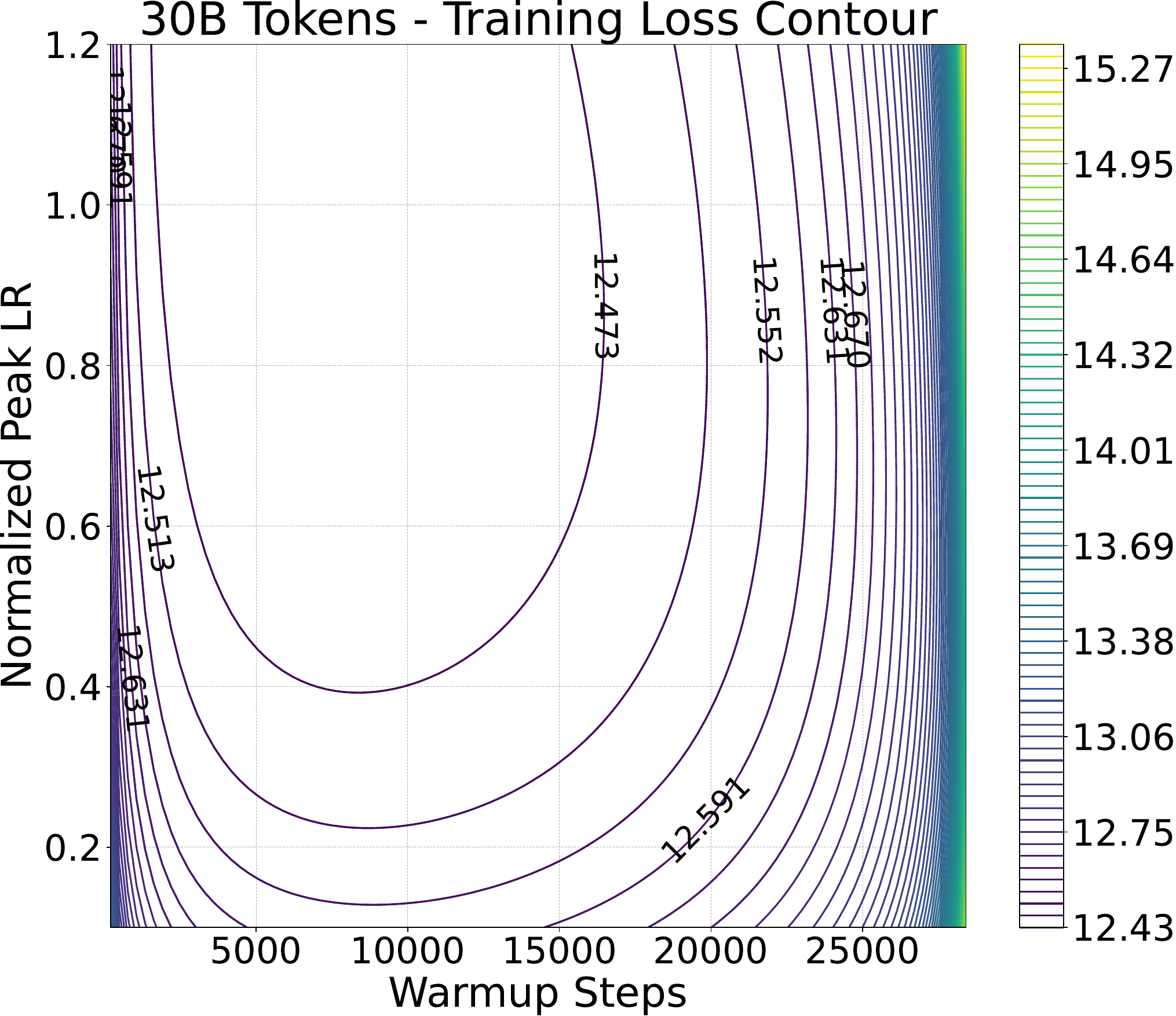}
        }
    \end{minipage}
   \caption{\blue{\abbr{}-predicted perplexity contours as a function of warmup steps and normalized peak LR for the $8{\times}0.1$B MoE model at four token budgets (3B, 6B, 10B, 30B). Darker 
   regions indicate lower predicted loss. As the token budget grows, the low-loss region expands and becomes less sensitive to warmup configuration.}}
    \label{fig:toyexample}
    \vspace{-4mm}
\end{figure}
\subsection{\blue{Illustrative Fixed-Model Fit}}\label{sec:fit1}
\blue{This subsection provides a qualitative fixed-model illustration of the simplified \abbr{}. The main predictive evaluation of the paper is deferred to the held-out and extrapolation benchmarks in Secs.~4 and~5.}

We apply \abbr{} in Eqn.~\eqref{eq:optlaw} to fit the training loss of an $8\times0.1$B MoE model~\citep{zhao2024longskywork,wei2024skywork}. The decision to use an MoE model over a dense model is based on two key factors: (1) many state-of-the-art LLMs~\citep{achiam2023gpt,deepseekai2024deepseekv2,reid2024gemini,yang2024qwen2,bai2023qwen} employ MoE architectures, reflecting current trends, and (2) MoE models are computationally more efficient, allowing faster experimentation within our computational constraints. In this experiment, we fix the model and batch sizes (both with a token length and batch size of 1024) while varying the token quantities from the RedPajama-v2 dataset~\citep{together2023redpajama}: 3B, 6B, 10B, and 30B tokens. We adjust the warmup steps and peak LRs to analyze their impact on final loss. The LR schedule $\eta(t)$ linearly increases to a peak LR $\eta_{\operatorname{max}}$ during warmup, followed by a linear decay to zero in the cooldown phase. The warmup steps are denoted by $a$.
\blue{Here and below, the normalized peak LR means the peak LR divided by the reference scale used in our fitting procedure (see Appendix~\ref{sec:fitted-coeff}).}

\blue{We use the fitted \abbr{} to predict final training loss across the warmup-step and peak-LR grid for each token budget. The predicted perplexity contours are shown in Fig.~\ref{fig:toyexample}, with a regression error on convergent runs below 0.5\%. Several patterns are visible: as training data increases, the strong-performance region expands, indicating greater robustness to hyper-parameter variation. The influence of warmup steps weakens at larger token budgets, consistent with prior observations~\citep{gupta2023continual,ibrahim2024simple}, though excessively short warmup remains harmful (left side of Fig.~\ref{fig:toyexample}). These patterns are interpreted through the \abbr{} formulation in the next subsection. Divergent runs are excluded from the regression error and analyzed in Sec.~\ref{sec:divergence}.}

\subsection{\blue{Qualitative Predictions of the Simplified \abbr{}}}
\blue{This subsection interprets several observed LLM training phenomena through the simplified \abbr{}. By analyzing the convergence and escape terms in Eqn.~\eqref{eq:optlaw}, we derive qualitative predictions about how warmup duration and LR schedule choice affect final loss.}

\subsubsection{Influence of Warmup Steps on Training Loss}
Recent research has demonstrated that the number of warmup steps \blue{has little effect on the final loss during continual training}. \blue{For instance, \citet{ibrahim2024simple,gupta2023continual} report negligible loss variation across warmup configurations for 0.5B-parameter dense models trained on 100B+ tokens.} This finding is consistent with our results in Fig.~\ref{fig:toyexample} (d). When the token count is sufficiently large and surpasses the Chinchilla Scaling Law (25.6 tokens per parameter)~\citep{hoffmann2022training,besiroglu2024chinchilla}, the range of effective warmup steps broadens significantly, \blue{making the final loss insensitive to changes in warmup duration.}

This observation does not contradict other studies~\citep{lv2023full,jin2023rethinking} that emphasize the need for careful tuning of warmup steps. Such sensitivity is primarily seen during fine-tuning, where the token-to-parameter ratio is much lower than the value suggested by the Chinchilla Scaling Law. This aligns with our findings in Fig.~\ref{fig:toyexample}~(a) and Fig.~\ref{fig:toyexample}~(b), which show a narrower optimal range for warmup steps. Consequently, fluctuations in loss due to variations in the warmup steps are more likely to occur during fine-tuning.

\subsubsection{Insights into LR schedule Effects through Opt-Laws}\label{sec:effect}
OpenAI researchers previously observed that the LR schedule $\eta(t)$ has a negligible impact on the final training loss~\citep{kaplan2020scaling}. However, recent studies on pre-training and continual training of LLMs have discovered new LR schedules that achieve lower training losses compared to the widely-used cosine decay schedule~\citep{ibrahim2024simple,hagele2024scaling}. \blue{\abbr{} offers a unified interpretation of these seemingly contradictory observations.}

\abbr{} reveal that the model loss has an asymptotic lower bound related to the model capability. When the dataset is sufficiently large, the training loss approaches this asymptotic bound, rendering it independent of the LR schedule. However, the rate at which the loss approaches this bound is influenced by the LR schedule $\eta(t)$. This explains why, during fine-tuning or continual training, it is possible to find some LR schedules that outperform cosine decay: these newly discovered schedules enable the model to reach its performance limits more rapidly. Yet, as the dataset size increases, the differences introduced by varying LR schedules diminish.

To further elucidate this asymptotic behavior, we will now consider two specific scenarios to derive the precise expressions of the \abbr{}.
The first is the classical linear warmup followed by cosine decay. In this schedule, the LR is linearly increased to $\eta_{\operatorname{max}}$ over $a$ steps, and then decays to zero in a cosine fashion.
\begin{equation}\label{eq:cosine}
\eta_{\operatorname{cos}}(t) = 
\left\{
\begin{array}{lll}
    \eta_{\operatorname{max}}\cdot \frac{t}{a} & \qquad t \in [0, a]  & \qquad \text{(warmup)}, \\
    \frac{\eta_{\operatorname{max}}}{2} \cdot \qty(\cos\qty(\pi\cdot\frac{t-a}{S-a})+1) & \qquad  t \in [a, S] & \qquad \text{(cooldown)},
\end{array}
\right.
\end{equation}
where $a$ is the linear warmup steps, and $S$ represents the total training steps.
The second schedule has gained popularity in the recent continual training of LLMs, as evidenced in the work~\citep{hu2024minicpm,hagele2024scaling}. In this strategy, the LR is linearly increased to $\eta_{\operatorname{max}}$, held constant for a period, and then rapidly decayed to the minimum value. For simplicity, we model the decay as a linear decrease to zero. %The specific expression for this schedule is as follows:
\begin{equation}\label{eq:minicpm}
\eta_{\operatorname{const}}(t) = 
\left\{
\begin{array}{lll}
    \eta_{\operatorname{max}}\cdot \frac{t}{a} & \qquad t \in [0, a]  & \qquad \text{(warmup)}, \\
    \eta_{\operatorname{max}} & \qquad t \in [a, a_c]  & \qquad \text{(constant)}, \\
    \eta_{\operatorname{max}} \cdot \qty(1-\frac{t-a_c}{S-a_c}) & \qquad  t \in [a_c, S] & \qquad \text{(cooldown)}
\end{array}.
\right.
\end{equation}
Here, $a_c-a$ represents the duration of the constant LR phase, and $S-a_c$ is the period over which the LR linearly decays to zero. In \abbr{} Eqn.~\eqref{eq:optlaw}, we incorporate the constant LR phase into the cooldown phase for simplification. For more complex LR schedules, the distinction between the warmup phase and other phases is elaborated in Sec.~\ref{sec:apply-opt-laws}. 
We can derive the following proposition, which demonstrates that as the data size increases, the difference between $\eta_{\operatorname{cos}}$ and $\eta_{\operatorname{const}}$ diminishes under $\operatorname{\abbr{}}\qty(\eta(\cdot))$, which we define as $\log({\rm Loss})$ in Eqn.~\eqref{eq:optlaw}.
\begin{myprop}\label{prop:effect}
Let $a=r_a S$ and $a_c=r_{a_c}S$. For any $r_a>0$ and $r_{a_c}>0$ such that $0<r_a\leq r_{a_c}<1$, it holds that 
\[
\lim_{S\rightarrow\infty}\abs{\operatorname{\abbr{}}\qty(\eta_{\operatorname{cos}}(\cdot))-\operatorname{\abbr{}}\qty(\eta_{\operatorname{const}}(\cdot))}=0.
\]
\end{myprop}
\blue{If we set $a = 0.01S$ and $a_c = 0.85S$ (as recommended by~\cite{hu2024minicpm}), when $S$ is relatively small, $\eta_{\operatorname{const}}$ yields a lower predicted loss because $\int_a^S \eta_{\operatorname{const}}(s) \dd s$ is larger, consistent with \citet{hu2024minicpm,hagele2024scaling}. However, when $a = a_c$ (no constant phase), the ordering reverses, matching our experimental results in Sec.~\ref{sec:exp}.}

As $S$ becomes sufficiently large, the \abbr{} for both LR schedules asymptotically converge to a fixed value. This value depends solely on the model capacity and data property, and is independent of the specific LR schedule. \blue{This explains the finding of \citet{kaplan2020scaling}, where small models (3M parameters) trained on large datasets showed negligible schedule sensitivity, as the loss was already near its asymptotic limit.}

\blue{These observations motivate the generalized construction in Sec.~\ref{sec:general}, which introduces cross-scale variation, divergence modeling, and interaction terms, with empirical evaluation in Sec.~\ref{sec:exp}.}

\section{\blue{Generalized Opt-Law across Scales}}\label{sec:general}
\blue{This section extends the simplified \abbr{} to the multi-scale setting by incorporating model size, divergence boundaries, and interaction terms.}
We standardize the optimizer and LR schedule (linear warmup followed by linear cooldown) and perform grid experiments across a range of model sizes, LRs, warmup steps, and data sizes. Specific details on the MoE configurations are provided in Sec.~\ref{sec:exp-supplement}. The results are summarized in Fig.~\ref{fig:MoE-Loss}, where each grid point represents the smoothed final training loss; divergent runs are assigned a loss of $7$.
\blue{As Fig.~\ref{fig:MoE-Loss} shows, incorporating model size reveals non-monotonic local structure and divergence boundaries that the simplified unimodal form in Sec.~3 cannot capture. To address these challenges, we first define a divergence prediction criterion (Sec.~\ref{sec:divergence}), then construct the generalized feature family (Sec.~\ref{sec:generalized}), and evaluate the resulting law via a held-out benchmark on the training grid (Sec.~\ref{sec:apply-opt-laws}).}

\begin{figure}[t]
    \centering
    \includegraphics[width=\linewidth]{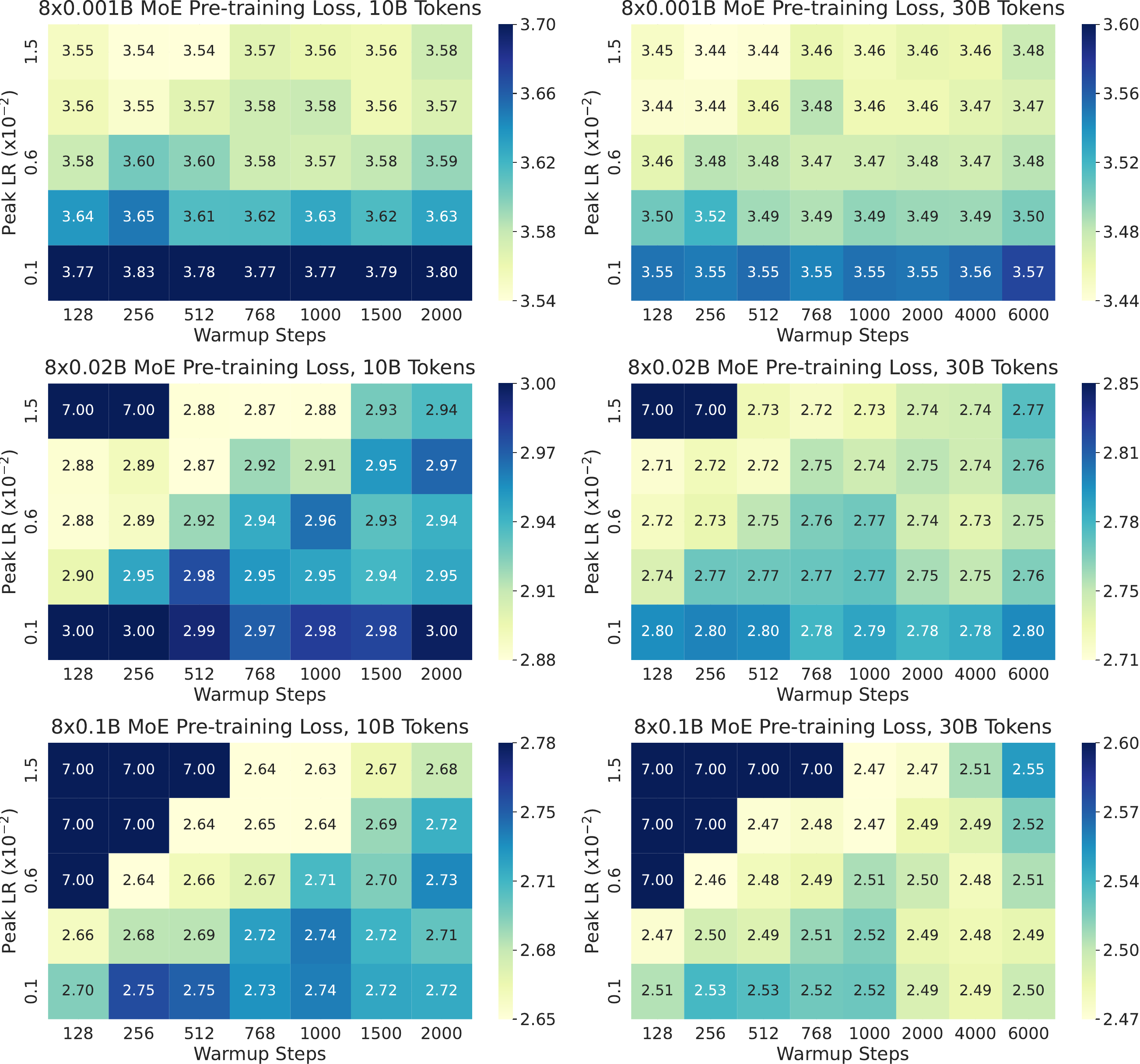}
    \caption{Smoothed final training loss across various combinations of training parameters, including model sizes from $8 \times 0.001$B to $8 \times 0.3$B MoEs, peak LRs from 1e-3 to 1.5e-2, warmup steps from 128 to 6000, and data sizes of 10B and 30B tokens. Each grid point represents the loss for a specific parameter set. Divergent training runs were assigned a loss of 7, reflecting the typical plateau observed in practice.}
    \label{fig:MoE-Loss}
    \vspace{-4mm}
\end{figure}

\subsection{Predicting Training Divergence}\label{sec:divergence}
We first discuss the selection of hyper-parameters to prevent training divergence. Fig.~\ref{fig:MoE-Loss} provides two key insights: (1) Excessively high peak LRs can cause divergence, but increasing the warmup duration can prevent this, to make training stable. (2) The hyper-parameter combinations that lead to divergence vary, depending on the model and data size.

For general non-convex optimization problems $\min_{\*x} f(\bm{\*x})$, optimization theory typically dictates that the peak LR $\eta_{\max}$ should not exceed $2/L$, where $L$ is the Lipschitz constant of the function gradient $\nabla f(\cdot)$, regardless of the LR schedule~\citep{arjevani2019lower,xie2024adan, rotaru2024exact}. However, estimating the Lipschitz constant for neural networks remains an open challenge~\citep{kim2021lipschitz,khromov2024some}. Moreover, recent studies indicate that surpassing certain theoretical thresholds for peak LRs does not necessarily lead to divergence; rather, when appropriately managed, peak LRs above $2/L$ can even enhance convergence~\citep{grimmer2024accelerated}.

From these observations, we hypothesize that training divergence is not merely a consequence of an excessively large $\eta_{\max}$. Instead, it is influenced by the duration that the LR remains above a critical threshold $\eta_L$ and the length of the warmup phase. Ideally, $\eta_L$ would correspond to $2/L$, where $L$ is the Lipschitz constant of the function gradient. For neural networks, $L$ tends to increase with model size~\citep{khromov2024some}, suggesting that $\eta_L$ should decrease as the model size grows. However, with larger datasets $\eta_L$ can increase. A larger data volume gives the training dynamics more iterations to recover from an excessively high peak LR, as illustrated in Fig.~\ref{fig:MoE-Loss}. Based on this intuition, we let $\eta_L = \order{S^{\hat{\alpha}_1}/N^{\hat{\alpha}_2}}$, where $\hat{\alpha}_1, \hat{\alpha}_2>0$ are data-driven constants, $S$ is the number of iterations (proportional to data size), and $N$ is the model size (i.e., the number of learnable parameters).
% From these observations, we hypothesize that training divergence is not merely a consequence of an excessively large $\eta_{\max}$. Instead, it is influenced by the duration that the LR remains above a critical threshold $\eta_L$ and the length of the warmup phase. {Ideally, $\eta_L$ would correspond to $2/L$, where $L$ is the Lipschitz constant of the function gradient.} For neural networks, $L$ tends to increase with model size~\citep{khromov2024some}. However, in practice, $\eta_L$ is also affected by data size, as observed in Fig.~\ref{fig:MoE-Loss}. Intuitively, larger datasets tend to reduce the likelihood of divergence by providing more time to correct the effects of an excessively high LR. Nonetheless, this relationship is also influenced by the duration of the warmup phase.
\blue{Building on these insights, we define a divergence criterion by comparing the below-threshold warmup area against the above-threshold area in the LR schedule (Fig.~\ref{fig:criterion}). When above-threshold exposure dominates, training is expected to diverge. We parameterize this comparison as:}\vspace{-4mm}
\begin{equation}\label{eq:criterion}
R(\eta_{\max}, a_1, N, S) \coloneqq \frac{S \qty(\eta_{\max} - \eta_L )^2}{\hat{c}_3 a_1\eta_L^2}, \quad \eta_L \coloneqq \min(\eta_{\max}, \frac{\hat{c}_1 S^{\hat{\alpha}_1}}{\hat{c}_2 N^{\hat{\alpha}_2}}),
\vspace{-2mm}
\end{equation}
where $S$ is the iteration number, $N$ is the model size, $a_1$ is the number of warmup steps \blue{(corresponding to the warm-up duration $a$ in the simplified \abbr{} above)}, and $\hat{c}_1 > 0$, $\hat{c}_2 > 0$, $\hat{c}_3 > 0$, $\hat{\alpha}_1 > 0$, and $\hat{\alpha}_2 > 0$ are data-driven parameters. Similar to the \abbr{} approach, these parameters are estimated by fitting the real data presented in Fig.~\ref{fig:MoE-Loss}, with the fitted values provided in Appendix Sec.~\ref{sec:exp-supplement}.
\begin{figure}[t]
  \centering
  \includegraphics[width=0.4\linewidth]{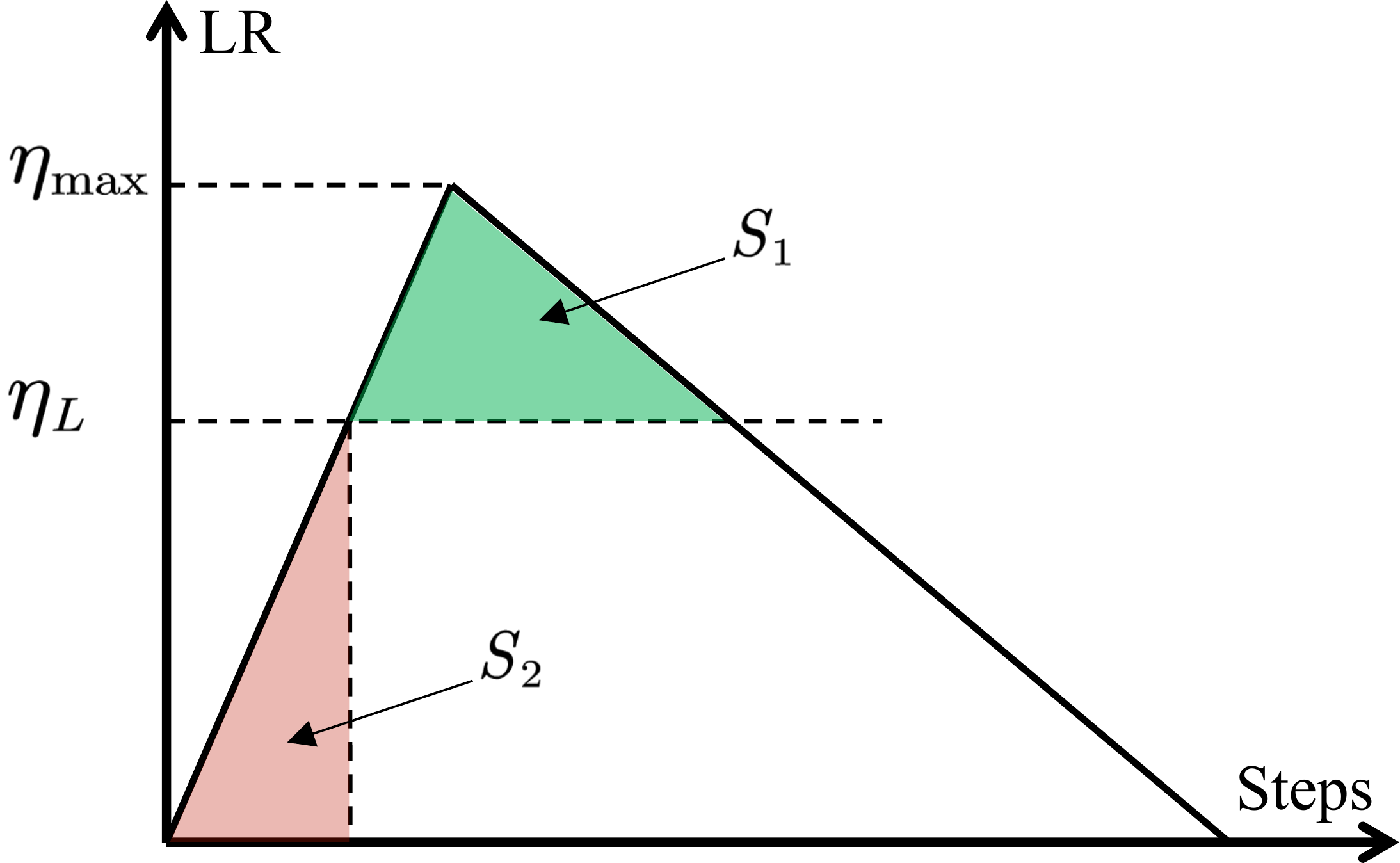}
  \caption{\blue{Divergence prediction criterion. The shaded regions denote the below-threshold warmup area (LR $< \eta_L$) and the above-threshold area (LR $\geq \eta_L$). The ratio of these two areas motivates the divergence discriminant $R$ in Eqn.~\eqref{eq:criterion}.}}
  \label{fig:criterion}
  \vspace{-4mm}
\end{figure}
\blue{The area comparison motivates the functional form of $R$, while the coefficient $\hat{c}_3$ calibrates the decision boundary to $R = 1$. Empirically, warmup stabilizes training more effectively than above-threshold exposure destabilizes it, so the fitted boundary lies far from the equal-area point.}
We find that if $R(\eta_{\max}, a_1, N, S) > 1$, indicating that the time spent above $\eta_L$ is too long relative to the effective warmup duration, training is likely to fail. Conversely, if $R(\eta_{\max}, a_1, N, S) < 1$, the chosen hyper-parameters are unlikely to cause divergence. In Eqn.~\eqref{eq:criterion}, $\eta_L$ is inversely related to model size and directly proportional to data size, implying that larger models require lower peak LRs, while larger datasets may permit a higher peak LR. \blue{This is consistent with} established practices in training modern LLMs, such as LLaMAs~\citep{dubey2024llama}.

As shown in Fig.~\ref{fig:loss-fit}, the $R(\eta_{\max}, a_1, N, S)$ metric effectively predicts training outcomes across various data sizes, model sizes, and hyper-parameter combinations. In the figure, for scenarios where $R > 1$, we assign a fixed sentinel value of 7 to mark expected divergence.

\subsection{Generalized \abbr{}}\label{sec:generalized}

\blue{Building on the patterns in Fig.~\ref{fig:MoE-Loss}, we construct the feature family in three stages.}
First, we define schedule-dependent basis terms motivated by the convergence and escape analyses, which summarize how the schedule affects optimization progress and schedule-induced exploration. Second, we introduce scale-dependent terms for model size and training horizon, in order to capture cross-scale variation that is absent from the simplified law. Third, we include a restricted set of interaction terms between these components, so that the effect of a schedule can vary with model and data scale.

The goal of this construction is not to enumerate all possible monomials, but to retain a compact and interpretable theory-motivated basis that is expressive enough for the observed loss grids while avoiding unnecessary over-parameterization.
Specifically, given a LR schedule $\eta(t)$, \blue{SDE-based convergence analysis (Sec.~\ref{sec:sde}, Theorems~\ref{thm:covergence-SGD} and~\ref{thm:covergence-adam}) shows that the averaged squared gradient norm for both SGD and Adam is bounded in terms of the integrated learning rate $\int_0^t \eta(s)\,ds$. This integral therefore serves as a natural surrogate for optimization progress. The resulting convergence-related basis terms take the form}
\begin{equation}\label{eq:cov-basis}
    \text{Convergence Bound} \propto \qty [\frac{N}{\int_0^{a_{c_1}} \eta(t) \dd t}, \frac{N}{\int_{a_{c_2}}^S \eta(t) \dd t}],
\end{equation}
where the convergence bound represents the average gradient norm, usually serving as an indicator of how near the current loss is to a local minimum in continuous optimization, $S$ represents the number of iterations, $N$ is the model size, and $a_{c_1} \leq a_{c_2}$ are specific points within the interval $[0, S]$ that define the divisions in the LR schedule. 
The ``proportional to'' $\propto$ indicates that the convergence bound is influenced by these basis functions in Eqn.~\eqref{eq:cov-basis}. Specifically, when the values of these functions are smaller for a given  $\eta(t)$, the optimization algorithm tends to converge more quickly to a stationary point.

\blue{Typically, $a_{c_1}$ marks the end of the warmup phase, and $a_{c_2}$ indicates the start of the cooldown phase. \cite{ibrahim2024simple} observed that LR variations during intermediate phases have limited effect on the final loss when warmup and cooldown are properly configured. We therefore retain only the warmup and cooldown segments in the convergence features. The specific instantiation and a supporting ablation are given in Sec.~\ref{sec:apply-opt-laws} and Appendix Sec.~\ref{sec:ablation}.}

\blue{Complementarily, SDE-based escape analysis (Sec.~\ref{sec:sde}, Theorems~\ref{thm:SGD-prob} and~\ref{thm:Adam-prob}) shows that the probability of remaining trapped near a local minimum grows with the integrated squared learning-rate derivative, motivating the following escape-related basis terms:}
\begin{equation}\label{eq:escap-basis}
    \text{Trapping Probability} \propto \qty [\frac{1}{SN}, \int_0^{a_{e_1}} {\eta^\prime}(t)^2 \dd t,\int_{a_{e_2}}^S {\eta^\prime}(t)^2 \dd t],
\end{equation}
where $N$ is the model size, and $a_{e_1} \leq a_{e_2}$ are values between $0$ and $S$. Similar to the previous Eqn.~\eqref{eq:cov-basis}, smaller values of these basis functions in Eqn.~\eqref{eq:escap-basis} indicate a higher probability of escaping a local region under the chosen $\eta(t)$.

To extend the \abbr{}, we consider pairwise combinations of the basis functions from Eqn.~\eqref{eq:cov-basis} and Eqn.~\eqref{eq:escap-basis}. We introduce three sets of terms: the convergence term $C(\eta_t,\*a_c,S,N,\bm{\alpha})$:
\[
C(\eta_t,\*a_c,S,N,\bm{\alpha}) \coloneqq \qty[\frac{1}{\int_0^{a_{c_1}} \eta_t}, \frac{1}{\int_{a_{c_2}}^S \eta_t},\qty(\frac{N}{\int_{a_{c_2}}^S \eta_t})^{\alpha_1}, \qty(\frac{1}{\int_0^{a_{c_1}} \eta_t\int_{a_{c_2}}^S \eta_t})^{\alpha_2}],
\]
where $\bm{\alpha}$ is a predefined vector representing the powers of certain basis functions; the escaping local region term $E(\eta_t,a_e,S,N,\bm{\alpha})$:
\[
E(\eta_t,a_e,S,N,\bm{\alpha}) \coloneqq \qty [\int_{a_{e_2}}^S {\eta^\prime_t}^2, \qty(\int_0^{a_{e_1}} {\eta^\prime_t}^2)^{\alpha_3}, \qty(\int_{a_{e_2}}^S {\eta^\prime_t}^2)^{\alpha_4}, \qty(\frac{1}{SN})^{\alpha_5}],
\]
And the mixed term $M(\eta_t,\*a_c, \*a_e,S,N,\bm{\alpha})$:
\[
M(\eta_t,\*a_c, \*a_e,S,N,\bm{\alpha}) \coloneqq \!\!\!\qty [\qty(\frac{\int_{a_{e_2}}^S {\eta^\prime_t}^2}{\int_0^{a_{c_1}} \eta_t})^{\alpha_6}, \qty(\frac{\int_{a_{e_2}}^S {\eta^\prime_t}^2}{\int_{a_{c_2}}^S \eta_t})^{\alpha_7},
\qty(\frac{N\int_{a_{e_2}}^S {\eta^\prime_t}^2}{\int_0^{a_{c_1}} \eta_t})^{\alpha_8}, \qty(\frac{N\int_{a_{e_2}}^S {\eta^\prime_t}^2}{\int_{a_{c_2}}^S \eta_t})^{\alpha_9}].
\]
We then combine these terms to define the optimization-feature vector for \abbr{}:
\begin{equation}\label{eq:feature}
  F(\eta_t,\*a_c,\*a_e,S,N,\bm{\alpha}) \coloneqq \qty[C(\eta_t), E(\eta_t), M(\eta_t), N^{-\alpha_{10}}, S^{-\alpha_{11}}, \eta_{\max}^{\alpha_{12}}, 1],
\end{equation}
where, for the sake of notation, we omit certain arguments of $C(\cdot)$, $E(\cdot)$, and $M(\cdot)$.
\blue{For readability, the main text groups the terms into convergence, escape, mixed, and scale components; the corresponding fitted coefficients and closed-form expressions for two representative schedule families are given in Appendix Table~\ref{tab:opt-law-coeff}.}
\paragraph{Discussion} \blue{The terms in $F(\cdot)$ are derived from the SDE convergence bound $O(N/\int\eta)$ (Theorems~\ref{thm:covergence-SGD}--\ref{thm:covergence-adam}) and the escape bound involving $\int{\eta'_t}^2$ and $1/SN$ (Theorems~\ref{thm:SGD-prob}--\ref{thm:Adam-prob}). Splitting the schedule into warmup and cooldown segments yields the per-segment convergence terms $1/\int\eta_w$ and $1/\int\eta_c$, and the model-size-scaled term $(N/\int\eta_c)^{\alpha}$ that retains $N$ from the bound. The warmup-cooldown product $(1/(\int\eta_w \cdot \int\eta_c))^{\alpha}$ captures their joint contribution. The escape block follows the same splitting logic. The mixed terms multiply the convergence basis $1/\int\eta$ with the escape basis $\int{\eta'_t}^2$, producing $\int{\eta'_t}^2/\int\eta$ and $N \cdot \int{\eta'_t}^2/\int\eta$ with warmup and cooldown variants. Combinations not present in the bounds (e.g., $\int\eta \times S$) are not included. This yields 15 terms plus one intercept.}

\blue{The three terms ($1/\int\eta_{w}$, $1/\int\eta_{c}$, $\int{\eta'_t}^2$) 
% \fbox{\red{Should be $1/\int\eta_{w}$, $1/\int\eta_{c}$???}}
preserve the closed-form expressions from the SDE analysis and carry no fitted exponent. Scale terms carry fitted exponents following standard power-law parameterization. Interaction terms require fitted exponents because the convergence and escape bounds are derived independently and do not predict their coupling. The escape exponents $\alpha_3, \alpha_4$ are fitted rather than fixed at the theoretical value, as the underlying bounds rely on approximations that may not fully capture real training dynamics.}

\blue{The feature construction reveals a natural convergence-escape trade-off.} The convergence terms $C(\cdot)$ suggest that maintaining a higher LR allows for a larger $\int \eta_t$, which accelerates the convergence. However, an excessively high LR can lead to increased $(\eta_t^\prime)^2$ during warmup and cooldown, causing the escape terms $E(\cdot)$ to rise and potentially trapping the model in poor local minima. Unlike in Sec.~\ref{sec:opt-law1}, where model size was not factored in, introducing model size reveals a trade-off: larger models slow convergence (as $N/\int \eta_t$ increases) but reduce the escape terms (as $1/SN$ decreases), highlighting the need for balance.
Additionally, the mixed terms provide insights into the interaction between model size and peak LR. For example, the term ${N\int (\eta^\prime_t)^2}/{\int \eta_t}$ indicates that when data size increases modestly while model size $N$ grows significantly, reducing the peak LR becomes necessary to keep this term small. This aligns with empirical observations that the optimal peak LR tends to decrease as model size increases~\citep{dubey2024llama}.

\blue{The feature vector $F(\cdot)$ is designed to capture the dominant factors relating training hyper-parameters to final loss,} enabling us to derive generalized \abbr{} through linear regression.
\begin{equation}\label{eq:opt-law2}
    \log\qty(\text{Loss})  = \*c^\top F(\eta_t,\*a_c,\*a_e,S,N,\bm{\alpha}),    \tag{{\abbr{}}}
\end{equation}
where $(\eta_t,\*a_c,\*a_e,S,N)$ is the combination of training parameters, $\bm{\alpha}$ represents the powers of the basis functions, and $\*c$ is the solution to the following linear regression problem:
\begin{equation}\label{eq:LSR}
    \*c = \argmin_{\*c} \qty{\sum_i \qty(c^\top F(\eta_t^i,\*a_c^i,\*a_e^i,S^i,N^i,\bm{\alpha}) - \log\qty(\text{Loss}_i))^2},
\end{equation}
where $(\eta_t^i, \*a_c^i, \*a_e^i, S^i, N^i)$ represents the different hyper-parameter combinations shown in Fig.~\ref{fig:MoE-Loss}, and $\text{Loss}_i$ denotes the final training loss associated with each combination (excluding divergent losses during regression). \blue{We adopt a staged procedure: the exponent vector $\bm{\alpha}$ is selected from a discrete candidate set using only the training data, and the linear coefficients $\mathbf{c}$ are then fitted on the training split. The held-out set is reserved for final evaluation only (details in Appendix~Sec.~\ref{sec:alpha-protocol}).}
\begin{figure}[t]
  \centering
  \includegraphics[width=0.35\linewidth]{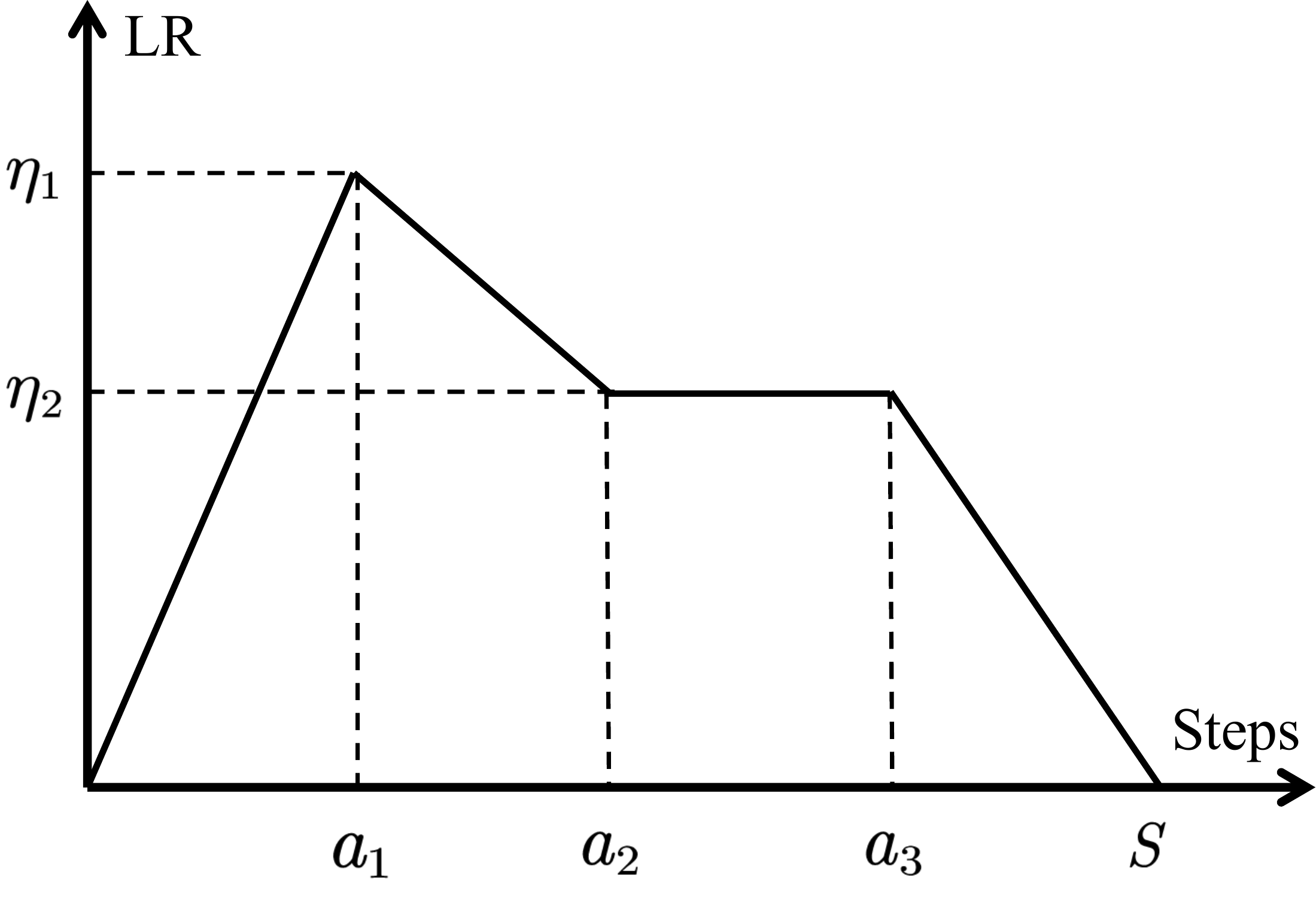}
  \caption{A four-phase LR schedule template illustrating the selection of $\mathbf{a}_c$ and $\mathbf{a}_e$ in \ref{eq:opt-law2}.}
  \label{fig:lr_illustration}
  \vspace{-4mm}
\end{figure}
The proposed \ref{eq:opt-law2} can, under certain conditions, recover classical scaling laws. For instance, when the LR schedule is fixed and model size $N$ is the only variable, normalizing $N$ to a range much smaller than 1 allows \ref{eq:opt-law2} to approximate the broken neural scaling laws (BNSL)~\citep{caballero2023broken}. Specifically, BNSL is expressed as $\log(\text{Loss}) = \bar{b} + \sum_i \bar{c}_i \log(1+N^{\bar{\alpha}_i}) \approx \bar{b} + \sum_i\bar{c}_i N^{\bar{\alpha}_i}$, where $\bar{\alpha}_i$ and $\bar{b}$ are data-driven parameters and the approximation represents \ref{eq:opt-law2}. BNSL has been shown to outperform the classical Kaplan scaling laws across both upstream and downstream tasks. Furthermore, by applying a logarithmic transformation to $N$ during normalization and selecting $\alpha_1=\alpha_8=\alpha_9=\alpha_{10}=1$, \ref{eq:opt-law2} can closely approximate the classical Kaplan and Chinchilla scaling laws~\citep{kaplan2020scaling,hoffmann2022training}.

\subsection{\blue{Evaluation on the Training Grid}}\label{sec:apply-opt-laws}
\blue{This subsection evaluates the generalized \ref{eq:opt-law2} on the training grid via held-out benchmarking. Fig.~\ref{fig:lr_illustration} depicts the piecewise-linear schedule template used throughout, parameterized by transition points $a_1, a_2, a_3$. Standard families arise as special cases: linear warmup plus linear cooldown ($a_1{=}a_2{=}a_3$), constant-rate-plus-cooldown~\citep{hagele2024scaling} ($a_1{=}a_2{<}a_3$), and multi-phase polygon schedules~\citep{ibrahim2024simple} ($a_1{<}a_2{<}a_3$).}

\blue{To evaluate \ref{eq:opt-law2} on this template, we map the abstract partition points $\mathbf{a}_c$ and $\mathbf{a}_e$ to the transition points $a_1, a_2, a_3$ that define the phase boundaries of the schedule. For the convergence terms, we set $a_{c_1}=a_1$ and $a_{c_2}=a_2$. If a constant phase precedes the cooldown at the same LR, it is included in $\int_{a_{c_2}}^S\eta\,dt$, as it involves no LR change and shares the same LR as the cooldown onset. The decay phase $[a_1, a_2]$ is excluded because LR variations during intermediate phases have limited effect on the final loss~\citep{ibrahim2024simple}. This information is instead captured by the escape feature $\int_0^{a_2}{\eta'_t}^2$, and the ablation in Appendix Sec.~\ref{sec:ablation} confirms that additionally including the decay $[a_1, a_2]$ in the convergence integral does not improve prediction. For the escape terms, we set $a_{e_1}=a_{e_2}=a_2$, so that the first escape integral covers warmup and decay jointly, while the second covers cooldown (the intervening plateau has $\eta'=0$).}

\blue{\paragraph{Evaluation protocol.}
The loss grids in Fig.~\ref{fig:MoE-Loss} contain multiple (model size, token budget) blocks, each comprising several hyper-parameter configurations (peak LR, warmup steps). We perform 5-fold cross-validation by partitioning the configurations within each block into five folds, so that every fold contains held-out samples from all blocks. The exponent vector $\bm{\alpha}$ and linear coefficients $\mathbf{c}$ are fitted on the training folds only. Divergent runs are identified separately by the criterion in Eqn.~\eqref{eq:criterion} and excluded from the loss regression, but retained for fitting and evaluating the divergence criterion itself.}

\blue{Fig.~\ref{fig:loss-fit} compares the actual and predicted loss grids for one representative fold, where dashed-border cells denote held-out configurations. Within each block, the predicted surface correctly identifies the low-loss region and reproduces the sensitivity to peak LR and warmup steps. The predicted loss for held-out cells closely matches the actual loss, confirming that the law generalizes to unseen configurations. By construction, the predicted surface is smooth and does not reproduce local irregularities in the empirical grid, but this property is desirable for schedule selection as it filters out per-run noise and yields a stable ranking. Divergent configurations, flagged by the $R$ criterion with a sentinel value of 7, align well with the observed divergence boundaries. Quantitative metrics aggregated across all five folds are reported in Table~\ref{tab:main-heldout-benchmark}.}

\begin{figure}[t!]
    \centering
    \includegraphics[width=\linewidth]{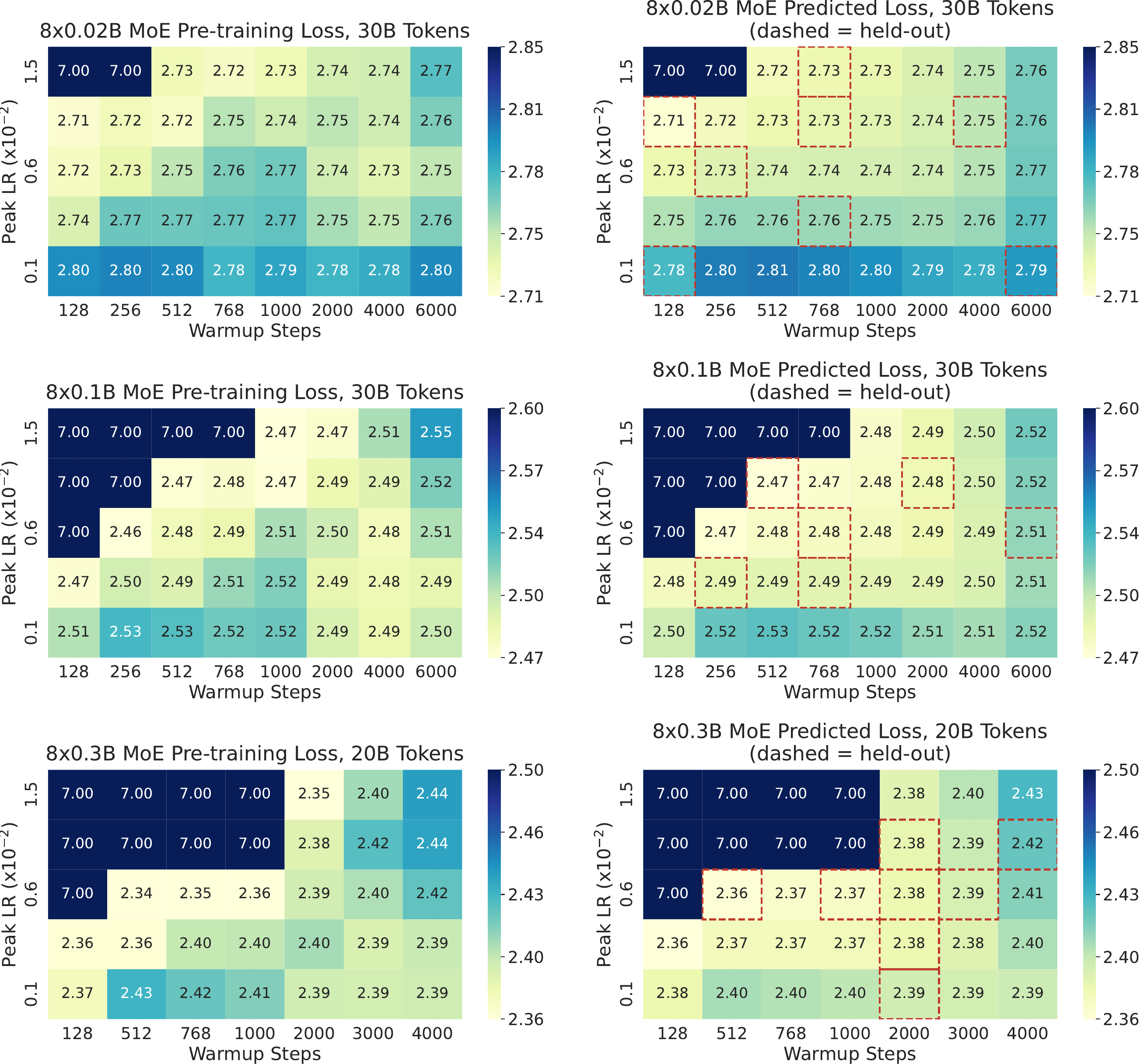}
    \vspace{-6mm}
    \caption{\blue{Actual training loss (left) versus predictions from generalized \ref{eq:opt-law2} (right) across model sizes and token budgets. Dashed-border cells are held out in the displayed fold. Divergent runs (sentinel value 7) are flagged by the $R$ criterion. The predicted surface correctly ranks configurations within each block, with held-out cells closely matching their actual values.}}
    \label{fig:loss-fit}
    \vspace{-6mm}
\end{figure}

\begin{table}[t]
% \color{blue} % disabled for arXiv
\centering
\small
\setlength{\tabcolsep}{5pt}
\renewcommand{\arraystretch}{1.2}
\caption{Held-out benchmark (5-fold configuration-level cross-validation).  All five baselines and the proposed method are evaluated under the same outer-fold protocol. Higher $R^2$, Spearman, and Top-2 are better; lower Rel.Err.\ and Regret are better. Metric definitions (Top-$k$ Hit Rate, Regret) are given in Appendix Sec.~\ref{sec:metric-def}. \textbf{\#p} counts fitted coefficients, including the intercept.}
\label{tab:main-heldout-benchmark}
\begin{tabular}{@{}l c rcccc@{}}
\toprule
\textbf{Method} & \textbf{\#p} & \textbf{$R^2$} & \textbf{Rel.Err.(\%)} & \textbf{Spearman} & \textbf{Top-2 (\%)} & \textbf{Regret} \\
\midrule
Constant (block mean) & -- & 0.991 & 1.08 & $-$0.29 & 3 & 0.103 \\
Chinchilla ($A/N^\alpha\!+\!B/D^\beta\!+\!E$) & 5 & 0.988 & 1.28 & $-$0.16 & 3 & 0.103 \\
Linear regression & 4 & 0.556 & 9.51 & 0.38 & 54 & 0.021 \\
Tissue & 7 & 0.973 & 1.81 & 0.61 & 74 & 0.012 \\
Simplified + Chin.\ scale & 11 & 0.992 & 1.01 & 0.68 & 77 & 0.008 \\
\midrule
\textbf{Generalized Opt-Law} & \textbf{16} & \textbf{0.998} & \textbf{0.50} & \textbf{0.84} & \textbf{94} & \textbf{0.003} \\
\bottomrule
\end{tabular}
\end{table}

\blue{We compare against five baselines, ranging from a parameter-free constant predictor to a schedule-aware cross-scale model. The \emph{Constant} baseline predicts the per-block mean loss (high $R^2$ but no within-block ranking). The \emph{Chinchilla} baseline~\citep{hoffmann2022training} uses the standard power-law form without schedule-dependent features. \emph{Linear regression} uses normalized $(\eta_{\max}, a_1, N, S)$ as inputs. The \emph{Tissue} baseline uses the cross-scale formula of \citet{tissue2025scaling}, which combines cumulative learning-rate and annealing-area features with a model-size interaction term. Because the between-block variation in loss (driven by model size and token budget) dominates $R^2$, a high $R^2$ does not guarantee accurate within-block ranking of hyper-parameter configurations. The selection-oriented metrics (Spearman, Top-2, Regret) directly measure this ranking quality and are therefore more informative for practical schedule selection.

On all five performance metrics, the generalized \abbr{} improves upon all baselines. The Constant and Chinchilla baselines model only scale factors ($N$, $S$) without schedule-dependent features. Despite achieving $R^2 > 0.988$, their Spearman correlation is near zero, as they lack the capacity to rank configurations within the same (model size, token budget) block. Linear regression incorporates schedule variables ($\eta_{\max}$, $a_1$) but uses a linear combination that cannot capture the nonlinear trade-off between convergence and escape. The Tissue baseline introduces cumulative learning-rate and annealing-area features, yielding moderate ranking ability (Top-2 74\%), but does not include escape or convergence-escape interaction terms. The generalized \abbr{}, whose features are derived from the SDE convergence and escape bounds, achieves Top-2 94\% and Regret 0.003, meaning that the selected configuration is the best or second-best in nearly all blocks and the expected sub-optimality in loss is below 0.3\%. Progressive ablation (Appendix Sec.~\ref{sec:ablation}) corroborates this result, showing that each feature block contributes independent predictive signal.}

\blue{As a separate classification task, divergence prediction is evaluated on all 260 training-grid configurations. The $R$ criterion achieves precision${=}0.96$, recall${=}0.88$, F1${=}0.92$, and balanced accuracy${=}0.94$, substantially outperforming a peak-LR threshold (F1${=}0.38$) and logistic regression (F1${=}0.54$). The full comparison is reported in Appendix Table~\ref{tab:divergence-benchmark}.}

\section{\blue{Evaluation beyond the Training Grid}}\label{sec:exp}
This section \blue{evaluates the generalized \abbr{} beyond the training grid used for fitting.}
In all experiments, we \blue{used} an $8\times0.6$B MoE model~\citep{zhao2024longskywork,wei2024skywork} with approximately 4B trainable parameters. Additionally, $8\times0.1$B and $8\times0.3$B MoE models, containing 0.5B and 2B learnable parameters, respectively, were employed in the pre-training experiments. All experiments were conducted with consistent token lengths and a batch size of 2048, using pre-training data sourced from the RedPajama-v2 dataset~\citep{together2023redpajama}. For continual training, over 100B tokens from Chinese Common Crawl data were incorporated, while fine-tuning involved sampling an additional 60B+ tokens from the Stack-Repo Java code dataset~\citep{shrivastava2023repofusion}.
\blue{Model configurations and fitted coefficients are provided in Appendix Sec.~\ref{sec:alpha-protocol}.
%~\ref{sec:exp-supplement}.
}

\blue{\paragraph{Roadmap.} The empirical results are organized in three parts.  We first study actual-loss behavior under alternative schedule families (Sec.~\ref{sec:exp:eta-effect}).  We then evaluate cross-scale extrapolation in large-scale pre-training (Sec.~\ref{sec:pretrain-extrap}).  Finally, we present downstream ranking studies in continual training (Sec.~\ref{sec:ct-ranking}) and fine-tuning (Sec.~\ref{sec:exp-ft}).  Each subsection below explicitly separates the description of observed loss behavior from the quantitative evaluation of prediction and selection quality.}

\subsection{Effects of Learning Rate Schedules}\label{sec:exp:eta-effect}
\begin{figure}[t]
    \centering
    \begin{minipage}{\textwidth}
        \subfigure[]{
            \includegraphics[width=0.48\linewidth]{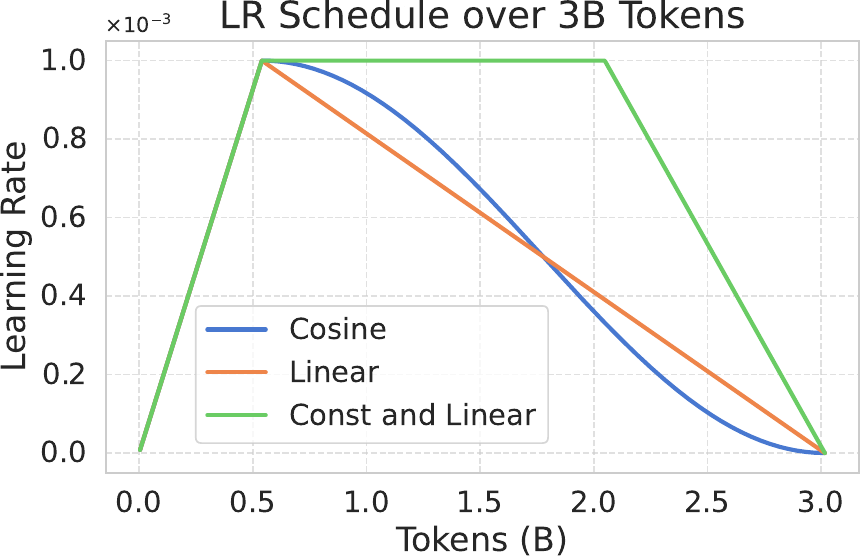}
        }
        \subfigure[]{
            \includegraphics[width=0.48\linewidth]{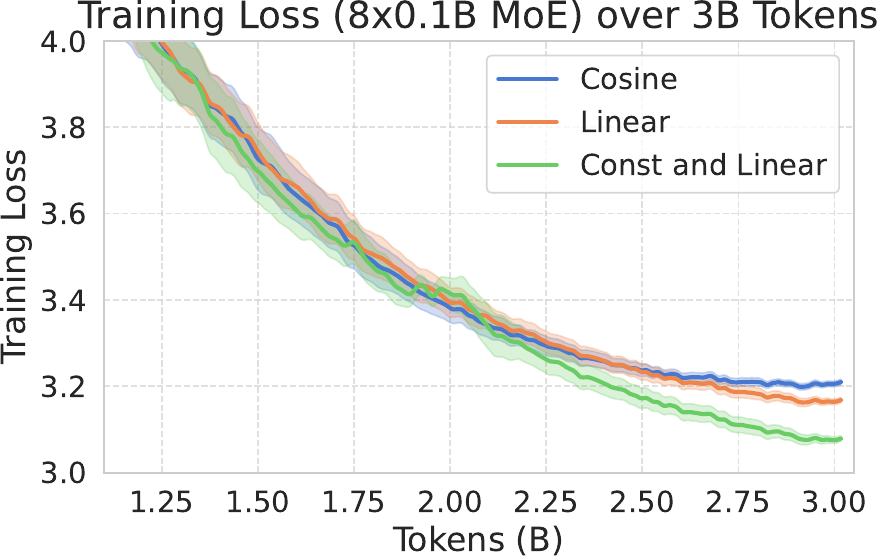}
        }
    \end{minipage}
    \vspace{-3mm}  % 调整这里的值来增加或减少距离
    \begin{minipage}{\textwidth}
        \subfigure[]{
        \includegraphics[width=0.48\linewidth]{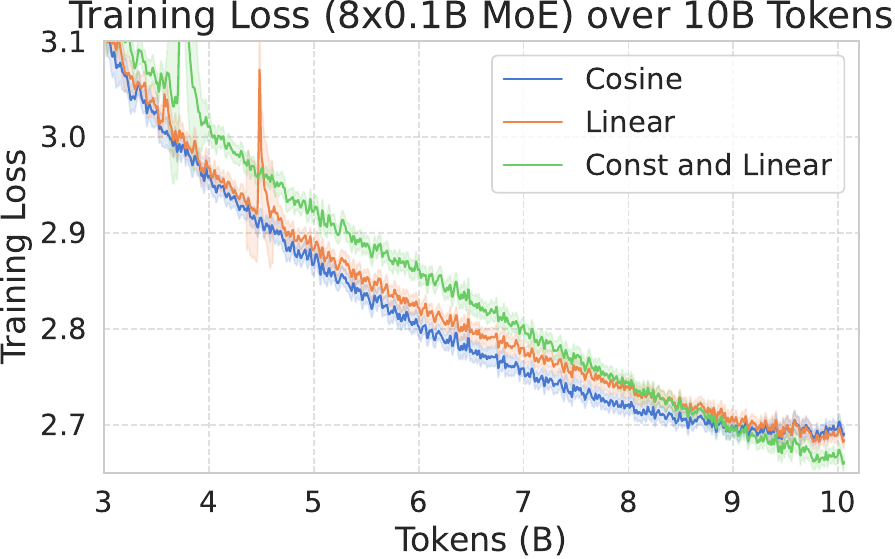}
        }
        \subfigure[]{
        \includegraphics[width=0.48\linewidth]{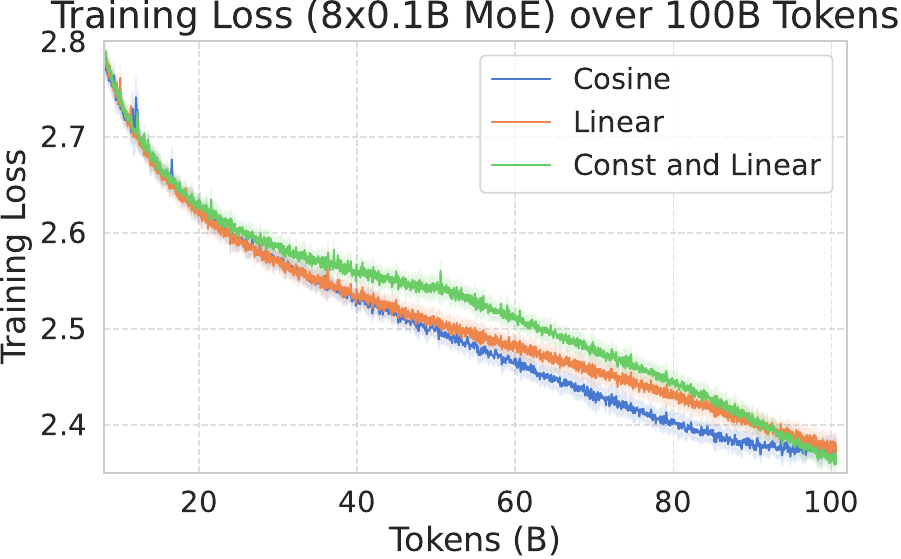}
        }
    \end{minipage}
    \vspace{0mm}  % 调整这里的值来增加或减少距离
    \begin{minipage}{\textwidth}
        \subfigure[]{
        \includegraphics[width=0.48\linewidth]{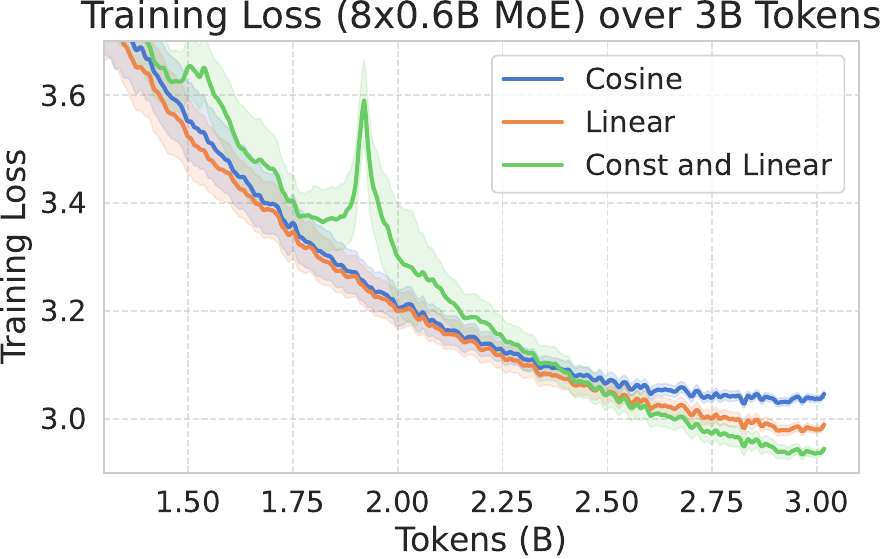}
        }
        \subfigure[]{
        \includegraphics[width=0.48\linewidth]{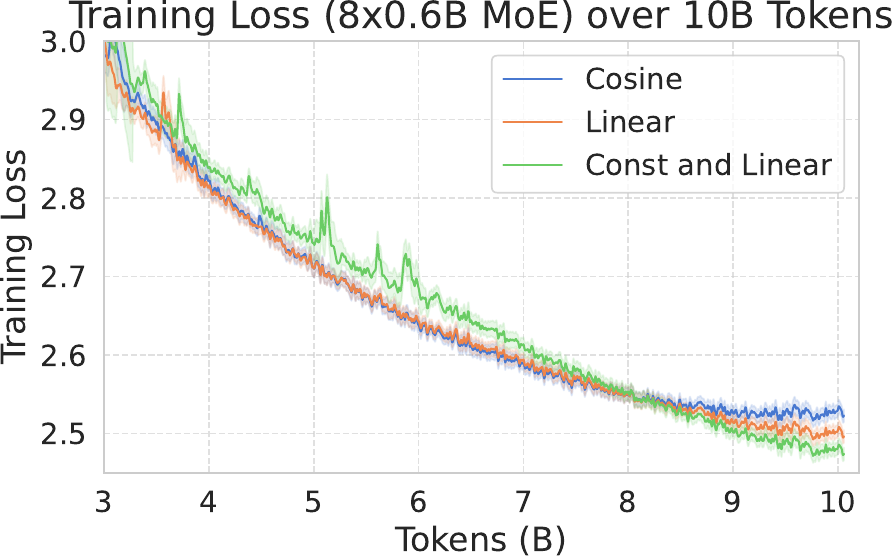}
        }
    \end{minipage}
    \caption{Training loss comparison for $8 \times 0.1$B and $8 \times 0.6$B MoEs under three LR schedules (linear decay, cosine decay, and constant followed by linear decay) across pre-training scales of 3B, 10B, and 100B tokens. \blue{The loss gap across schedule families is larger at smaller token budgets and closes more slowly for the larger model.}}
    \label{fig:effect}
    \vspace{-3mm}
\end{figure}
\blue{This subsection compares three LR schedule families (linear decay, cosine decay, and constant-plus-linear decay) across data scales of 3B, 10B, and 100B tokens. The corresponding cross-family prediction accuracy is evaluated quantitatively at the end of Sec.~\ref{sec:pretrain-extrap}.} Each schedule began with a linear warmup phase, gradually increasing from zero to a peak LR, $\eta_{\operatorname{max}}$. The specific mathematical formulations are provided in Eqn.~\eqref{eq:cosine} and Eqn.~\eqref{eq:minicpm}.

\blue{As shown in Fig.~\ref{fig:effect}, the three schedule families produce notably different final losses at the 3B token scale, where the constant-plus-linear schedule achieves the lowest loss. As the training data increases to 10B and 100B tokens, the gap across families narrows, consistent with the analysis in Sec.~\ref{sec:effect}: a larger token budget increases $\int\eta\,dt$ for all families, making the convergence terms dominant and reducing the relative contribution of the escape terms that differentiate the schedules.

However, this narrowing is slower for larger models. Comparing the $8{\times}0.1$B results (Fig.~\ref{fig:effect}(b,c)) with the $8{\times}0.6$B results (Fig.~\ref{fig:effect}(e,f)) at matched token budgets, the larger model retains a wider loss spread across schedule families. For instance, at 10B tokens the cosine and linear schedules nearly converge for the $8{\times}0.1$B model, but still differ visibly for the $8{\times}0.6$B model. This is expected from the \abbr{} perspective: the convergence terms scale with $N/\int\eta$, so a larger $N$ amplifies the schedule sensitivity at the same token budget.

Among the families, linear decay consistently achieves a slightly lower final loss than cosine decay at equal token budgets (Fig.~\ref{fig:effect}(b,c)). The two schedules share the same convergence integral $\int\eta\,dt$ but differ in the escape term $\int{\eta'_t}^2$: the linear schedule has a smaller derivative integral, corresponding to a lower trapping probability under the escape analysis (Sec.~\ref{sec:effect}). The observed loss difference is consistent with this prediction.

These results confirm that the convergence-escape decomposition in \abbr{} captures the key factors governing schedule selection: data volume controls the convergence terms, model scale modulates their sensitivity, and the escape terms account for residual differences among families with matched convergence budgets.}
\subsection{\blue{Pre-training Extrapolation}}\label{sec:pretrain-extrap}

\blue{In this subsection, we evaluate generalized \abbr{} as an extrapolation predictor for large-scale pre-training. The law is fitted on the small-scale piecewise-linear grid (Sec.~\ref{sec:apply-opt-laws}) and applied without re-fitting to larger models, larger token budgets, and more complex schedule shapes.}

\blue{We consider an $8\times 0.6$B and an $8\times 0.3$B MoE model, pre-trained on up to 300B tokens. Table~\ref{tab:pre-train} reports the actual and predicted final losses across three groups of configurations with varied LRs $\eta_1$, $\eta_2$ and schedule parameters $a_1$, $a_2$, $a_3$ (Fig.~\ref{fig:lr_illustration}). The fitted exponents and coefficients are listed in Appendix Sec.~\ref{sec:alpha-protocol}.
%\ref{sec:exp-supplement}.
}

\begin{table}[t]
\centering
\caption{\blue{Actual versus predicted pre-training loss for the $8{\times}0.6$B and $8{\times}0.3$B MoE models at 100B--300B tokens. Three groups are tested: single-phase schedules at 300B tokens, multi-phase schedules at 100B tokens, and loss-equivalent schedules identified by inverse design. All prediction errors are below 0.5\%.}}
\label{tab:pre-train}
\vspace{2pt}
\small
\setlength{\tabcolsep}{5pt}
\renewcommand{\arraystretch}{1.2}
\begin{tabular}{@{}ll  cc  rrr  rrr@{}}
\toprule
\multicolumn{2}{c}{\textbf{Configuration}}
  & \multicolumn{5}{c}{\textbf{LR Schedule Parameters}}
  & \multicolumn{3}{c}{\textbf{Loss}} \\
\cmidrule(lr){1-2}\cmidrule(lr){3-7}\cmidrule(l){8-10}
\textbf{Model} & \textbf{Tokens}
  & $\eta_1$ & $\eta_2$ & $a_1$ & $a_2$ & $a_3$
  & \textbf{Actual} & \textbf{Predicted} & \textbf{Err.\,(\%)} \\
\midrule
\multicolumn{10}{l}{\textit{\small 300B tokens} \textemdash{} \textit{\small single-phase schedules}} \\[1pt]
$8{\times}$0.6B & 300B & $10^{-3}$          & $10^{-3}$          & 500  & 500   & 500   & 1.985 & 1.984 & 0.04 \\
$8{\times}$0.6B & 300B & $6{\times}10^{-3}$ & $6{\times}10^{-3}$ & 2000 & 2000  & 2000  & 1.996 & 1.995 & 0.05 \\
$8{\times}$0.3B & 300B & $6{\times}10^{-3}$ & $6{\times}10^{-3}$ & 500  & 500   & 500   & 2.073 & 2.073 & 0.00 \\
\midrule
\multicolumn{10}{l}{\textit{\small 100B tokens} \textemdash{} \textit{\small multi-phase schedules}} \\[1pt]
$8{\times}$0.6B & 100B & $6{\times}10^{-4}$   & $6{\times}10^{-4}$ & 1200 & 1200  & 10000 & 2.076 & 2.075 & 0.04 \\
$8{\times}$0.6B & 100B & $1.2{\times}10^{-3}$ & $6{\times}10^{-4}$ & 1200 & 7000  & 13000 & 2.098 & 2.095 & 0.13 \\
$8{\times}$0.6B & 100B & $1.2{\times}10^{-3}$ & $6{\times}10^{-4}$ & 1200 & 5000  & 11500 & 2.097 & 2.095 & 0.08 \\
$8{\times}$0.6B & 100B & $10^{-3}$            & $5{\times}10^{-4}$ & 5000 & 10000 & 15000 & 2.057 & 2.064 & 0.35 \\
\midrule
\multicolumn{10}{l}{\textit{\small 100B tokens} \textemdash{} \textit{\small loss-equivalent schedules}} \\[1pt]
$8{\times}$0.6B & 100B & $10^{-3}$          & $10^{-3}$          & 2000 & 2000  & 2000  & 2.077 & 2.078 & 0.05 \\
$8{\times}$0.6B & 100B & $10^{-3}$          & $5{\times}10^{-4}$ & 2450 & 7000  & 12000 & 2.079 & 2.080 & 0.05 \\
$8{\times}$0.6B & 100B & $5{\times}10^{-5}$ & $5{\times}10^{-4}$ & 1000 & 1000  & 9500  & 2.077 & 2.078 & 0.05 \\
\bottomrule
\end{tabular}
\end{table}

\begin{table}[t]
% \color{blue} % disabled for arXiv
\centering
\small
\setlength{\tabcolsep}{7pt}
\renewcommand{\arraystretch}{1.2}
\caption{\blue{Baseline comparison for pre-training extrapolation. All methods are fitted on the small-scale training grid (Table~\ref{tab:data-inventory}) and evaluated on the configurations in Table~\ref{tab:pre-train} without re-fitting. Mean relative error~(\%) is reported per group.}}
\label{tab:pt-baseline}
\vspace{2pt}
\begin{tabular}{@{}l cccc@{}}
\toprule
& \multicolumn{3}{c}{\textbf{Per-Group Rel.\ Err.\ (\%)}} & \\
\cmidrule(lr){2-4}
\textbf{Method}
  & \textbf{300B-SP\,(3)} & \textbf{100B-MP\,(4)} & \textbf{100B-LE\,(3)}
  & \textbf{All\,(10)} \\
\midrule
Linear regression & 63.8 & 44.8 & 44.7 & 50.8 \\
Tissue cross-scale & 11.3 & 9.7 & 10.2 & 10.3 \\
Simplified + Chin.\ scale & 8.9 & 6.9 & 7.2 & 7.6 \\
\midrule
\textbf{Generalized Opt-Law} & \textbf{0.1} & \textbf{0.2} & \textbf{0.2} & \textbf{0.2} \\
\bottomrule
\multicolumn{5}{@{}l}{\footnotesize SP: single-phase\,; MP: multi-phase\,; LE: loss-equivalent.} \\
\end{tabular}
\end{table}

\begin{figure}[t]
    \centering
    \begin{minipage}{\textwidth}
        \subfigure[]{
            \includegraphics[width=0.48\linewidth]{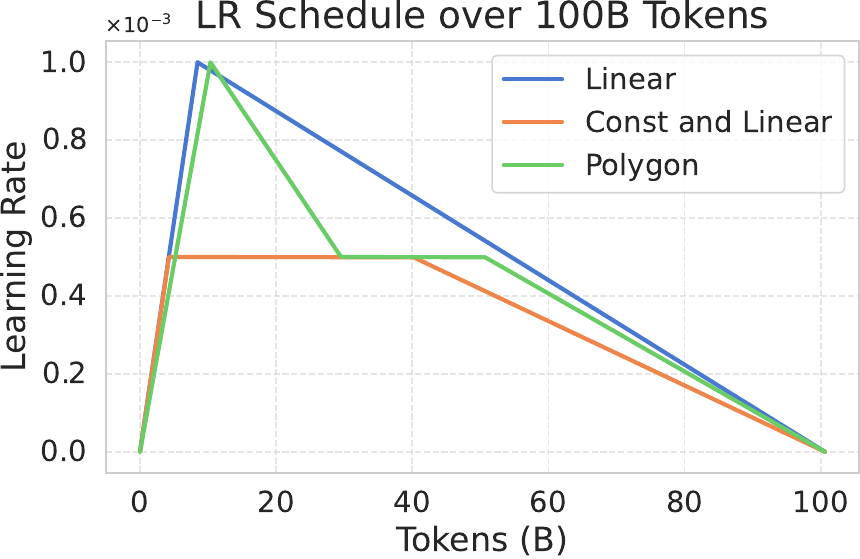}
        }
        \subfigure[]{
            \includegraphics[width=0.48\linewidth]{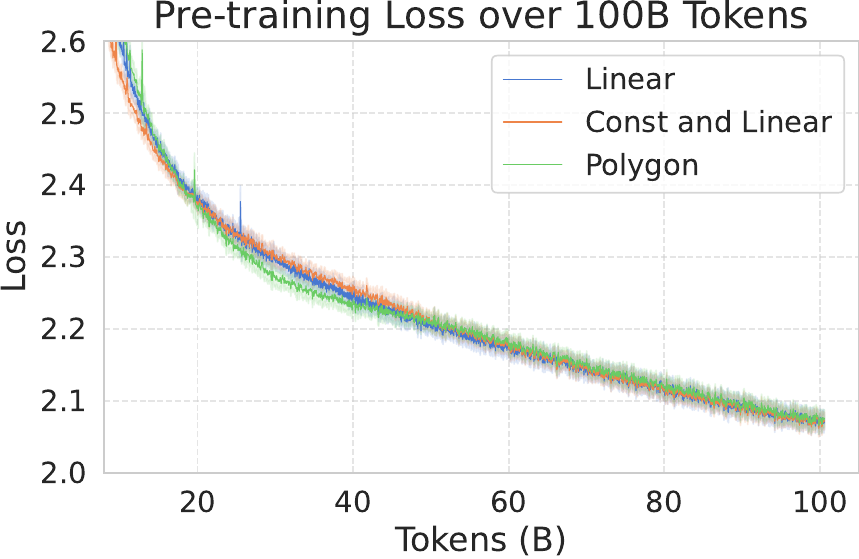}
        }
    \end{minipage}
    \caption{\blue{(a) Three LR schedules for the $8{\times}0.6$B model on 100B tokens, selected by inverse design to yield the same predicted loss. (b) The corresponding training loss curves. The final losses converge to nearly the same value, confirming that \abbr{} correctly identifies loss-equivalent configurations.}}
    \label{fig:sameloss}
    \vspace{-4mm}
\end{figure}

\blue{The first group in Table~\ref{tab:pre-train} evaluates predictions at 300B tokens under simple warmup-plus-cooldown schedules. Across different model sizes and peak LRs, prediction errors remain below 0.2\%. The second group tests multi-phase piecewise-linear schedules at 100B tokens. Despite the increased schedule complexity, all errors stay below 0.5\%. These results are obtained without re-fitting, confirming that the small-scale calibration transfers to configurations well beyond the training grid.}

\blue{As a further test,} we \blue{used} generalized \abbr{} \blue{to} infer specific combinations of $\eta_1$, $\eta_2$, $a_1$, $a_2$, and $a_3$ (last three rows of Table~\ref{tab:pre-train}) that would yield the same predicted loss. We then pre-trained the $8\times 0.6$B MoE on 100B tokens using these schedules. Fig.~\ref{fig:sameloss} shows that the training losses converge to nearly the same final value across all three schedules, \blue{confirming that the law correctly identifies loss-equivalent configurations even when the schedules differ substantially in shape.}

\blue{Table~\ref{tab:pt-baseline} compares the generalized \abbr{} against three schedule-aware baselines on the same Table~\ref{tab:pre-train} configurations. All methods are fitted on the small-scale training grid, which contains only linear warmup-plus-cooldown schedules, and are evaluated without re-fitting. Linear regression uses $(\eta_{\max}, a_1, N, S)$ as inputs and produces mean errors above 44\%, as the linear form cannot capture the nonlinear relationship between schedule features and loss at new scales. The Tissue cross-scale baseline incorporates cumulative learning rate and annealing features, reducing errors to approximately 10\%, but its functional form does not separate convergence and escape contributions, limiting extrapolation precision. The simplified Opt-Law with Chinchilla-style scale terms achieves 7--9\% errors. By contrast, the generalized \abbr{} maintains errors below 0.5\% on all three groups. The gap widens at 300B tokens (0.1\% vs.\ 8.9--63.8\%), where the extrapolation distance from the training grid is largest, suggesting that the SDE-derived convergence-escape decomposition provides a more robust basis for cross-scale transfer than the features used by the baselines.}

\blue{\paragraph{Schedule-family generalization.}
All results above involve piecewise-linear schedules. The training grid contains only linear warmup-plus-cooldown schedules, and the Table~\ref{tab:pre-train} configurations extend this to multi-phase piecewise-linear variants. To test generalization to entirely different schedule shapes, we evaluate the generalized \abbr{} on cosine and constant-plus-linear schedules. The generalized \abbr{} correctly identifies the optimal schedule family in all five evaluated groups, with prediction errors below 2\%. The detailed per-group results are reported in Appendix Table~\ref{tab:schedule-transfer}. This cross-family transfer reflects the design of the feature construction: the convergence and escape features capture the mechanisms through which the schedule affects training, rather than the parametric form of the schedule itself. Specifically, $\int\eta$ measures how much total learning the optimizer accumulates, and $\int{\eta'_t}^2$ measures how aggressively the learning rate changes during warmup and cooldown. Different families that deliver similar values of these quantities produce similar predicted losses. We do not claim universal validity for arbitrary schedules, but these results provide evidence that the SDE-derived features are not restricted to the family used for fitting.}

\subsection{\blue{Continual-Training Ranking}}\label{sec:ct-ranking}
\blue{This subsection evaluates whether the generalized \abbr{} can rank candidate LR schedules for continual training, where a pre-trained model is further trained on new data with a re-warmed learning rate.}

\blue{We use} an $8\times0.6$B MoE model pre-trained on 300B English tokens from the RedPajama-v2 dataset (Table~\ref{tab:pre-train}, row one). We evaluate two scenarios: strong distribution shift (100B Chinese tokens) and weak distribution shift (100B English tokens).

\begin{figure}[t]
    \centering
    \begin{minipage}{\textwidth}
        \subfigure[]{
            \includegraphics[width=0.48\linewidth]{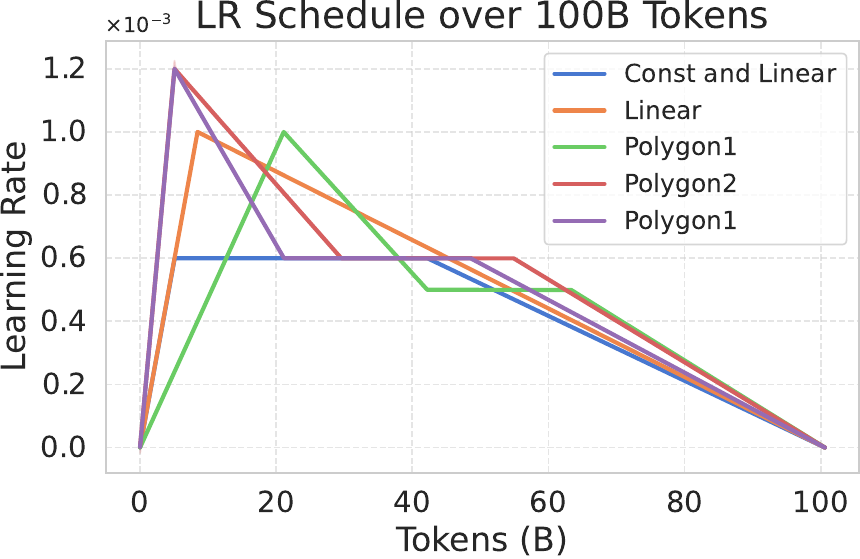}
        }
        \subfigure[]{
            \includegraphics[width=0.48\linewidth]{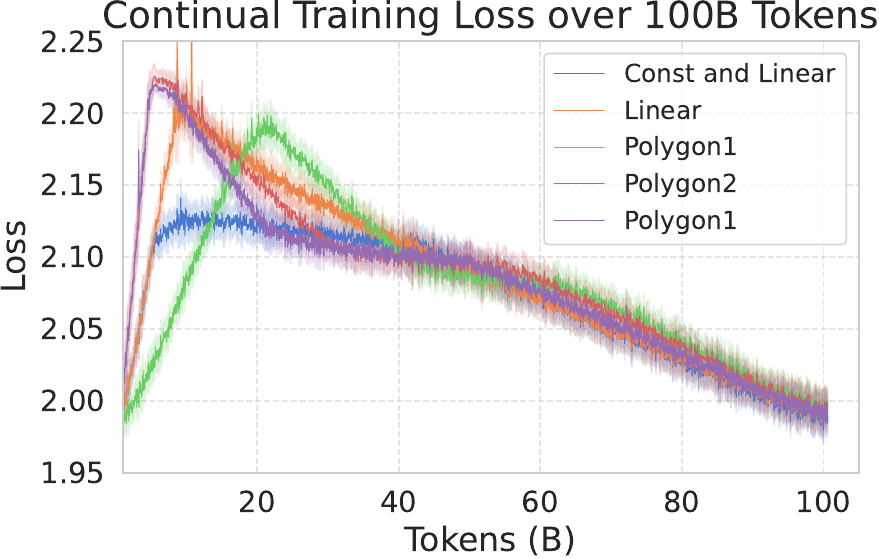}
        }
    \end{minipage}
    \caption{\blue{(a) Five LR schedules for continual training of the $8{\times}0.6$B model on 100B English tokens (weak distribution shift). (b) The corresponding loss curves. Final losses are closely clustered despite substantial schedule variation.}}
    \label{fig:c-EN-loss}
    \vspace{-4mm}
\end{figure}

\begin{table}[t]
% \color{blue} % disabled for arXiv
\centering
\caption{\blue{Ranking evaluation for continual training of the $8{\times}0.6$B model on 100B tokens under strong shift (Chinese) and weak shift (English). Rank (Act./Pred.): actual versus predicted rank. Pred.\ Score is used for ranking only; lower score indicates lower predicted loss.}}
\label{tab:continual-train}
\vspace{2pt}
\small
\setlength{\tabcolsep}{3pt}
\renewcommand{\arraystretch}{1.2}
\begin{tabular}{@{}c l cc rrr rr cc@{}}
\toprule
& \textbf{Schedule}
  & \multicolumn{5}{c}{\textbf{LR Parameters}}
  & \textbf{Obs.} & \textbf{Pred.}
  & \multicolumn{2}{c}{\textbf{Rank}} \\
\cmidrule(lr){3-7}\cmidrule(lr){8-8}\cmidrule(lr){9-9}\cmidrule(l){10-11}
& & $\eta_1$ & $\eta_2$ & $a_1$ & $a_2$ & $a_3$
  & \textbf{Loss} & \textbf{Score}
  & \textbf{Act.} & \textbf{Pred.} \\
\midrule
\multicolumn{11}{@{}l}{\textit{Strong distribution shift ($8{\times}$0.6B, 100B Chinese tokens)}} \\[1pt]
\ding{172} & Polygon1 & $10^{-3}$            & $5{\times}10^{-4}$ & 5000  & 10000 & 15000 & 1.826 & 2.058 & 3 & 3 \\
\ding{173} & Polygon2 & $1.2{\times}10^{-3}$ & $6{\times}10^{-4}$ & 1200  & 7000  & 13000 & 1.821 & 2.014 & 2 & 2 \\
\ding{174} & Linear   & $10^{-3}$            & $10^{-3}$          & 2000  & 2000  & 2000  & 1.818 & 1.975 & 1 & 1 \\
\midrule
\multicolumn{11}{@{}l}{\textit{Weak distribution shift ($8{\times}$0.6B, 100B English tokens; near-tie)}} \\[1pt]
\ding{175} & Const+Linear & $6{\times}10^{-4}$   & $6{\times}10^{-4}$   & 1200 & 1200  & 10000 & 1.991 & 1.995 & 1 & 2 \\
\ding{176} & Linear       & $10^{-3}$            & ---                  & 2000 & 2000  & 2000  & 1.992 & 1.975 & 2 & 1 \\
\ding{177} & Polygon1v2   & $1.2{\times}10^{-3}$ & $6{\times}10^{-4}$   & 1200 & 5000  & 11500 & 1.993 & 2.010 & 3 & 3 \\
\ding{178} & Polygon2     & $1.2{\times}10^{-3}$ & $6{\times}10^{-4}$   & 1200 & 7000  & 13000 & 1.995 & 2.017 & 4 & 4 \\
\ding{179} & Polygon1     & $10^{-3}$            & $5{\times}10^{-4}$   & 5000 & 10000 & 15000 & 1.996 & 2.059 & 5 & 5 \\
\bottomrule
\end{tabular}
\end{table}

\begin{table}[t]
% \color{blue} % disabled for arXiv
\centering
\caption{\blue{Fine-tuning ranking evaluation for the $8{\times}0.6$B model at 3B, 30B, and 60B tokens. Rank (Act./Pred.) denotes actual versus predicted rank out of 3 candidates per token budget. Lower Pred.\ Score indicates a schedule predicted to achieve lower loss. All runs use warmup-and-decay schedules except \ding{174} in the 60B group (two-phase). The predicted ranking matches the actual ranking in all nine configurations.}}
\label{tab:finetune}
\vspace{2pt}
\small
\setlength{\tabcolsep}{4pt}
\renewcommand{\arraystretch}{1.2}
\begin{tabular}{@{}c l r c r rr cc@{}}
\toprule
& & \multicolumn{3}{c}{\textbf{LR Parameters}}
  & \textbf{Observed} & \textbf{Predicted}
  & \multicolumn{2}{c}{\textbf{Rank}} \\
\cmidrule(lr){3-5}\cmidrule(lr){6-6}\cmidrule(lr){7-7}\cmidrule(l){8-9}
& \textbf{Schedule} & $\eta_{\max}$ & $a_1$ & $a_3$
  & \textbf{Loss} & \textbf{Score}
  & \textbf{Act.} & \textbf{Pred.} \\
\midrule
\multicolumn{9}{@{}l}{\textit{3B tokens --- small-scale fine-tuning}} \\[1pt]
\ding{172} & Linear-S & $5{\times}10^{-4}$   & 128 & 128  & 0.920 & 2.302 & 3 & 3 \\
\ding{173} & Linear-L & $10^{-3}$            & 512 & 512  & 0.909 & 2.158 & 2 & 2 \\
\ding{174} & Linear-H & $1.5{\times}10^{-3}$ & 128 & 128  & 0.904 & 2.102 & 1 & 1 \\
\midrule
\multicolumn{9}{@{}l}{\textit{30B tokens --- medium-scale fine-tuning}} \\[1pt]
\ding{172} & Linear-S & $5{\times}10^{-4}$   & 128 & 128  & 0.857 & 2.116 & 3 & 3 \\
\ding{173} & Linear-L & $10^{-3}$            & 512 & 512  & 0.849 & 2.028 & 2 & 2 \\
\ding{174} & Linear-H & $1.5{\times}10^{-3}$ & 128 & 128  & 0.847 & 2.004 & 1 & 1 \\
\midrule
\multicolumn{9}{@{}l}{\textit{60B tokens --- large-scale fine-tuning (mixed schedules)}} \\[1pt]
\ding{172} & Linear-S & $5{\times}10^{-4}$   & 128 & 128  & 0.822 & 2.052 & 2 & 2 \\
\ding{173} & Linear-H & $2{\times}10^{-3}$   & 512 & 512  & 0.817 & 1.972 & 1 & 1 \\
\ding{174} & Two-phase & $3{\times}10^{-4}$   & 128 & 7000 & 0.826 & 2.178 & 3 & 3 \\
\bottomrule
\end{tabular}
\end{table}

\blue{Since the model starts from a pre-trained checkpoint rather than from scratch, the generalized \abbr{} must account for the optimization history. For the convergence term, the cumulative optimization budget should reflect both stages: if pre-training already accumulated a large $\int\eta^{pre}$, the model is closer to a local minimum and the marginal effect of the continual-training schedule differs from that of training from scratch. We therefore define the history-aware convergence integral as}
\[
  \int_0^{a_{c_1}^{ct}} \eta_t \, dt = \int_0^{S^{pre}} \eta_t^{pre} \, dt + \int_0^{a_{c_1}^{ct}} \eta_t^{ct} \, dt.
\]
\blue{For the escape term, the pre-training schedule is already finished and contributes no derivatives during continual training, so only the current-stage schedule shape enters. In the pre-training feature construction, the effect of $\eta_{\text{max}}$ on the escape term is implicitly absorbed by the fitted coefficients together with the scale feature $\eta_{\text{max}}^{\alpha_{12}}$. In continual training, however, the coefficients are not re-fitted, and the re-warmed peak LR may lie outside the pre-training range. To ensure the escape feature remains comparable under different LR scales, we use the normalized form $\int_{a_{e_2}}^S {{\eta^\prime_t}^2}/{\eta^4_{\text{max}}}$ directly, following the ratio that appears in the escape bound (Theorems~\ref{thm:SGD-prob} and~\ref{thm:Adam-prob}).}

In the strong-shift scenario, we used 100B tokens of Chinese common crawl data, mixed with $5\%$ of the original English data as replay~\citep{parmar2024reuse,guo2024efficient,ke2023continual}. \blue{The observed losses are well separated (spread $>$0.4\%), and the modified \abbr{} recovers the exact 1--2--3 ranking (Table~\ref{tab:continual-train}, top). The linear schedule achieves the lowest loss, followed by Polygon2 and Polygon1. Notably, this ranking differs from the pre-training ranking under the same schedules, confirming that the history-aware modification captures the changed optimization dynamics rather than simply reproducing the pre-training ordering.}

In the weak-shift scenario, we sampled 100B English tokens and tested five LR schedules (Fig.~\ref{fig:c-EN-loss}). \blue{The final losses differ by only 0.005 (0.25\% relative), consistent with prior observations that schedule choice has limited impact under small distribution shift~\citep{ibrahim2024simple}. The predicted scores correctly recover positions 3--5, while the top two schedules are separated by only 0.001 in observed loss (Table~\ref{tab:continual-train}, bottom). This near-tie setting illustrates the resolution limit of the fitted law.}

\blue{In both scenarios, the history-aware \abbr{} produces correct rankings wherever the observed loss differences are resolvable, providing reliable schedule selection for continual training.}

\subsection{\blue{Fine-Tuning Ranking}}\label{sec:exp-ft}

\blue{This subsection extends the ranking evaluation to fine-tuning, the final stage of a multi-stage training pipeline. In practice, production LLMs are typically not fine-tuned directly from a base pre-training checkpoint~\citep{dubey2024llama}. Continual training on domain-specific corpora usually precedes task-level fine-tuning, so we test the history-aware \abbr{} on the full three-stage pipeline (pre-training $\rightarrow$ continual pre-training $\rightarrow$ fine-tuning), evaluating whether the law remains effective after multiple training stages and data shifts.}

In our experiments, we fine-tuned an $8\times 0.6$B MoE model on the Stack-Repo Java code dataset. Prior to fine-tuning, the model had been pre-trained on 300B English tokens and subsequently underwent continual training on 100B Chinese tokens. We conducted fine-tuning with three token budgets (3B, 30B, and 60B) across three LR schedules per budget, as detailed in Table~\ref{tab:finetune}.

\blue{Following the same history-aware approach as in Sec.~\ref{sec:ct-ranking}, the convergence integral accumulates the optimization budget from all preceding stages: $\int_0^{a_{c_1}^{ft}} \eta_t \, dt = \int_0^{S^{pre}} \eta_t^{pre} \, dt + \int_0^{S^{ct}} \eta_t^{ct} \, dt + \int_0^{a_{c_1}^{ft}} \eta_t^{ft} \, dt$. The escape term uses the current-stage schedule only, with the same $\eta^4_{\text{max}}$ normalization as in Sec.~\ref{sec:ct-ranking}.}

\blue{At 3B and 30B tokens, the predicted scores correctly rank all three candidates (Table~\ref{tab:finetune}). The predicted score gap between the best and worst schedule narrows from 0.20 at 3B to 0.11 at 30B, mirroring the observed loss gap (0.016 at 3B versus 0.010 at 30B). This parallel narrowing indicates that the law not only preserves the ranking but also tracks the diminishing sensitivity of loss to schedule choice as the fine-tuning budget grows. At 60B tokens, one candidate (\ding{174}) uses a two-phase schedule rather than simple warmup-and-decay. The law still correctly identifies the best and worst candidates, indicating that the history-aware construction generalizes across schedule types even after three training stages. Across all nine fine-tuning configurations, the predicted and actual rankings agree perfectly.}

\blue{These results demonstrate that the history-aware \abbr{} retains ranking utility across a realistic multi-stage pipeline spanning pre-training, continual training, and fine-tuning.}

\section{\blue{SDE Analysis behind Opt-Laws}}\label{sec:sde}
\blue{This section presents the SDE-based analysis that motivates the feature construction in Secs.~3--4.} We analyze the training dynamics of non-convex smooth optimization problems for both SGD and Adam~\citep{kingma2014adam}, modeling them as SDEs to capture the effects of time-varying learning rates.

Our analysis examines the influence of key hyper-parameters, including the LR schedule $\eta(t)$, data size $D$, and model size $N$, on two critical aspects: convergence speed and the ability to escape suboptimal local minima. Notably, we observe that both SGD and Adam yield similar forms in their bounds related to these hyper-parameters. This similarity allows us to generalize \abbr{}, encapsulating the relationship between these factors and the final training loss, thereby offering deeper insights into the optimization process.

While existing research frequently uses SDEs to model optimization processes and analyze convergence~\citep{li2017stochastic,soto2022sde,dambrine2024stochastic}, or to study escape time from local minima (e.g., the Eyring--Kramers law~\citep{berglund2013sharp,bovier2015metastability}), these studies are predominantly restricted to convex problems, quadratic objectives, or fixed learning rates. In contrast, our work extends to general non-convex smooth problems with time-varying LRs, thus leading to a nonlinear, time-inhomogeneous SDE. This added complexity precludes the direct application of existing results. Our contributions address these new theoretical challenges, providing novel insights into the optimization of LLMs under more general conditions.

\subsection{Optimization Methods and SDEs}
We study the following  non-convex optimization problem: 
\begin{equation}\label{eq:problem}
\min_{\*x}   f(\bm{\*x}) \coloneqq  \bb E_{\bm{\zeta}\sim \$D}\qty[F(\bm{\*x},\bm{\zeta})],  
\end{equation}
where the objective function $F(\cdot,\cdot)$ is differentiable and possibly non-convex, data $\bm{\zeta}$ is drawn from an unknown data distribution $\$D$, $\bm{\*x}$ is the learnable parameters.  
The formulation Eqn.~\eqref{eq:problem} encapsulates a large body of machine learning problems, e.g., LLM training problems, and least square regression.  
For SGD, the update scheme is as follows:
\begin{equation}\label{eq:sgd-adam}
    \quad \*x_{k+1} =    \*x_{k}  - \eta_0\eta_k {\tilde{\*g}}_k,
    \tag{SGD}
\end{equation}
where gradient estimate $\tilde{\*g}_k$ is defined as $\tilde{\*g}_k \coloneqq \nabla f(\*x_k) + \*z_k$, with a LR of $\eta_0\eta_k$, where $\eta_0$ is a small rescaling parameter and $\eta_k$ is the normalized LR. This choice of LR facilitates the derivation of the corresponding SDE. Following previous work~\citep{zhu2019anisotropic,malladi2022sdes,xie2020diffusion,zhou2024towards}, we assume $\*z_k \sim \mathcal{N}\qty(0,\bm{\Sigma}\qty(\*x_k))$, where
\begin{equation}\label{eq:sigma} \bm{\Sigma}\qty(\*x_k) \coloneqq \frac{1}{B}\qty(\frac{1}{D}\sum_{i=1}^D \qty(\nabla F(\*x_{k},\bm{\zeta}_i) - \nabla f(\*x_k))\qty(\nabla F(\*x_{k},\bm{\zeta}_i) - \nabla f(\*x_k))^\top). 
\end{equation}
We should mention that the exact characterization of the noise $\*z_k$ in stochastic optimization remains unresolved~\citep{simsekli2019tail,haochen2021shape}. For analytical convenience, here we adopt a Gaussian noise assumption, acknowledging that determining the precise stochastic process of the noise during LLM training is beyond the scope of this work.

Under mild conditions, SGD can be approximated by an \ito-SDE, where the noise is modeled by Brownian motion, as derived via the Euler-Maruyama method in Sec.~\ref{sec:sgd-sde-derivation}:
\begin{equation}
\dd \*X_t = - \eta(t)\nabla f(\*X_t) \dd t + \sqrt{\eta_0}\eta(t)\sigma(\*X_t)   \dd \*W_t.
\label{eq:SGD-SDE}
\tag{SGD-SDE}
\end{equation}
where $\*W_t$ is a Wiener process, $\sigma: \mathbb{R}^N \to \mathbb{R}^{N \times N}$ is defined as $\sigma(\*x) \coloneqq \sqrt{\bm{\Sigma}(\*x)}$, $\eta(t)$ is the normalized LR schedule, and $\eta_0$ is a small rescaling parameter. Specifically, $\eta_0 \cdot \eta(t)$ corresponds to the original LR schedule.

Similarly, the dynamics of the Adam  can be approximated using a lifted \ito-SDE. The algorithmic steps for Adam are presented as follows:
\begin{equation}\label{eq:adam}
    \begin{cases}
        \*x_{k+1} = \*x_k - \eta_0\eta_k {\*m_k} \odot \qty({{\*v_k+\epsilon}})^{-\frac{1}{2}},\\
    	\*m_{k+1} =  (1-\beta_{1,k})  \*m_{k} +  \beta_{1,k} \tilde{\*g}_{k+1}, \\
     \*v_{k+1} = (1-\beta_{2,k}) \*v_{k} + \beta_{2,k} \tilde{\*g}_{k+1}^2, 
    \end{cases}
    \tag{Adam}
\end{equation}
 where $\odot$ denotes the element-wise product, with $\beta_{1,k}, \beta_{2,k} \in (0,1)$ and initial conditions $\*m_0 = \*v_0 = \bm{0}$. The parameter $\epsilon$ prevents degeneracy. %For analytical simplicity, we modify the preconditioner during the update of $*x$. 
 In practice, both $\beta_{1,k}$ and $\beta_{2,k}$ are typically close to 1. To facilitate SDE-based analysis, we use the parameters $\beta_{1,k} = 1 - \hat{c}_1\eta_k$ and $\beta_{2,k} = 1 - \hat{c}_2\eta_k$, following a single-time scale scheme~\citep{ding2023adam,shen2022single,xiao2023convergence}, where $\hat{c}_1$ and $\hat{c}_2$ are small constants. The resulting \ito-SDE for  Adam's dynamics is:
\begin{equation}
\left\{ 
\begin{aligned}
\dd \*X_t &= - \eta(t)\*m_t\odot\qty(\*v_t+\epsilon)^{-\frac{1}{2}}\dd t,\\
\dd\*m_t&=-c_1\eta(t)(\*m_t-\nabla f(\*X_t))\dd t+c_1'\eta(t)\sigma(\*X_t)\dd \*W_t,\\
\dd \*v_t&=-c_2\eta(t)(\*v_t-{\rm diag}\qty(\*\Sigma(\*X_t)))\dd t,
\end{aligned}\right.
\label{eq:Adam-SDE}
\tag{Adam-SDE}
\end{equation}
where $c_1$, $c_1'$, and $c_2$ are constants. The notation ${\rm diag}(\*M)$ refers to the vector formed by the diagonal elements of matrix $\*M$, while ${\rm Diag}(\*v)$ denotes the diagonal matrix with the entries of $\*v$ on its diagonal. Given the initial condition $\*v_0 = \bm{0}$, the solution for $\*v_t$ is:
\[
\*v_t=\exp\qty(-c_2\int_0^t\eta(s)\dd s)\qty[\int_0^t\exp\qty(c_2\int_0^s\eta(\tau)\dd\tau)c_2\eta(s)\*d_s\dd s],
\]
where $\*d_t:={\rm diag}\qty(\*\Sigma(\*X_t))$. This formulation ensures that $\*v_t$ remains nonnegative throughout the process.

\subsection{Inspiration from Convergence Guarantee}
In this subsection, we analyze the convergence properties of \eqref{eq:SGD-SDE} and \eqref{eq:Adam-SDE} in the context of non-convex smooth problems, deriving the mathematical formula for the convergence component of \abbr{}. %The convergence rate can be expressed by the learning schedule. 
To present these results rigorously, we first provide several mild assumptions. In the following discussion, the notation $\norm{\cdot}$ denotes the Euclidean norm when applied to vectors and the Frobenius norm when applied to matrices. While $\norm{\cdot}_{\rm op}$ denotes the operator norm for matrix.

\begin{myassumption}[$L$-smoothness]\label{asm:L-smoothh}
The function $f(\cdot)$ is $L$-smooth with respect to the parameters, namely,
\[
\norm{\nabla f(\*x)-\nabla f(\*y)}\leq L \norm{\*x-\*y}, \quad \forall \*x,~\*y \in  \bb R^{N}.
\]
\end{myassumption}

\begin{myassumption}[Unbiased Estimator]\label{asm:zero-mean}
Given any $\*x \in \bb R^N$, we assume that the entries of $\nabla F(\*x,\bm{\zeta}_i) - \nabla f(\*x)$ are i.i.d. Gaussians $\mathcal{N}\qty(0,\*\Sigma_g)$ for all $i\in [D]$, where $\*\Sigma_g$ is given.
% and
% \[
% \operatorname{cov}\qty(\nabla F(\*x,\bm{\zeta}_i) - \nabla f(\*x),\nabla F(\*x,\bm{\zeta}_j) - \nabla f(\*x)) = \*0, \quad \forall i\neq j\in [D].
% \]
\end{myassumption}

Assumption \ref{asm:L-smoothh} asserts that the objective function of our optimization problem is $L$-smooth. This is a standard and mild assumption in the stochastic optimization literature, which facilitates the determination of specific convergence rates for optimization algorithms, a task that is otherwise challenging \citep{arjevani2019lower,guo2021novel,li2022restarted,xie2022optimization, xie2024adan,xie2024loco, zhou2022win}. 
Assumption~\ref{asm:zero-mean} stipulates that the noise introduced by different data points during the estimation of the gradient is unbiased and independent, given the weights of LLMs. This assumption is also common, especially when the Langevin diffusion is employed to analyze the behavior of SGD~ \citep{jastrzkebski2017three,he2019control,xie2020diffusion}. Notably, we do not assume that the gradient estimation noise originates from an oracle distribution $\mathcal{N}\qty(0,\*\Sigma)$. As indicated in Eqn.~\eqref{eq:sigma}, the covariance matrix encountered during actual training is not $\E_{\bm{\zeta}\sim \$D}[\qty(\nabla F(\*x,\bm{\zeta}_i) - \nabla f(\*x))\qty(\nabla F(\*x,\bm{\zeta}_j) - \nabla f(\*x))^\top]$. %For simplicity, we denote $\sigma_g:=\lambda_{\max}(\*\Sigma_g^{\frac{1}{2}})$.

Under these assumptions, we derive a concentration inequality for the trace of the covariance matrix $\*\Sigma(\*x_k)$ defined in Eqn.~\eqref{eq:sigma}, which is helpful for the following analysis. In this work, the batch size $B$ is fixed. For simplicity, we assume $B=1$ throughout the analysis. However, the results are readily adaptable to any fixed positive batch size $B$.

\begin{myprop}[Trace Boundedness]\label{prop:sigma_bound}
Suppose Assumption~\ref{asm:zero-mean} holds. Given any point $\mathbf{x} \in \mathbb{R}^N$ and a positive $t > 0$, the covariance matrix in Eqn.~\eqref{eq:sigma} satisfies
\[
\mathbb{P}\qty\bigg(\left|\operatorname{Tr}(\bm{\Sigma}(\mathbf{x})) - \operatorname{Tr}(\bm{\Sigma}_g)\right| \geq t) \leq 2 \exp\left(-\frac{D t^2}{4 \operatorname{Tr}(\Sigma_g^2) + 2t \sigma_g^2}\right),
\]
where $D$ is the number of samples and $\sigma_g:=\lambda_{\max}\qty(\*\Sigma_g^{\frac{1}{2}})$ is the largest eigenvalue of $\*\Sigma_g^{\frac{1}{2}}$.
\end{myprop}
We can also estimate the maximal eigenvalue of the random matrix $\bm{\Sigma}(\*x)$ by random matrix theory.
\begin{myprop}[Covariance Boundedness]\label{prop:sigma_op_bound}
Suppose Assumption~\ref{asm:zero-mean} holds, then the covariance matrix in Eqn.~\eqref{eq:sigma} satisfies
\[
\sup_{\mathbf{x} \in \mathbb{R}^N} \mathbb{E}\left[\lambda_{\max}(\bm{\Sigma}(\mathbf{x}))\right] \leq \left(1+\sqrt{\frac{D}{N}}\right) \sigma_g^2 + \frac{C \sigma_g^2}{N^{2/3}}.
\]
where $C > 0$ is a constant independent of $N$, the expectation is taken with respect to the Gaussian distribution $\mathcal{N}(0, \Sigma_g)$ specified in Assumption~\ref{asm:zero-mean}, and $\sigma_g:=\lambda_{\max}\qty(\*\Sigma_g^{\frac{1}{2}})$.
\end{myprop}
Proposition~\ref{prop:sigma_op_bound} provides a non-asymptotic extension of the classical Marchenko-Pastur theorem~\citep{bai2010spectral,ledoux2010small} from random matrix theory, offering a more precise variance estimate based on the convergence of the largest eigenvalue of covariance matrices to $(1+\sqrt{D/N})\sigma_g^2$. For most LLMs adhering to scaling laws, the ratio $D/N \gg 1$. Consequently, we can expect $\lambda_{\max}({\*\Sigma}(\*x)) = \order{\sigma_g^2\sqrt{D/N}}$ in expectation.

\subsubsection{SGD Convergence Analysis} With the smoothness and boundedness properties established, we can now derive the expected convergence rate of SGD-SDE for non-convex problems in terms of time-varying LRs. 
\begin{mytheorem}[SGD Convergence Bound]\label{thm:covergence-SGD}
Suppose Assumptions \ref{asm:L-smoothh} and \ref{asm:zero-mean} hold. For the dynamics described in \eqref{eq:SGD-SDE}, the following bound holds:
\begin{equation}
\E\qty[\overline{\norm{\nabla f(\*X_t)}^2}]\leq\frac{f(\*X_0)-f_{\min}}{\int_0^t\eta(s)\dd s}+\frac{\eta_0L\sigma_0^2N \int_0^t\eta(s)^2\dd s}{2\int_0^t\eta(s)\dd s}, 
\label{eq:SGD_opt_bound}    
\end{equation}
where $\overline{\norm{\nabla f(\*X_t)}^2}\coloneqq \frac{\int_0^t \eta(s)\norm{\nabla f(\*X_s)}^2\dd s}{\int_0^t\eta(s)\dd s}$ and $f_{\min} \coloneqq \min_{\*x \in \bb R^N} f(\*x)$, and 
\[
\sigma_0:=\sigma_g\sqrt{\qty(1+\sqrt{\frac{D}{N}}) + \frac{C}{N^{\frac{2}{3}}}}.
\]
\end{mytheorem}

It can be observed that the average squared norm of the gradient can be effectively upper bounded by a function of the training hyper-parameters. Given a maximal time horizon $T>0$, the bound in Eqn.~\eqref{eq:SGD_opt_bound} achieves the optimal bound of $\mathcal{O}(1/\sqrt{T})$ when we select constant LR $\eta(t)=\mathcal{O}(1/\sqrt{T})$. This upper bound consists of two terms, both of which share the denominator of $\int_0^t\eta(s)\dd s$. {Despite the complexity inherent in the numerator, the numerator of the first term is clearly of the order $\order{N}$, given that $f(\*X_0)-f_{\min} \leq \frac{L}{2} \norm{\*X_0-\*X^*}^2$, where $\*X^* \in \argmin_{\*x} f(\*x)$. The magnitude of the numerator in the second term is also $\order{N}$. Thus, the overall order of the average squared norm of the gradient is $\mathcal{O}(N/\int_0^t\eta(s)\dd s)$.} 
%If the Assumption \ref{asm:zero-mean} is replaced by Assumption \ref{asm:max_bound_sigma}, then $\sigma_0$ in Theorem\ref{thm:covergence-SGD} is replaced by $\bar\sigma$.
\blue{\paragraph{Practical interpretation.} In non-convex optimization, finding a global minimum is generally intractable, and the standard convergence criterion is to reach an approximate stationary point. In the stochastic setting, the gradient norm at any single iterate is subject to mini-batch noise and serves as a poor convergence indicator. The averaged squared gradient norm~\citep{ghadimi2013stochastic} smooths out this noise and measures how much time the trajectory spends in near-stationary regions, providing a robust trajectory-level convergence metric.

A small averaged gradient norm implies that the optimizer has largely converged to stationary points. Stochastic gradient methods almost surely avoid strict saddle points~\citep{lee2019first}, so the stationary points reached in practice are local minima. For overparameterized neural networks, recent work shows that the loss landscape is benign: \citet{chen2026exploring} observe a basin-like structure in LLMs where performance remains stable within each basin and the basin expands with model scale, consistent with the broadening of the low-loss region observed in Figs.~\ref{fig:toyexample} and~\ref{fig:MoE-Loss}. These observations connect the averaged gradient norm to the final loss: convergence to a local minimum in a benign landscape yields a good loss.

Our convergence bound (Theorem~\ref{thm:covergence-SGD}) shows that the averaged gradient norm is controlled by $1/\int\eta$, which therefore serves as a natural surrogate for convergence speed in the feature construction. Under the Polyak--\L{}ojasiewicz (PL) condition~\citep{karimi2016linear}, the qualitative chain above can be made quantitative: $\|\nabla f\|^2 \geq 2\mu(f - f_{\min})$ directly converts the gradient norm bound into a loss gap bound. \citet{liu2022loss} show that overparameterized networks generically satisfy PL*, a variant of PL, providing theoretical grounding for this connection. Our feature construction relies on the qualitative relationship and does not require PL to hold.}

\subsubsection{Adam Convergence Analysis} We proceed to analyze the convergence of~\eqref{eq:Adam-SDE}. Before delving into this analysis, we need to introduce an additional mild assumption, which is commonly employed in the study of stochastic first-order methods \citep{bertsekas2000gradient,reddi2018convergence}.
\begin{myassumption}[Smoothness and Boundedness]\label{asm:eta}
The function $f(\cdot)$ is smooth with respect to the parameters and the normalized LR schedule is bounded,
\begin{enumerate}
    \item $\abs{f(\*x)-f(\*y)}\leq\ell\norm{\*x-\*y}\quad \forall \*x,~\*y \in  \bb R^{N}$, \label{asm:eta-1}
    \item $\int_0^\infty\eta(s) \dd s=\infty, \quad \int_0^\infty\eta^2(s) \dd s <\infty$.
\end{enumerate}
\end{myassumption}
Assumption \ref{asm:eta}.1 asserts global Lipschitz continuity, implying that the gradient of $f$  is uniformly bounded. This is a common assumption in extensive literature on adaptive gradient methods \citep{reddi2018convergence,xiao2024adam,xie2024adan}. Assumption \ref{asm:eta}.2 requires that the LR be non-summable but square summable, which is a fundamental assumption for ensuring the convergence of stochastic first-order methods \citep{bertsekas2000gradient,davis2020stochastic}. Based on this assumption, we first establish the boundedness of the~\eqref{eq:Adam-SDE} dynamics as follows. The boundedness helps to guarantee the convergence of~\eqref{eq:Adam-SDE}.  From Proposition \ref{prop:sigma_op_bound}, we can observe that the maximal eigenvalue of $\*\Sigma(\*x)$ can be bounded by $\mathcal{O}\qty(\sigma_g^2\sqrt{D/N})$ in expectation. Based on this observation, we make the following stronger assumption than Assumption \ref{asm:zero-mean} to establish the convergence of Adam. 
\begin{myassumption}
\label{asm:max_bound_sigma}
There exists $\bar\sigma$, such that $\lambda_{\max}({\*\Sigma}(\*x))\leq\bar\sigma$ for all $\*x \in \mathbb{R}^N$.
\end{myassumption}

\begin{myprop}[Dynamics Boundedness]\label{prop:adam-bound}
Suppose Assumptions~\ref{asm:eta}~and~\ref{asm:max_bound_sigma} hold, then there exist positive constants $M$ and $V$ which are independent of $N$, such that 
\[
\sup_{t\geq0}\left\{\E \qty[\norm{\*m_t}^2]\right\} \leq M^2N, \quad\sup_{t\geq0}\left\{\norm{\*v_t}_\infty\right\} \leq V.
\]
\end{myprop}
This proposition demonstrates uniform bounds for $\*m_t$ and $\*v_t$. The distinction is that the bound for $\*v_t$ is deterministic, whereas the bound for $\*m_t$ is in the sense of expectation. This is due to the SDE associated with $\*m_t$ involving a Brownian motion. Since both $M$ and $V$ are independent of $N$, the scales of $\*m_t$ and $\*v_t$ are matched in the sense that both $\E[\norm{\*m_t}^2]$ and $\E[\norm{\*v_t}^2]$ are of the order $\mathcal{O}(N)$.

\begin{mytheorem}[Adam Convergence Bound]\label{thm:covergence-adam}
 Suppose Assumptions~\ref{asm:L-smoothh},~\ref{asm:eta}~and~\ref{asm:max_bound_sigma} hold. For the dynamics of~\eqref{eq:Adam-SDE}, the following bound holds:
\begin{equation}\label{eq:adam_opt_bound_m}
\begin{aligned}
\E\left[\overline{\norm{\*m_t}^2}\right]\leq\frac{\sqrt{V+\epsilon}\left(f(\*X_0)+\frac{1}{2c_1}\inner{\frac{\*m_0}{\sqrt{\*v_0+\epsilon}},\*m_0}-f_{\min}\right)}{\qty(1-\frac{c_2}{4c_1})\int_0^t\eta(s)\dd s}+\frac{\frac{(c_1')^2}{2c_1} \bar\sigma\sqrt{V+\epsilon}\int_0^t\eta(s)^2\dd s}{\qty(1-\frac{c_2}{4c_1})\sqrt{\epsilon}\int_0^t\eta(s)\dd s},
\end{aligned}
\end{equation}
where $\overline{\norm{{\*m_t}}^2} \coloneqq \frac{\int_0^t \eta(s)\norm{\*m_s}^2\dd s}{\int_0^t\eta(s)\dd s}$, and $f_{\min} \coloneqq \min_{\*x \in \bb R^N} f(\*x)$.
Moreover, we also have 
\begin{equation}
\begin{aligned}
\E\left[\overline{\norm{\nabla f(\*X_t)}^2}\right]\leq&\frac{2\sqrt{V+\epsilon}\left(f(\*X_0)-\frac{1}{c_1}\inner{\nabla f(\*X_0),\frac{\*m_0}{\sqrt{\*v_0+\epsilon}}}-f_{\min}+\frac{\ell M\sqrt{N}}{c_1\sqrt{\epsilon}}\right)}{\int_0^t\eta(s)\dd s}\\
&+\qty(\frac{2L\sqrt{V+\epsilon}}{c_1\epsilon}+\qty(1+\frac{\bar\sigma^2}{\epsilon^2})\frac{c_2^2(V+\epsilon)}{2c_1^2\epsilon})\E\left[\overline{\norm{\*m_t}^2}\right].
\end{aligned}
\label{eq:adam_opt_bound_grad}
\end{equation}
\end{mytheorem}
Theorem~\ref{thm:covergence-adam} outlines the convergence rates of the momentum $\*m_t$ and gradient $\nabla f(\*X_t)$ in~\eqref{eq:Adam-SDE}. When focusing solely on the LR schedule $\eta(t)$, and ignoring $N$ and other Lipschitz constants, the bound for $\E\left[\overline{\norm{\*m_t}^2}\right]$ simplifies to $\mathcal{O}\qty(\frac{1}{\int_0^t\eta(s)\dd s}+\frac{\int_0^t\eta(s)^2\dd s}{\int_0^t\eta(s)\dd s})$. By Eqn.~\eqref{eq:adam_opt_bound_grad}, the bound for $\E\left[\overline{\norm{\nabla f(\*X_t)}^2}\right]$ is identical, indicating that $\E\left[\overline{\norm{\nabla f(\*X_t)}^2}\right]$ and $\E\left[\overline{\norm{\*m_t}^2}\right]$ share the same convergence rate in terms of $\eta(t)$.
\paragraph{Discussion}
Theorems~\ref{thm:covergence-SGD} and~\ref{thm:covergence-adam} collectively demonstrate a common relationship between the optimization hyper-parameters and the gradient norm from an SDE perspective. Both theorems suggest that the average squared norm of the gradient can be upper bounded by $\mathcal{O}(N^\gamma/\int_0^t\eta(s)\dd s)$, where $\gamma\geq 1$ is a constant. \blue{The magnitude of $f(\*x_t) -f(\*x^*)$ is positively correlated with the gradient norm in smooth cases}, where $\*x^*$ is a local minimum in the current region. Consequently, the relationship between the training loss and the hyper-parameters $\eta(t)$ can also be expressed as $c(N/\int_0^t\eta(s)\dd s)^\alpha$.

As in \ref{eq:opt-law2}, we introduce the constants $c$ and $\alpha$ to generalize the relationship between convergence rate and training loss. These parameters are essential because $\mathcal{O}(N^\gamma/\int_0^t\eta(s)\dd s)$ reflects the worst-case convergence rate, while the actual rate depends on the specific architectures of LLMs, data, and training techniques. Therefore, a data-driven approach is used to determine $c$ and $\alpha$, ensuring these parameters are more practically applicable, as worst-case bounds are often too conservative for LLM training strategies.

\subsection{Inspiration from Escaping Capacity}
\blue{Effective optimization hyper-parameters can expedite escaping sharp local minima during training, which is why large LRs are favored when training LLMs, followed by cooldown techniques. These strategies help LLMs efficiently move away from sharp local minima, improving the final training loss.}
Building on this idea, this subsection quantitatively investigates the escape capacity of~\eqref{eq:SGD-SDE} and~\eqref{eq:Adam-SDE}. We calculate the probability of these dynamics escaping a local region, providing a rigorous analysis of their effectiveness.

\subsubsection{Linearization Approximation of SDEs}
In the context of SDEs, the density of the solutions adheres to the Fokker-Planck-Kolmogorov (FPK) equation. For instance, in the case of~\eqref{eq:SGD-SDE}, the specific FPK form is 
\begin{equation}\label{eq:fpk}
\pdv{p(\*x,t)}{t} =  \sum_{i=1}^N \pdv{x_i}\qty[\eta(t)\nabla_i f(\*x) p(\*x,t)] + \frac{\eta_0}{2} \sum_{i,j}^N \pdv{}{x_i}{x_j}\qty[\*\Sigma(\*x)_{ij} \eta^2(t) p(\*x,t)],    
\end{equation}
where $p(\*x,t)$ denotes the density of $\*X_t$ at time $t$, $\*\Sigma(\*x)_{ij}$ is the $(i,j)$-th entry of the covariance matrix $\*\Sigma(\*x)$ in Eqn.~\eqref{eq:sigma}, $\nabla_i f(\*x)$ is the $i$-th component of the gradient $\nabla f(\*x)$, and the initial distribution is specified by the Dirac delta function $\delta(\*x-\*X_0)$.
The non-linearity and time-dependent nature of the FPK equation pose significant challenges for deriving analytical solutions. Consequently, it is difficult to directly estimate the exit times using the density function $p(\*x,t)$. 
A common approach is to approximate the operators in Eqn.~\eqref{eq:fpk}, to make the estimation of exit times more tractable.

Given the density $p(\*x,t)$ at time $t$, we aim to determine the density $p(\*x,t+\Delta_t)$ at time $t+\Delta_t$.  Note that the transition density over the interval $[t,t+\Delta_t]$ still satisfies Eqn.~\eqref{eq:fpk} with the initial condition $p(\*x,t)$. We then perform a local linear approximation of Eqn.~\eqref{eq:fpk}, for all $s \in   [t,t+\Delta t]$:
\begin{equation}\label{eq:fpk-approx}
\begin{aligned}
    \pdv{p(\*x,s)}{s} 	\approx & \sum_{i=1}^N \pdv{x_i}\qty{\eta(s) p(\*x,s) \nabla_i f(\bar{\*x}_t)  + \eta(s) p(\*x,s) \qty[\nabla^2 f(\bar{\*x}_t)\qty(\*x-\bar{\*x}_t)]_i } \\
    &+ \frac{\eta_0}{2} \sum_{i,j}^N \pdv{}{x_i}{x_j}\qty[\eta^2(s)p(\*x,s)\*\Sigma(\bar{\*x}_t)_{ij}  ],
\end{aligned}
\end{equation}
where the drift term $\nabla f(\*x)$ is approximated by its first order expansion around the mean value $\bar{\*x}_t\coloneqq \E[\*X_t]$ at time $t$, while the diffusion term $\*\Sigma(\*x)$ is approximated by its value at $\bar{\*x}_t$ without expansion. Higher-order terms are neglected. This local linearization implies that the FPK equation now describes a local Ornstein-Uhlenbeck process. Consequently, the solution $p(\*x,s)$ becomes a Gaussian distribution $ \forall s \in   [t,t+\Delta_t]$. By iteratively applying such local approximations and letting  $\Delta_t \to 0$, the dynamics of $\*X_t$ evolve into a Gaussian process derived from piecewise linear approximations. The mean $\bar{\*x}_t \coloneqq \E\qty[\*X_t]$ and variance $\*P_t\coloneqq \operatorname{Cov}(\*X_t)$ of this process are governed by the ODEs as follows:
\begin{equation}\label{eq:SDE-GA}
% \label{eq:ODE-approx}
\left\{
\begin{aligned}
\dv{\bar{\*x}_t}{t}&=-\eta(t) \nabla f\qty(\bar{\*x}_t),\\
\dv{\*P_t}{t}&= -\eta(t)\*P_t \nabla^2 f(\bar{\*x}_t) -  \eta(t)\nabla^2 f(\bar{\*x}_t) \*P_t + \eta_0\eta^2(t) \*\Sigma(\bar{\*x}_t).
\end{aligned}
\right.
\tag{SDE-GA}
\end{equation}
Since $\{\*X_t\}$ is a Gaussian process, the solutions to the ODEs~\eqref{eq:SDE-GA} for the mean and variance fully determine the form of the density. This allows us to analytically determine $p(\*x,t)$, which is pivotal for estimating the escape probability from a local region. These approximations, also known in the literature as Gaussian approximations (GAs) ~\citep{solin2021scalable,archambeau2007gaussian}, are widely utilized in filtering theory~\citep{sarkka2013gaussian,sarkka2015posterior}. For a comprehensive overview, please refer to~\cite{sarkka2019applied}, Sec. 9.1.

\subsubsection{Escape Probability from Local Minima}
The GA results for both~\eqref{eq:SGD-SDE} and~\eqref{eq:Adam-SDE} are provided as follows:

\begin{myprop}[SGD-SDE Approximation]\label{prop:P-SGD}
Considering any local minimum $\*x^*$ of $f(\cdot)$, and setting the initial conditions of \eqref{eq:SGD-SDE} as $\*X_0 = \*x^*$ and $\*P_0 = 0$, the mean $\bar{\*x}_t \coloneqq \E\qty[\*X_t]$ and variance $\*P_t\coloneqq \operatorname{Cov}(\*X_t)$ of the \eqref{eq:SDE-GA} for \eqref{eq:SGD-SDE} satisfies the following ODE:
\begin{equation}
\label{eq:SGD-ODE-approx}
\left\{
\begin{aligned}
\dv{\bar{\*x}_t}{t}&=\*0,\\
\dv{\*P_t}{t}&= -\eta(t)\*P_t \*H -  \eta(t)\*H \*P_t + \eta_0\eta^2(t) \*\Sigma,
\end{aligned}
\right.
\end{equation}
where $\*H \coloneqq \nabla^2 f(\*x^*)$, and the definition of $\*\Sigma\coloneqq \*\Sigma(\*x^*)$ is provided in Eqn.~\eqref{eq:sigma}.
Furthermore, the solution to Eqn.~\eqref{eq:SGD-ODE-approx} has the following closed form:
\begin{equation}\label{eq:SGD-P_t}
\*P_t=\*A(t)\qty(\int_0^t\exp\qty(\*H\int_0^s\eta(\tau)\dd\tau)  \*\Sigma\exp\qty(\*H\int_0^s\eta(\tau)\dd\tau)\eta_0\eta^2(s)\dd s)\*A(t),    
\end{equation}
where 
\[
\begin{aligned}
\*A(t)&=\exp\qty(-\*H\int_0^t\eta(s)\dd s).
\end{aligned}
\]
\end{myprop}
We can reformulate the \eqref{eq:Adam-SDE} in the similar form of \eqref{eq:SGD-SDE}. Let $\*Z_t:=[\*X_t;\*m_t;\*v_t]$ and $\widehat{\*W}_t$ be the $3N$-dimensional Brownian motion. Then, \eqref{eq:Adam-SDE} can be rewritten as 
\begin{equation}
\dd\*Z_t=-\eta(t)\underbrace{\qty[
\begin{array}{c}
    \*m_t\odot(\*v_t+\epsilon)^{-\frac{1}{2}} \\
    c_1\qty(\*m_t-\nabla f(\*X_t))  \\
    c_2\qty(\*v_t-{\rm diag}(\*\Sigma(\*X_t)))
\end{array}
]}_{=:\*F(\*Z_t)}\dd t+c_1'\eta(t)\qty[
\begin{array}{ccc}
    \*0&&  \\
    &\sigma(\*X_t)&  \\
     &&\*0 
\end{array}
]\dd\widehat{\*W}_t.
\label{eq:Adam-Z}
\end{equation}
Subsequently, the Jacobian of $\*F(\*Z)$ with respect to $\*Z$ is given by:
\begin{equation}
\partial_{\*Z}\*F(\*Z)=\qty[
\begin{array}{ccc}
    \*0 & {\rm Diag}\qty(\*v+\epsilon)^{-\frac{1}{2}} & -\frac{1}{2}{\rm Diag}\qty(\*m\odot\qty(\*v+\epsilon)^{-\frac{3}{2}})  \\
    -c_1\nabla^2f(\*X) & c_1\*I & \*0\\
     -c_2\partial_{\*X}{\rm diag}\qty(\*\Sigma(\*X)) & \*0 & c_2\*I
\end{array}
].
\label{eq:dF-Adam-SDE}
\end{equation}
Based on the reformulation, we also have a similar GA approximation for \eqref{eq:Adam-SDE}. 

\begin{myprop}[Adam-SDE Approximation]\label{prop:P-Adam}
Denote $\bar{\*z}_t:=\bb{E}[\*Z_t]$ and $\widehat{\*P}_t:={\rm Cov}(\*Z_t)$. Consider any local minimum $\*x^*$ of $f(\cdot)$ and set the initial conditions of~\eqref{eq:Adam-SDE} as $\*X_0 = \*x^*$, $\*m_0 = \*0$, $\*v_0 = {\rm diag}(\*\Sigma(\*x^*))$ and $\*P_0 = 0$. The Gaussian approximation for~\eqref{eq:Adam-SDE} satisfies
\begin{equation}
\label{eq:Adam-ODE-approx}
\left\{
\begin{aligned}
\dv{\bar{\*z}_t}{t}&=\*0,\\
\dv{\widehat{\*P}_t}{t}&= -\eta(t)\widehat{\*P}_t \widehat{\*H}^\top -  \eta(t)\widehat{\*H} \widehat{\*P}_t + (c_1')^2\eta(t)^2\widehat{\*\Sigma},
\end{aligned}
\right.
\end{equation}
where 
\[
\widehat{\*H}\coloneqq\qty[
\begin{array}{ccc}
    \*0 & {\rm Diag}\qty(\*\Sigma(\*x^*)+{\epsilon}\*I)^{-\frac{1}{2}} & \*0  \\
    -c_1\nabla^2f(\*x^*) & c_1\*I & \*0\\
     -c_2\partial_{\*X}{\rm diag}\qty(\*\Sigma(\*x^*)) & \*0 & c_2\*I
\end{array}
],\quad and \quad \widehat{\*\Sigma}:=\qty[
\begin{array}{ccc}
    \*0 & & \\
     & \*\Sigma(\*x^*) & \\
     & &\*0
\end{array}
].
\]
Furthermore, the solution to Eqn.~\eqref{eq:Adam-ODE-approx} has the following closed form:
\begin{equation}\label{eq:Adam-P_t}
\widehat{\*P}_t=\hat{\*A}(t)\left(\int_0^t\exp\qty(\widehat{\*H}\int_0^s\eta(\tau)\dd\tau)    \widehat{\*\Sigma}\exp\qty(\widehat{\*H}^\top\int_0^s\eta(\tau)\dd\tau)(c_1')^2\eta(s)^2\dd s\right)\hat{\*A}^\top(t),    
\end{equation}
where 
\[
\begin{aligned}
\hat{\*A}(t)&=\exp\qty(-\widehat{\*H}\int_0^t\eta(s)\dd s).
\end{aligned}
\]
\end{myprop}
Propositions~\ref{prop:P-SGD}~and~\ref{prop:P-Adam} indicate that, under specified initial conditions, the approximated solutions to both \eqref{eq:SGD-SDE} and \eqref{eq:Adam-SDE} follow Gaussian distribution with mean $\*x^*$. For ease of notation, we will continue to denote this Gaussian approximated solution as $\*X_t$ throughout this paper. By leveraging the anti-concentration inequality~\citep{carbery2001distributional}, we can effectively calculate the probability that $\*X_T$ remains within the local vicinity of $\*x^*$ after a time period $T$. A smaller probability suggests a greater likelihood of escape from this region, indicating better escape capacity.

%Based on Proposition \ref{prop:sigma_bound}, it is reasonable to make the following assumption on the covariance matrix on $\*x^*$. 

% \begin{myassumption}
% \label{assumption:sigma_bound}
% It holds that $\Tr(\*\Sigma(\*x^*))=\mathcal{O}\left(\sqrt{ND}\right)$.
% \end{myassumption}

\begin{mytheorem}[SGD Escape Probability]\label{thm:SGD-prob}
With the same initial conditions specified in Propositions~\ref{prop:P-SGD}, the Gaussian approximated solution for \eqref{eq:SGD-SDE} satisfies
\[
\bb P[\norm{\*X_t-\*x^*}^2 \geq \varepsilon] \geq 1 -\sqrt{\frac{ e \varepsilon}{\Tr(\*P_t)}},
\]
where $\*P_t$ is the covariance matrix of $\*X_t$ defined in Eqn.~\eqref{eq:SGD-P_t}. Suppose Assumption \ref{asm:zero-mean} holds, and considering the learning rate conditions $\eta(0)=\eta_{\max}$ and $\eta(T)=0$, then we have
\[
\bb P[\norm{\*X_T-\*x^*}^2 \leq \varepsilon ] = \order{\qty(\frac{\varepsilon  }{\eta^4_{\max}\Tr(\*\Sigma_g)  }\int_0^T \eta^\prime(s)^2 \dd s)^{1/2}}.
\]
\end{mytheorem}

\begin{mytheorem}[Adam Escape Probability]\label{thm:Adam-prob}
With the same initial conditions specified in Propositions~\ref{prop:P-Adam}, the Gaussian approximated solution for \eqref{eq:Adam-SDE} satisfies
\[
\bb P[\norm{\*X_t-\*x^*}^2 \geq \varepsilon] \geq 1 -\sqrt{\frac{ e \varepsilon}{\Tr(\hat{\*P}_t)}},
\]
where $\widehat{\*P}_t$ is the covariance matrix of $\*Z_t$ defined in Eqn.~\eqref{eq:Adam-P_t}. Suppose Assumption \ref{asm:zero-mean} holds, and considering the learning rate conditions $\eta(0)=\eta_{\max}$ and $\eta(T)=0$, then we have
\[
\bb P[\norm{\*X_T-\*x^*}^2 \leq \varepsilon ] = \order{\qty(\frac{\varepsilon  }{\eta^4_{\max}\Tr(\*\Sigma_g) }\int_0^T \eta^\prime(s)^2 \dd s)^{1/2}}.
\]
% \[
% \bb P[\norm{\*X_T-\*x^*}^2 \leq \varepsilon ] = \order{\qty(\frac{\varepsilon  }{\eta^2_{\max} \sigma_g^2 \sqrt{DN}}\int_0^T \eta^\prime(s)^2 \dd s)^{1/2}}.
% \]
\end{mytheorem}
Theorems~\ref{thm:SGD-prob} and~\ref{thm:Adam-prob} provide lower bounds on the probability that $\*X_t$ is located outside an $\varepsilon$-radius ball centered in $\*x^*$ after time $t$. A larger lower bound indicates a stronger escaping capacity of $\*X_t$. Equivalently, a smaller probability that $\*X_T$ remains within the local region of $\*x^*$ after the LR has cooled down (at time $T$) suggests a greater capacity of $\*X_T$ to escape suboptimal local minima. 
%Notably, the likelihood of $\*X_T$ staying in its current region can be related to the integral $\int \dot{\eta}_t^2$. An interesting insight from our findings is that during the cooldown phase of the learning rate, the rate of decrease should not be excessively rapid.

Based on the unified upper bounds provided by Theorems~\ref{thm:SGD-prob} and~\ref{thm:Adam-prob}, we propose incorporating $\mathcal{O} ((\int \eta^\prime(s)^2 \dd s)^\alpha)$ into \abbr{}. \blue{Together with the convergence analysis, this yields two complementary perspectives on how the LR schedule affects the final training loss. The convergence terms capture optimization progress toward stationarity, while the escape terms characterize the schedule's influence on basin selection. These two aspects jointly motivate the feature construction in \ref{eq:opt-law2}.}

\blue{\paragraph{On the escape analysis in practice.} During the high-learning-rate phase, stochastic noise prevents the iterate from settling at any local minimum. However, as the LR decays to zero the noise vanishes and the iterate converges to a basin. The escape analysis (Theorems~\ref{thm:SGD-prob} and~\ref{thm:Adam-prob}) characterizes which basin the iterate ultimately falls into. A more gradual cooldown sustains noise-driven diffusion longer, allowing migration toward better basins before the iterate is frozen, while an abrupt cooldown locks the iterate into its current region. The quantity $\int{\eta'(s)}^2\,ds$ summarizes this effect from the covariance dynamics, measuring the abruptness of the cooldown and hence the degree to which the schedule constrains basin selection. Progressive ablation (Appendix Sec.~\ref{sec:ablation}) confirms that escape features carry independent predictive signal, improving Top-2 from 66\% to 89\% on top of convergence-only features.}

\blue{\paragraph{On the $\beta$--$\eta$ coupling in Adam-SDE.} The coupling between $\beta_{1,k},\beta_{2,k}$ and $\eta_k$ is a mathematical necessity for a non-degenerate SDE: with fixed $\beta$, the momentum relaxes infinitely fast relative to the parameter updates as $\eta\to 0$, reducing the SDE to a deterministic normalized gradient flow in which the stochastic dynamics and the rich schedule-dependent structure (such as the escape-related terms) are lost~\citep{malladi2022sdes}. The single-time-scale coupling ensures that momentum and parameter dynamics evolve on the same effective time scale. Regardless of whether $\beta$ is fixed (discrete analysis) or coupled with $\eta$ (SDE analysis), $\beta$ affects only the constants in the convergence bound, not the functional dependence on the schedule: discrete fixed-$\beta$ analyses~\citep{xie2024adan,zhang2022adam,chen2018on} yield bounds controlled by $\sum_k\eta_k$, confirming that the convergence feature form $1/\int_0^t\eta(s)\,ds$ is preserved. The SDE framework additionally provides escape-related insights involving $\int_0^T{\eta'(s)}^2\,ds$ that are inaccessible from discrete analysis alone. Since our feature construction only extracts these functional forms and fits all coefficients from data, this is sufficient for our purpose. Bridging the gap between the discrete fixed-$\beta$ framework and the continuous coupled-$\beta$ SDE framework is an open problem beyond the scope of this work.}

\section{\blue{Conclusion}}

\blue{We introduced \abbr{}, a framework that predicts the final training loss of LLMs as a function of the learning-rate schedule, model size, and data size.}

\blue{The framework decomposes the schedule's effect into two SDE-motivated mechanisms. The convergence features, derived from the integrated learning rate $\int\eta$, capture how fast the optimizer approaches stationarity. The escape features, derived from $\int{\eta'_t}^2$, capture the schedule's influence on which basin the optimizer ultimately settles into. This decomposition yields 15 interpretable features whose coefficients are fitted from small-scale experiments, bridging theoretical analysis and data-driven prediction.}

\blue{On the held-out benchmark, the generalized \abbr{} achieves Top-2 hit rate of 94\% and Spearman correlation of 0.84, outperforming all five baselines. The divergence criterion attains F1${=}0.92$. In extrapolation settings, the law correctly identifies the best schedule family in all five out-of-family groups, and the history-aware extension produces correct rankings across continual-training and fine-tuning pipelines.}

\blue{The current empirical validation covers MoE models with up to 4B trainable parameters and piecewise-linear, cosine, and constant-plus-linear schedules. The SDE-based feature construction does not assume a particular architecture or schedule family, and extending the validation to broader settings is a natural next step. A discussion of scope and limitations is provided in Appendix~\ref{sec:limitations}.}
% \acks{Thanks to Skyworks AI for supporting the computational resource.}

% Acknowledgements and Disclosure of Funding should go at the end, before appendices and references


\appendix

\input{appendix}

\vskip 0.2in
\bibliography{scaling}

\end{document}

%% file: appendix.tex
\section{\blue{Experimental Setup and Supplementary Results}}\label{sec:exp-supplement}
\blue{This appendix provides supplementary details for the experiments in the main text. We first summarize the shared model and training settings (Sec.~\ref{sec:model-config}), then define the small-scale training grid and evaluation protocol used for held-out benchmarking (Sec.~\ref{sec:source-pool}). Next, we document the exponent selection procedure and post-selection fitted coefficients (Sec.~\ref{sec:alpha-protocol}). We then present the feature ablation and sensitivity analysis (Sec.~\ref{sec:ablation}), followed by the divergence prediction benchmark (Sec.~\ref{sec:heldout-benchmark}) and schedule-family generalization results (Sec.~\ref{sec:transfer-results}). 

The theory supplement appears in Appendix.~\ref{sec:theory-supplement}, and a discussion of scope and limitations in Appendix.~\ref{sec:limitations}.}

\subsection{\blue{Model Configurations and Shared Training Settings}}\label{sec:model-config}
Table~\ref{tab:moe-config} summarizes the key parameters of the MoE model employed in our experiments. These include architectural parameters such as the number of layers, hidden size, and the number of attention heads, along with parameter sizes: total trainable parameters, activated parameters during the forward pass, and the total parameters for the eight experts. Additionally, common training parameters are listed, including token length, global batch size, optimizer (AdamW~\citep{loshchilov2017decoupled}), weight decay, minimal learning rate, and gradient clipping threshold. These parameters adhere to standard LLM training practices without any special adjustments.
\blue{We set the minimal learning rate to 0.0 to maximize comparability across schedules and to avoid an artificial performance floor from a non-zero residual learning rate.}

\begin{table}[t]
\centering
\vspace{2pt}
\small
\setlength{\tabcolsep}{5pt}
\renewcommand{\arraystretch}{1.2}
\caption{Key parameters of the MoE models used in experiments.
\emph{Top:} architecture parameters (per model size).
\emph{Bottom:} training settings shared across all models.}
\label{tab:moe-config}
\begin{tabular}{@{}l ccc ccc@{}}
\toprule
& \multicolumn{3}{c}{\textbf{Architecture}} & \multicolumn{3}{c}{\textbf{Parameters}} \\
\cmidrule(lr){2-4}\cmidrule(l){5-7}
\textbf{Model Size}
  & \textbf{Layers} & \textbf{Hidden} & \textbf{Attn.\ Heads}
  & \textbf{Total} & \textbf{Activated} & \textbf{Expert} \\
\midrule
$8{\times}0.001$B & 4  & 128  &  4 & 0.023B & 0.019B & 0.0063B \\
$8{\times}0.02$B  & 12 & 384  & 12 & 0.17B  & 0.093B & 0.113B  \\
$8{\times}0.1$B   & 12 & 768  & 12 & 0.58B  & 0.27B  & 0.45B   \\
$8{\times}0.3$B   & 24 & 1024 & 16 & 1.90B  & 0.75B  & 1.66B   \\
$8{\times}0.6$B   & 24 & 1536 & 16 & 4.05B  & 1.56B  & 3.62B   \\
\midrule
\multicolumn{7}{@{}l}{\textit{Shared training settings (all models):}} \\[2pt]
\multicolumn{7}{@{}l}{
  Token length:\;2048\quad
  Batch size:\;2048\quad
  Optimizer:\;AdamW
}
\\[2pt]
\multicolumn{7}{@{}l}{
  Weight decay:\;0.1\quad
  Min.\ LR:\;0.0\quad
  Grad.\ clip:\;1.0
} \\
\bottomrule
\end{tabular}
\end{table}

\subsection{\blue{Fitting Setup and Evaluation Protocol}}\label{sec:source-pool}
\blue{This subsection defines the training grid and evaluation protocol used for the held-out benchmarks in the main text.}

\blue{\paragraph{Training grid.} The training grid corresponds to the loss grids in Fig.~\ref{fig:MoE-Loss}. We train MoE models at four model sizes ($8{\times}0.001$B, $8{\times}0.02$B, $8{\times}0.1$B, $8{\times}0.3$B) on 10B, 20B, or 30B tokens, yielding seven (model size, token budget) blocks listed in Table~\ref{tab:data-inventory}. Within each block, we vary peak LR from $10^{-3}$ to $1.5{\times}10^{-2}$ and warmup steps from 128 to 6000, producing 35 or 40 configurations per block and 260 configurations in total. All runs use a linear warmup followed by linear cooldown schedule. A run is labeled divergent if its final loss exceeds 6 (the typical plateau for failed runs), yielding 234 convergent and 26 divergent configurations. The convergent configurations are used for regression evaluation, and all 260 for divergence prediction.}

\begin{table}[t]
\caption{\blue{Data inventory for the small-scale training grid. Each row is a (model size, token budget) block with 35--40 hyper-parameter configurations. Divergent runs are those whose final loss exceeds 6.}}
\label{tab:data-inventory}
\vspace{2pt}
\centering
\small
\setlength{\tabcolsep}{6pt}
\renewcommand{\arraystretch}{1.15}
\begin{tabular}{@{}ll rrr@{}}
\toprule
\textbf{Model Size} & \textbf{Tokens} & \textbf{Total} & \textbf{Convergent} & \textbf{Divergent} \\
\midrule
$8{\times}$0.001B & 10B & 35 & 35 & 0 \\
$8{\times}$0.001B & 30B & 40 & 40 & 0 \\
$8{\times}$0.02B  & 10B & 35 & 33 & 2 \\
$8{\times}$0.02B  & 30B & 40 & 38 & 2 \\
$8{\times}$0.1B   & 10B & 35 & 29 & 6 \\
$8{\times}$0.1B   & 30B & 40 & 33 & 7 \\
$8{\times}$0.3B   & 20B & 35 & 26 & 9 \\
\midrule
\multicolumn{2}{@{}l}{\textbf{Total}} & \textbf{260} & \textbf{234} & \textbf{26} \\
\bottomrule
\end{tabular}
\end{table}

\blue{\paragraph{Cross-validation protocol.} We evaluate via 5-fold cross-validation at the configuration level. Within each (model size, token budget) block, the hyper-parameter configurations are randomly partitioned into five folds, so that every fold contains held-out samples from all blocks. In each round, the training portions from all blocks are combined to fit a single exponent vector $\bm{\alpha}$ and coefficient vector $\mathbf{c}$ across the full grid, following the staged procedure in Sec.~\ref{sec:alpha-protocol}. The held-out configurations are used only for evaluation. Metrics are computed on the held-out configurations of each round and averaged across all five rounds. Divergent runs are excluded from the regression but retained for the divergence classification task in Sec.~\ref{sec:heldout-benchmark}. Runs with the constant-plus-linear cooldown schedule are excluded from fitting and reserved for the schedule-family generalization evaluation in Sec.~\ref{sec:transfer-results}.}

\blue{\paragraph{Task separation.} The training grid supports three evaluation tasks. (1) Regression with the simplified \abbr{} (Sec.~3), which uses a single (model size, token budget) block. (2) Regression with the generalized \abbr{} (Sec.~\ref{sec:generalized}), which uses all convergent configurations across blocks. (3) Divergence prediction (Sec.~\ref{sec:heldout-benchmark}), which uses all 260 configurations including divergent runs. The three tasks share the same underlying grid but are evaluated independently.}

\subsubsection{\blue{Metric Definitions}}\label{sec:metric-def}
\blue{We report the following five metrics in Table~\ref{tab:main-heldout-benchmark}. The first two assess regression quality and the remaining three assess schedule-selection quality.}
\begin{itemize}[leftmargin=*,itemsep=2pt]
\item \blue{\textbf{$R^2$.} Computed over all held-out convergent configurations across all blocks and folds: $R^2 = 1 - \sum_i(y_i - \hat{y}_i)^2 / \sum_i(y_i - \bar{y})^2$, where $y_i = \log(\text{Loss}_i)$, $\hat{y}_i$ is the predicted log-loss, and $\bar{y}$ is the mean of $y_i$ over the held-out set.}
\item \blue{\textbf{Relative Error (\%).} The mean absolute relative error over held-out convergent configurations: $\frac{1}{n}\sum_i |(\text{Loss}_i - \widehat{\text{Loss}}_i) / \text{Loss}_i| \times 100\%$, where $\text{Loss}_i$ is the actual loss and $\widehat{\text{Loss}}_i$ is the predicted loss.}
\item \blue{\textbf{Spearman.} For each (model size, token budget) block in each fold, we rank the held-out configurations by predicted loss and by actual loss, and compute the Spearman rank correlation between the two rankings. The reported value is the average across all blocks and folds.}
\item \blue{\textbf{Top-$k$ Hit Rate.} For each block in each fold, we select the held-out configuration with the lowest predicted loss. If this configuration is among the $k$ held-out configurations with the lowest actual loss, the block counts as a hit. This simulates the practical scenario of ranking unseen candidate schedules. Top-$k$ Hit Rate is the number of hits divided by the total number of (block, fold) pairs evaluated. We report $k=2$ throughout.}
\item \blue{\textbf{Regret.} For each block in each fold, let $s^*$ be the held-out configuration with the lowest actual loss and $\hat{s}^*$ be the one with the lowest predicted loss. The regret is $\text{Loss}(\hat{s}^*) - \text{Loss}(s^*)$, measuring the excess loss incurred by following the prediction. The reported value is the mean regret across all block-fold evaluations.}
\end{itemize}

\subsection{\blue{Exponent Selection and Fitted Coefficients}}\label{sec:alpha-protocol}

\blue{The exponent vector $\bm{\alpha}$ in \ref{eq:opt-law2} is not optimized jointly with the linear coefficients $\mathbf{c}$. The three SDE-derived terms ($1/\int\eta_{\text{w}}$, $1/\int\eta_{\text{c}}$, $\int{\eta'_t}^2$) 
use exponent $\pm 1.0$ fixed by theory. For the remaining terms, each exponent is selected from the discrete candidate set}
\[
\{0.15,\,0.20,\,0.25,\,0.50,\,1.0\}.
\]
\blue{The selection follows a two-level procedure. In the first level, we assign a shared exponent to each of the four groups (mixed interactions, convergence cross-terms, escape powers, scale powers) and search over all group-level combinations ($5^4 = 625$), with $\mathbf{c}$ refitted via {linear regression} for each candidate, minimizing the fitting loss on the training folds. In the second level, we fix the group-level values and adjust only the most sensitive term within each group from the same candidate set. The held-out fold is used only for final evaluation. The final exponent vector and coefficients are reported in Table~\ref{tab:opt-law-coeff}.}

\subsubsection{\blue{Fitted Coefficients}}\label{sec:fitted-coeff}
\blue{Table~\ref{tab:opt-law-coeff} reports the selected exponent vector and fitted coefficient vector for the generalized \ref{eq:opt-law2}.} Unlike the simplified \abbr{} in Eqn.~\eqref{eq:optlaw}, the coefficient vector $\*c$ is not constrained to be strictly positive. The simplified law has a small number of terms with direct physical interpretations, where positive coefficients are natural. The generalized law, however, contains 15 features including interaction terms, many of which are correlated. In multivariate regression with correlated features, individual coefficients reflect partial effects conditional on all other features. For example, two features that are each positively correlated with loss may receive opposite signs because one already accounts for part of the other's contribution. Their signs therefore cannot be inferred from marginal correlations alone. We do not impose positivity constraints on $\*c$.

\blue{\paragraph{Normalization.}}
Before estimating the model loss using the parameters from these tables, it is necessary to normalize the training iteration steps $S$, model size $N$, and LR.
For model parameter count $N$, we use the total number of learnable parameters, expressed in billions. \blue{The iteration count $S$ and all step-related variables (warmup steps, cooldown steps) are converted to token counts in billions: $S_{\text{tokens}} = S \times L \times B / 10^9$, where $L = 2048$ is the token length and $B = 2048$ is the global batch size (Table~\ref{tab:moe-config}).} The LR is normalized by dividing it by $1.5 \times 10^{-2}$, ensuring that the normalized LR values fall within the range of 0 to 1 across all experiments.

Table~\ref{tab:opt-law-coeff} provides the coefficients and powers for the convergence, escape, mixed, and bias terms for two LR schedules: (1) a linear warmup over $a$ steps followed by a linear cooldown over $S-a$ steps, and (2) a linear warmup over $a_1$ steps, followed by a constant LR ($\eta_{\max}$) from $a_1$ to $a_2$ steps, and a cooldown over $S-a_2$ steps (as illustrated in Fig.~\ref{fig:effect} (a)).

\begin{table}[t]
\centering
\caption{\blue{Selected exponents $\alpha$, fitted coefficients $c$, and closed-form feature expressions for the generalized \ref{eq:opt-law2}. Two schedule families are listed. \textbf{S1}: linear warmup ($a$ steps) followed by linear cooldown ($S{-}a$ steps). \textbf{S2}: linear warmup ($a_1$ steps), constant phase at $\eta_{\max}$ ($a_1$ to $a_2$), then linear cooldown ($S{-}a_2$ steps). $\mathcal{E}$ denotes $\int_{a_{e_2}}^S{\eta^\prime_t}^2$ in the mixed terms.}}
\label{tab:opt-law-coeff}
\vspace{4pt}
\scriptsize
\renewcommand{\arraystretch}{1.2}
\setlength{\tabcolsep}{2pt}
\noindent
\begin{minipage}[t]{0.53\linewidth}
\centering
\scriptsize
\setlength{\tabcolsep}{2pt}
\begin{tabular}{@{}l r r c c@{}}
\toprule
$\phi$ & $\bm{\alpha}$ & $c$ & \textbf{S1} & \textbf{S2} \\
\midrule
\multicolumn{5}{@{}l}{\textbf{Convergence Terms}} \\[3pt]
$\int_0^{a_{c_1}}\!\eta_t$ & $-1.0$ & $-6.92{\times}10^{-4}$ &
$\tfrac{ah}{2}$ &
$\tfrac{a_1h}{2}$ \\[8pt]
$\int_{a_{c_2}}^S\!\eta_t$ & $-1.0$ & $-1.27{\times}10^{-3}$ &
$\tfrac{(S{-}a)h}{2}$ &
$\tfrac{(S{+}a_2{-}2a_1)h}{2}$ \\[8pt]
$N{/}\int_{a_{c_2}}^S\eta_t$ & $0.25$ & $-4.68{\times}10^{-2}$ &
$\tfrac{2N}{(S{-}a)h}$ &
$\tfrac{2N}{(S{+}a_2{-}2a_1)h}$ \\[8pt]
$\int_0^{a_{c_1}}\!\eta_t{\cdot}\int_{a_{c_2}}^S\!\eta_t$ & $-0.23$ & $4.65{\times}10^{-2}$ &
$\tfrac{a(S{-}a)h^2}{4}$ &
$\tfrac{a_1(S{+}a_2{-}2a_1)h^2}{4}$ \\[8pt]
\midrule
\multicolumn{5}{@{}l}{\textbf{Escape Terms}} \\[3pt]
$\int_{a_{e_2}}^S\!{\eta^\prime_t}^2$ & $1.0$ & $9.62{\times}10^{-3}$ &
$\tfrac{h^2}{S{-}a}$ &
$\tfrac{h^2}{S{-}a_2}$ \\[8pt]
$\int_0^{a_{e_1}}\!{\eta^\prime_t}^2$ & $0.25$ & $1.92{\times}10^{-2}$ &
$\tfrac{h^2}{a}$ &
$\tfrac{h^2}{a_1}$ \\[8pt]
$\int_{a_{e_2}}^S\!{\eta^\prime_t}^2$ & $0.25$ & $-5.05{\times}10^{-2}$ &
$\tfrac{h^2}{S{-}a}$ &
$\tfrac{h^2}{S{-}a_2}$ \\[8pt]
$SN$ & $-0.25$ & $-1.82{\times}10^{-1}$ & $SN$ & $SN$ \\[4pt]
\bottomrule
\end{tabular}
\end{minipage}%
\hfill%
\begin{minipage}[t]{0.46\linewidth}
\centering
\scriptsize
\setlength{\tabcolsep}{2pt}
\begin{tabular}{@{}l r r c c@{}}
\toprule
$\phi$ & $\bm{\alpha}$ & $c$ & \textbf{S1} & \textbf{S2} \\
\midrule
\multicolumn{5}{@{}l}{\textbf{Mixed Terms} ($\mathcal{E} =
\int_{a_{e_2}}^S{\eta^\prime_t}^2$)} \\[3pt]
$\mathcal{E}{/}\int_0^{a_{c_1}}\!\eta_t$ & $0.20$ &
$-4.68{\times}10^{-2}$ &
$\tfrac{2h}{a(S{-}a)}$ &
$\tfrac{2h}{a_1(S{-}a_2)}$ \\[8pt]
$\mathcal{E}{/}\int_{a_{c_2}}^S\!\eta_t$ & $0.15$ &
$-4.18{\times}10^{-2}$ &
$\tfrac{2h}{(S{-}a)^2}$ &
$\tfrac{2h}{(S{-}a_2)(S{+}a_2{-}2a_1)}$ \\[8pt]
$N{\cdot}\mathcal{E}{/}\int_0^{a_{c_1}}\!\eta_t$ & $0.15$ &
$-1.19{\times}10^{-1}$ &
$\tfrac{2Nh}{a(S{-}a)}$ &
$\tfrac{2Nh}{a_1(S{-}a_2)}$ \\[8pt]
$N{\cdot}\mathcal{E}{/}\int_{a_{c_2}}^S\!\eta_t$ & $0.15$ &
$2.18{\times}10^{-1}$ &
$\tfrac{2Nh}{(S{-}a)^2}$ &
$\tfrac{2Nh}{(S{-}a_2)(S{+}a_2{-}2a_1)}$ \\[8pt]
\midrule
\multicolumn{5}{@{}l}{\textbf{Bias Terms} (global scalar)} \\[3pt]
$N$ (model size) & $-0.25$ & $3.1{\times}10^{-1}$ & $N$ & $N$ \\[8pt]
$S$ (steps) & $-0.25$ & $6.98{\times}10^{-1}$ & $S$ & $S$ \\[8pt]
$\eta_{\max}$ & $0.20$ & $5.26{\times}10^{-2}$ & $h$ & $h$ \\[8pt]
$1$ & $1.0$ & $3.14{\times}10^{-1}$ & $1$ & $1$ \\[4pt]
\bottomrule
\end{tabular}
\end{minipage}
\end{table}

\subsubsection{Divergence Criterion Coefficients}
The parameters in $R(\eta_{\max}, a_1, N, S)$ (Eqn.~\eqref{eq:criterion}) are fitted on the train split of the small-scale training grid, where each configuration is labeled convergent or divergent based on whether the final loss exceeds 6. Before applying $R$, the variables $S$, $a_1$, $N$, and $\eta_{\max}$ are normalized following the same procedure as for the regression (Sec.~\ref{sec:fitted-coeff}). We found that squaring the normalized values of $S$ and $a_1$ before substitution into Eqn.~\eqref{eq:criterion} improves the fitting stability. The fitted parameters are $\hat{c}_1 = 1.76$, $\hat{c}_2 = 33.21$, $\hat{c}_3 = 292.03$, $\hat{\alpha}_1 = 0.218$, and $\hat{\alpha}_2 = 0.5$.

\subsection{\blue{Feature Ablation and Sensitivity Analysis}}\label{sec:ablation}

\blue{This subsection examines three aspects of the feature design: whether each of the four feature blocks (scale, convergence, escape, mixed) carries independent predictive signal, whether the selected exponent values are robust to perturbation, and whether including the decay phase $[a_1, a_2]$ in the convergence integral improves prediction. All experiments use the same 5-fold cross-validation protocol as in Sec.~\ref{sec:source-pool}.}

\paragraph{Feature-family ablation.}

\begin{table}[t]
% \color{blue} % disabled for arXiv
\centering
\small
\setlength{\tabcolsep}{5pt}
\renewcommand{\arraystretch}{1.2}
\caption{Feature-family ablation (5-fold CV, held-out).
Features are added by theoretical block: scale, convergence, escape, mixed interactions.
The final regression model contains 15 non-bias features plus one intercept term (16 fitted coefficients in total).}
\label{tab:feature-ablation}
\begin{tabular}{@{}l rr rr r@{}}
\toprule
& \multicolumn{2}{c}{\textbf{Fit Quality}} & \multicolumn{2}{c}{\textbf{Selection}} & \textbf{Regret} \\
\cmidrule(lr){2-3}\cmidrule(lr){4-5}
\textbf{Feature variant} & $R^2$ & Rel.Err.(\%) & Spearman & Top-2 (\%) & \\
\midrule
Scale only (3d)              & 0.992 & 1.06 &    0.60 & 51 & 0.027 \\
+ Convergence (7d)           & 0.992 & 1.06 &    0.61 & 66 & 0.016 \\
+ Escape (11d)               & 0.997 & 0.61 &    0.77 & 89 & 0.006 \\
\textbf{+ Mixed interactions (15d)}  & \textbf{0.998} & \textbf{0.50} & \textbf{0.84} & \textbf{94} & \textbf{0.003} \\
\bottomrule
\end{tabular}
\end{table}

\blue{Table~\ref{tab:feature-ablation} reports a progressive ablation in which feature blocks are added one at a time. Scale terms alone already capture moderate within-block ranking (Top-2 51\%) because the $\eta_{\max}$ feature carries direct schedule information. Adding the four convergence terms raises Top-2 to 66\% and halves the regret from 0.027 to 0.016, confirming that the cumulative optimization budget provides selection value beyond what $\eta_{\max}$ alone supplies. Escape-related terms contribute the single largest improvement (Top-2 from 66\% to 89\%, Spearman from 0.61 to 0.77), consistent with the SDE escape analysis: the derivative-driven fluctuation terms carry independent predictive signal. Adding mixed interaction terms brings the full model to Top-2 94\% and Spearman 0.84, capturing cross-scale modulation that single-block features cannot represent.}

\blue{\paragraph{Decay-phase ablation.}
As described in Sec.~\ref{sec:apply-opt-laws}, the convergence integral excludes the decay phase $[a_1, a_2]$, whose LR-transition information is captured by the escape feature $\int_0^{a_2}{\eta'_t}^2$. Table~\ref{tab:decay-ablation} tests the alternative of including the decay in the convergence integral ($a_{c_2}=a_1$ instead of $a_{c_2}=a_2$) on the polygon pre-training configurations in Table~\ref{tab:pre-train} and the continual-training configurations in Table~\ref{tab:continual-train} that have a non-trivial decay phase ($a_1 < a_2$). For pre-training, including the decay slightly increases the mean prediction error (0.15\% to 0.23\%), though both are well below 0.5\%. For continual training under strong distribution shift, both treatments produce identical rankings (Spearman~1.00). Under weak distribution shift the Spearman drops from 0.90 to 0.30, but this setting has actual losses within a 0.005 range (1.991--1.996), so the ranking task itself is near the resolution limit. Overall, the two treatments produce similar prediction quality, with the default ($a_{c_2}=a_2$) performing slightly better. This is consistent with the observation in Sec.~\ref{sec:apply-opt-laws} that the decay-phase information is already captured by the escape feature, so including it in the convergence integral provides little additional benefit.}

\begin{table}[t]
% \color{blue} % disabled for arXiv
\centering
\caption{\blue{Decay-phase ablation on multi-phase schedules ($a_1{<}a_2$). $a_{c_2}{=}a_2$ (default) starts the convergence tail at $a_2$, excluding the decay $[a_1,a_2]$. $a_{c_2}{=}a_1$ includes it. Left: prediction error (\%) on polygon pre-training configs (Table~\ref{tab:pre-train}). Right: predicted scores and ranking for continual training (Table~\ref{tab:continual-train}), with ranks in parentheses.}}
\label{tab:decay-ablation}
\vspace{4pt}
\small
\renewcommand{\arraystretch}{1.2}
\noindent
\begin{minipage}[t]{0.46\linewidth}
\centering
\setlength{\tabcolsep}{4pt}
\begin{tabular}{@{}l r rr@{}}
\toprule
& & \multicolumn{2}{c}{\textbf{Pred.\ Error (\%)}} \\
\cmidrule(l){3-4}
\textbf{Schedule} & \textbf{Actual} & ${=}a_2$ & ${=}a_1$ \\
\midrule
Polygon2       & 2.098 & \textbf{0.13} & 0.18 \\
\quad\scriptsize $a_2{=}7$k,\,$a_3{=}13$k & & & \\[4pt]
Polygon1v2     & 2.097 & \textbf{0.08} & 0.13 \\
\quad\scriptsize $a_2{=}5$k,\,$a_3{=}11.5$k & & & \\[4pt]
Polygon1       & 2.057 & \textbf{0.35} & 0.51 \\
\quad\scriptsize $a_2{=}10$k,\,$a_3{=}15$k & & & \\[4pt]
Loss-eq        & 2.079 & \textbf{0.05} & 0.09 \\
\quad\scriptsize $a_2{=}7$k,\,$a_3{=}12$k & & & \\[4pt]
\cmidrule{3-4}
\textbf{Mean}  &       & \textbf{0.15} & 0.23 \\[6pt]
\textbf{Max}   &       & \textbf{0.35} & 0.51 \\
\bottomrule
\end{tabular}
\end{minipage}%
\hfill%
\begin{minipage}[t]{0.53\linewidth}
\centering
\setlength{\tabcolsep}{3pt}
\begin{tabular}{@{}l r rr@{}}
\toprule
& & \multicolumn{2}{c}{\textbf{Pred.\ Score\,(Rank)}} \\
\cmidrule(l){3-4}
\textbf{Schedule} & \textbf{Actual} & ${=}a_2$ & ${=}a_1$ \\
\midrule
\multicolumn{4}{@{}l}{\textit{Strong distribution shift}} \\[1pt]
Polygon1       & 1.826 & 2.058\,(3) & 2.023\,(3) \\
Polygon2       & 1.821 & 2.014\,(2) & 1.983\,(2) \\
Linear         & 1.818 & 1.975\,(1) & 1.974\,(1) \\
\cmidrule{3-4}
\multicolumn{2}{@{}r}{\textit{Spearman}} & 1.00 & 1.00 \\
\midrule
\multicolumn{4}{@{}l}{\textit{Weak distribution shift}} \\[1pt]
Const+Lin      & 1.991 & 1.995\,(2) & 1.995\,(4) \\
Linear         & 1.992 & 1.975\,(1) & 1.975\,(1) \\
Polygon1v2     & 1.993 & 2.010\,(3) & 1.991\,(3) \\
Polygon2       & 1.995 & 2.017\,(4) & 1.987\,(2) \\
Polygon1       & 1.996 & 2.059\,(5) & 2.024\,(5) \\
\cmidrule{3-4}
\multicolumn{2}{@{}r}{\textit{Spearman}} & \textbf{0.90} & 0.30 \\
\bottomrule
\end{tabular}
\end{minipage}
\end{table}

\paragraph{Exponent sensitivity.}
\begin{table}[t]
% \color{blue} % disabled for arXiv
\centering
\vspace{2pt}
\small
\setlength{\tabcolsep}{6pt}
\renewcommand{\arraystretch}{1.2}
\caption{Exponent sensitivity (5-fold CV). One exponent group is set to the stated value while the rest remain at their defaults. The linear coefficients $\*c$ are refitted for each setting.}
\label{tab:alpha-sensitivity}
\begin{tabular}{@{}ll rrr@{}}
\toprule
\textbf{Perturbed group} & \textbf{$\bm{\alpha}$ value} & \textbf{$R^2$} & \textbf{Spearman} & \textbf{Top-2 (\%)} \\
\midrule
\textbf{Default (selected)} & \textbf{--} & \textbf{0.998} & \textbf{0.84} & \textbf{94} \\
\midrule
\multirow{3}{*}{Mixed interactions}   & 0.10 & 0.998 & 0.83 & 94 \\
                     & 0.20 & 0.998 & 0.84 & 94 \\
                     & 0.30 & 0.998 & 0.85 & 94 \\
\midrule
\multirow{3}{*}{Conv.\ cross-terms}   & 0.15 & 0.998 & 0.84 & 94 \\
                     & 0.30 & 0.998 & 0.84 & 94 \\
                     & 0.50 & 0.998 & 0.82 & 94 \\
\midrule
\multirow{2}{*}{Escape powers}        & 0.15 & 0.998 & 0.85 & 94 \\
                     & 0.50 & 0.998 & 0.84 & 94 \\
\midrule
\multirow{2}{*}{Scale powers}         & 0.15 & 0.997 & 0.79 & 94 \\
                     & 0.50 & 0.996 & 0.76 & 80 \\
\bottomrule
\end{tabular}
\end{table}

\blue{Table~\ref{tab:alpha-sensitivity} perturbs one exponent group at a time while keeping all others at their defaults and refitting $\*c$ on the training split. For the three schedule-dependent groups (mixed interactions, convergence cross-terms, and escape powers), perturbing $\bm{\alpha}$ within the candidate range leaves Top-2 at 94\% and $R^2 \geq 0.998$. This robustness indicates that the SDE-derived functional forms already impose sufficient structural constraint, and the exact exponent values play a refinement role rather than a defining one. The scale-power group is the only sensitive component: at $\alpha{=}0.50$, Top-2 drops from 94\% to 80\%. Unlike the schedule-dependent terms, scale terms have no underlying SDE structure and rely entirely on the exponent to capture how model size modulates the prediction across a wide range ($8{\times}0.001$B to $8{\times}0.3$B). Together with the feature-family ablation above, these results suggest that the primary predictive signal comes from the feature structure itself, not from the specific exponent values.}

\subsection{\blue{Divergence Prediction Benchmark}}\label{sec:heldout-benchmark}

\blue{This subsection evaluates the $R$ criterion (Eqn.~\eqref{eq:criterion}) as a standalone divergence classifier. The training grid contains 260 configurations (234 convergent, 26 divergent), evaluated under the same 5-fold cross-validation protocol as the loss regression. Because of the class imbalance, we report precision, recall, F1, and balanced accuracy rather than overall accuracy. Three baselines are included: a majority-class predictor (all convergent), a peak-$\eta_{\max}$ threshold, and logistic regression on $(\eta_{\max}, \text{warmup}, N, S)$.}

\begin{table}[t]
\centering
\vspace{2pt}
\small
\setlength{\tabcolsep}{6pt}
\renewcommand{\arraystretch}{1.2}
\caption{\blue{Divergence prediction benchmark (5-fold CV, 234 convergent vs.\ 26 divergent configurations). The $R$ criterion is compared against three baselines. Precision, recall, F1, and balanced accuracy are reported to account for class imbalance.}}
\label{tab:divergence-benchmark}
\begin{tabular}{@{}l rrrr@{}}
\toprule
\textbf{Method} & \textbf{Precision} & \textbf{Recall} & \textbf{F1} & \textbf{Bal.\ Accuracy} \\
\midrule
Majority class (all conv.)       & 0.00 & 0.00 & 0.00 & 0.50 \\
Peak-$\eta_{\max}$ threshold     & 0.29 & 0.58 & 0.38 & 0.71 \\
Logistic regression (LOO)        & 0.73 & 0.42 & 0.54 & 0.70 \\
\midrule
\textbf{$R$ criterion (5-fold)} & \textbf{0.96} & \textbf{0.88} & \textbf{0.92} & \textbf{0.94} \\
\bottomrule
\end{tabular}
\end{table}

\blue{Table~\ref{tab:divergence-benchmark} shows that the $R$ criterion achieves F1\,=\,0.92 and balanced accuracy 0.94, well above all three baselines. Divergence in the training grid is not determined by $\eta_{\max}$ alone: the same peak LR can produce convergence or divergence depending on the warmup duration and model/data scale. A simple threshold on $\eta_{\max}$ misses this interaction (F1\,=\,0.38), and logistic regression has too few positive examples to learn the nonlinear boundary from raw features (F1\,=\,0.54). The $R$ criterion encodes the interaction through its parameterized form (Eqn.~\eqref{eq:criterion}), which is why it achieves high precision (0.96) with relatively few false alarms. The lower recall (0.88) reflects a few borderline configurations where a long warmup nearly offsets a high peak LR, placing them close to the convergence/divergence boundary.}

\subsection{\blue{Schedule-Family Generalization}}\label{sec:transfer-results}

\blue{The generalized \abbr{} is fitted on piecewise-linear schedules. To test cross-family generalization, we evaluate it on cosine and constant-plus-linear cooldown schedules that are not included in the fitting data. The convergence and escape integrals ($\int\eta$ and $\int{\eta'_t}^2$) have exact closed-form expressions for all three schedule families, so the feature computation does not introduce any approximation. We compare the three cooldown shapes on five (model, token) groups spanning $8{\times}0.1$B and $8{\times}0.6$B, where each group shares the same $\eta_{\max}$, $a_1$, and $S$.}

\begin{table}[t]
\centering
\caption{\blue{Schedule-family generalization. The generalized \abbr{}, fitted only on piecewise-linear schedules, is evaluated on cosine and constant-plus-linear (Const+Lin) cooldown shapes. Within each group the three schedules share the same $\eta_{\max}$, $a_1$, and $S$, differing only in cooldown shape. \textbf{Bold} marks the actual-best schedule, correctly identified in all five groups. All errors are below 2\%.}}
\label{tab:schedule-transfer}
\vspace{2pt}
\small
\setlength{\tabcolsep}{4pt}
\renewcommand{\arraystretch}{1.15}
\noindent
\begin{minipage}[t]{0.49\linewidth}
\centering
\begin{tabular}{@{}ll rrr@{}}
\toprule
\textbf{Model} & \textbf{Cooldown} & \textbf{Act.} & \textbf{Pred.} & \textbf{Err.(\%)} \\
\textbf{Tokens} & & & & \\
\midrule
\multirow{3}{*}{\shortstack[l]{$8{\times}0.1$B\\3B}}
 & Linear                            & 3.177 & 3.179 & 0.1 \\
 & Cosine                            & 3.205 & 3.175 & 0.9 \\
 & \textbf{Const+Lin}               & \textbf{3.084} & \textbf{3.100} & \textbf{0.5} \\[4pt]
\multirow{3}{*}{\shortstack[l]{$8{\times}0.1$B\\10B}}
 & Linear                            & 2.689 & 2.693 & 0.1 \\
 & Cosine                            & 2.695 & 2.689 & 0.2 \\
 & \textbf{Const+Lin}               & \textbf{2.667} & \textbf{2.653} & \textbf{0.5} \\[4pt]
\multirow{3}{*}{\shortstack[l]{$8{\times}0.1$B\\100B}}
 & Linear                            & 2.369 & 2.332 & 1.5 \\
 & Cosine                            & 2.374 & 2.329 & 1.9 \\
 & \textbf{Const+Lin}               & \textbf{2.358} & \textbf{2.312} & \textbf{1.9} \\
\bottomrule
\end{tabular}
\end{minipage}%
\hfill
\begin{minipage}[t]{0.49\linewidth}
\centering
\begin{tabular}{@{}ll rrr@{}}
\toprule
\textbf{Model} & \textbf{Cooldown} & \textbf{Act.} & \textbf{Pred.} & \textbf{Err.(\%)} \\
\textbf{Tokens} & & & & \\
\midrule
\multirow{3}{*}{\shortstack[l]{$8{\times}0.6$B\\10B}}
 & Linear                            & 2.501 & 2.526 & 1.0 \\
 & Cosine                            & 2.519 & 2.524 & 0.2 \\
 & \textbf{Const+Lin}               & \textbf{2.480} & \textbf{2.484} & \textbf{0.2} \\[4pt]
\multirow{3}{*}{\shortstack[l]{$8{\times}0.6$B\\3B}}
 & Linear                            & 2.995 & 2.988 & 0.2 \\
 & Cosine                            & 3.037 & 2.986 & 1.7 \\
 & \textbf{Const+Lin}               & \textbf{2.956} & \textbf{2.908} & \textbf{1.6} \\
& & & & \\
& & & & \\
& & & & \\
\bottomrule
\end{tabular}
\end{minipage}
\end{table}

\blue{In all five groups, the fitted law correctly identifies the best schedule family, with prediction errors below 2\% (Table~\ref{tab:schedule-transfer}). For the $8{\times}0.1$B model at 3B and 10B tokens, most errors are below 1\%, comparable to the in-family results in Table~\ref{tab:pre-train}. This cross-family accuracy is consistent with the continual-training and fine-tuning evaluations (Tables~\ref{tab:continual-train} and~\ref{tab:finetune}), where polygon, constant-plus-linear, and two-phase schedules outside the fitting family are also correctly ranked. The theoretical basis for this generalization is that the SDE convergence and escape bounds (Theorems~\ref{thm:covergence-SGD}--\ref{thm:Adam-prob}) depend on the schedule only through $\int\eta$ and $\int{\eta'_t}^2$. These two integrals summarize the schedule's contribution to optimization progress and basin selection regardless of its parametric form, and all tested schedule families admit exact closed-form expressions for both quantities. The fitted law therefore transfers across families without re-fitting.}

\section{\blue{Theory Supplement}}\label{sec:theory-supplement}
\blue{This section collects the proofs and derivations underlying the theoretical results stated in the main text.}

\subsection{\blue{Analysis for \abbr{}, SGD-SDE, and Adam-SDE}}
\label{sec:supp:SDE}
\subsubsection{Proof of Proposition~\ref{prop:effect}}
\begin{proof}
Based on the proposed \abbr{}, the analytical expressions for the two LR schedules, $\eta_{\operatorname{cos}}(t) $ and $\eta_{\operatorname{const}}(t)$, can be derived as follows:
\[
\begin{aligned}      \operatorname{\abbr{}}\qty(\eta_{\operatorname{cos}}) =&c_1\qty(\frac{2}{\eta_{\max}a})^{\alpha_1}+c_2\qty(\frac{2}{\eta_{\max}(S-a)})^{\alpha_2} +\frac{c_3}{S}+b \\
 & +c_4\qty(\frac{\pi^2\eta_{\max}^2}{8(S-a)})^{\alpha_3} + c_5\qty(\frac{\eta_{\max}^2}{a})^{\alpha_4}.     \\
\operatorname{\abbr{}}\qty(\eta_{\operatorname{const}}) =&c_1\qty(\frac{2}{\eta_{\max}a})^{\alpha_1}+c_2\qty(\eta_{\max}(a_c-a)+\frac{\eta_{\max}(S-a_c)}{2})^{-\alpha_2}\\
&+\frac{c_3}{S}+b+c_3\qty(\frac{\eta_{\max}^2}{S-a_c})^{\alpha_3}+c_5\qty(\frac{\eta_{\max}^2}{a})^{\alpha_4}.
\end{aligned}
\]
Note that $\lim_{S\rightarrow\infty}\operatorname{\abbr{}}\qty(\eta_{\operatorname{cos}}(\cdot))=\lim_{S\rightarrow\infty}\operatorname{\abbr{}}\qty(\eta_{\operatorname{const}}(\cdot))=b$. Therefore, we have $\lim_{S\rightarrow\infty}\abs{\operatorname{\abbr{}}\qty(\eta_{\operatorname{cos}}(\cdot))-\operatorname{\abbr{}}\qty(\eta_{\operatorname{const}}(\cdot))}=0$.
\end{proof}
\subsubsection{Derivations for SGD-SDE and Adam-SDE}
\label{sec:sgd-sde-derivation}
In this subsection, we derive the SDEs to model the iterative SGD  and Adam sequences.

\vspace{1em}
\noindent\textbf{SGD-SDE:} The iterative sequence of SGD is given by
\begin{equation}
\label{eq:sgd_append}
\*x_{k+1}=\*x_k-\eta_0\eta_k(\nabla f(\*x_k)+\*z_k),\;\;\*z_k\sim\mathcal{N}\qty(0,\*\Sigma(\*x_k)),
\end{equation}
where $\eta_k$ is the normalized learning rate, $\eta_0$ is a small rescaling parameter, and $\xi_k$ is a Gaussian noise.
% The next proposition shows that \eqref{eq:sgd_append} is the Euler-Maruyama method for the SDE \eqref{eq:SGD-SDE}.
% \begin{proposition}
%     Given $T>0$
% \end{proposition}
Applying the Euler-Maruyama method (for a detailed description, see \cite[Chapter 8.2]{sarkka2019applied}) to the corresponding SDE \eqref{eq:SGD-SDE}, we obtain:
\[
\*x_{k+1}=\*x_k-\eta(t_k)\nabla f(\*x_k)\Delta t_{k+1}+\sqrt{\eta_0}\eta(t_k)({\Delta t_{k+1}}\*\Sigma(\*x_k))^{\frac{1}{2}}\Delta\*W_k,
\]
where $\Delta t_{k+1}=t_{k+1}-t_k$ and $\Delta\*W_k\sim\mathcal{N}(0,\*I_N)$. By setting $\Delta t_{k+1}\equiv\eta_0$, the discrete scheme exactly recovers the SGD sequence \eqref{eq:sgd_append}. With the same initial conditions, where $\*x_0=\*X_0$, and under certain smooth regularity conditions on the functions involved, it can be established that a positive constant $\alpha>0$ (referred to as the order) exists satisfying the following property. For any time horizon $T>0$ and positive integer $m\leq\floor{\frac{T}{\eta_0}}$, there exists a constant $K$ and a sufficiently small $\eta_0$ such that the following strong error bound holds
\[
\E\qty[\norm{\*X_{m\eta_0}-\*x_{m}}]\leq K\eta_0^\alpha.
\]
The classical approximation order for the Euler-Maruyama method is typically $\alpha=\frac{1}{2}$, which was proven in \cite{gihman1972stochastic}. We do not delve into the details on this topic, as it is beyond the scope of this work.  For a comprehensive discussion on the error analysis of numerical approximations for SDEs, interested readers may consult \cite{kloeden1999stochastic,sarkka2019applied}. Additionally, for insights into the errors associated with modeling SGD sequences using SDEs, one can refer to the works of \cite{li2019stochastic,li2021validity}.

\vspace{1em}
\noindent\textbf{Adam-SDE:} By applying the Euler-Maruyama method for the SDE associated with $\*m_t$ in \eqref{eq:Adam-SDE}, we obtain:
\[
\begin{aligned}
\*m_{k+1}=&\*m_k-c_1\eta(t_k)(\*m_k-\nabla f(\*x_k))\Delta{t_{k+1}}+c_1'\eta(t_{k+1})(\Delta{t_{k+1}}\*\Sigma(\*x_k))^{\frac{1}{2}}\Delta\*W_k\\
=&(1-c_1\Delta{t_{k+1}}\eta_k)\*m_k+c_1\Delta{t_{k+1}}\eta_k\left(\nabla f(\*x_k)+\frac{c_1'}{c_1\sqrt{\Delta{t_{k+1}}}}(\*\Sigma(\*x_k))^{\frac{1}{2}}\Delta\*W_k\right).
\end{aligned}
\]
Setting $\Delta{t_{k+1}}=\frac{\hat c_1}{c_1}$ and $c_1'=c_1\sqrt{\Delta{t_{k+1}}}=\sqrt{c_1\hat c_1}$, we derive that
\[
\*m_{k+1}=(1-\beta_{1,k})\*m_k+\beta_{1,k}(\nabla f(\*x_k)+\*z_k),
\]
where $\beta_{1,k}=1-\hat c_1\eta_k$, $\*z_k\sim\mathcal{N}(0,\*\Sigma(\*x_k))$. This recovers the update of $\{\*m_k\}$ in \eqref{eq:adam}. Furthermore, the update of ${\*x_k}$ corresponds to Euler's method for the ODE associated with $\*X_t$ in \eqref{eq:Adam-SDE}. The update of ${\*v_k}$ aligns with the SDE associated with $\*v_t$ in \eqref{eq:Adam-SDE}, which directly follows from \cite[Theorem 4.2 or Theorem 4.5]{malladi2022sdes}.

\subsubsection{Proof of Proposition~\ref{prop:sigma_bound}}
\begin{proof}
Let $\*z_i:=\nabla F(\*x,\*\zeta_i)-\nabla f(\*x)$ and $\*Z:=[\*z_1,\*z_2,...\*z_D]$, then ${\*\Sigma}(\*x)=\frac{1}{D}\*Z\*Z^\top$. Therefore, $\Tr(\*\Sigma(\*x))=\frac{1}{D}\sum_{i=1}^D\Tr(\*z_i\*z_i^\top)=\frac{1}{D}\sum_{i=1}^D\norm{\*z_i}^2$. Note that $\*z_i \sim N(0, \*\Sigma_g)$ are i.i.d. Gaussian vectors, where $\*\Sigma_g$ is an $N \times N$ positive semidefinite covariance matrix. So next, we analyze the concentration of $\frac{1}{D}\sum_{i=1}^D \|\*z_i\|^2$. Each $\|\*z_i\|^2$ can be represented as a sum of weighted $\chi^2$ variables:
\[
\|\*z_i\|^2 \sim \sum_{j=1}^n \lambda_j \chi_{ij}^2,
\]
where $\lambda_j$ are the eigenvalues of $\*\Sigma_g$. The expected value and variance of $\|\*z_i\|^2$ are given by:
\[
\mathbb{E}[\|\*z_i\|^2] = \operatorname{tr}(\*\Sigma_g), \quad \text{Cov}(\|\*z_i\|^2) = 2 \operatorname{Tr}(\*\Sigma_g^2).
\]
Given that $\chi^2$ distribution is sub-exponential, we apply Bernstein's inequality to estimate the concentration of $\frac{1}{D}\sum_{i=1}^D \|\*z_i\|^2$, which is given by:
\[
\mathbb{P}\left(\left|\frac{1}{D}\sum_{i=1}^D \|\*z_i\|^2 - \operatorname{Tr}(\*\Sigma_g)\right| \geq t\right) \leq 2 \exp\left(-\frac{D t^2}{4 \operatorname{Tr}(\*\Sigma_g^2) + 2t \|\*\Sigma_g\|_{\rm op}}\right),
\]
where $\|\*\Sigma_g\|_{\rm op}$ is the operator norm (largest eigenvalue) of $\*\Sigma_g$. This completes the proof.

\end{proof}

\subsubsection{Proof of Proposition~\ref{prop:sigma_op_bound}}
\begin{proof}
Let $\*Z_i:=\*\Sigma_g^{-\frac{1}{2}}\qty(\nabla F(\*x,\*\zeta_i)-\nabla f(\*x))$ and $\*Z:=[\*Z_1,\*Z_2,...\*Z_D]$. Then, $\{\*Z_i\}$ are standard Gaussian variables. Let $\hat{\*\Sigma}(\*x):=\frac{1}{D}\*Z\*Z^T$. By the eigenvalue variance bounds for covariance matrices in \cite[Corollary 3]{ledoux2010small}, there exists a constant $C>0$ such that
\[
\begin{aligned}
&\E\qty[\lambda_{\max}\qty(\hat{\*\Sigma}(\*x))]\leq\E\qty[\abs{\lambda_{\max}\qty(\hat{\*\Sigma}(\*x))-\qty(1+\sqrt{\frac{D}{N}})}]+\qty(1+\sqrt{\frac{D}{N}})\\
\leq&\sqrt{\E\qty[\qty(\lambda_{\max}\qty(\hat{\*\Sigma}(\*x))-\qty(1+\sqrt{\frac{D}{N}}))^2]}+\qty(1+\sqrt{\frac{D}{N}})\\
\leq&\frac{C}{N^{\frac{2}{3}}}+\qty(1+\sqrt{\frac{D}{N}}).
\end{aligned}
\]
Furthermore, note that $\*\Sigma(\*x)=\*\Sigma_g^{\frac{1}{2}}\hat{\*\Sigma}(\*x)\*\Sigma_g^{\frac{1}{2}}$, it follows that:
\[
\E\qty[\lambda_{\max}({\*\Sigma}(\*x))]\leq\norm{{\*\Sigma_g}}_{\rm op}\E\qty[\norm{{\hat{\*\Sigma}}(\*x)}_{\rm op}]=\qty(1+\sqrt{\frac{D}{N}})\sigma_g^2 + \frac{C\sigma_g^2}{N^{\frac{2}{3}}},
\]
where $\norm{\cdot}_{\rm op}$ is the operator norm of matrix. This completes the proof.
\end{proof}

\subsection{\blue{Proofs for Convergence Analysis}}
\label{sec:supp:convergence}
\subsubsection{Proof of Theorem~\ref{thm:covergence-SGD}}
\begin{proof}
Apply the \ito~formula to $f(\*X_t)$ and utilize the definition of~\eqref{eq:SGD-SDE}, we obtain:
\[
\begin{aligned}
\dd f(\*X_t)=&\inner{\nabla f(\*X_t),\dd \*X_t}+\frac{1}{2}\inner{\nabla^2f(\*X_t)\dd \*X_t,\dd \*X_t}\\
=&-\eta(t)\norm{\nabla f(\*X_t)}^2\dd t+\eta_0\eta(t)\inner{\nabla f(\*X_t),\sigma(\*X_t)\dd\*W_t}\\
&+\frac{\eta_0\eta(t)^2}{2}\inner{\nabla^2f(\*X_t),\sigma(\*X_t)\sigma(\*X_t)^\top}\dd t.
\end{aligned}
\]
Taking the integral and then taking the expectation, we have:
\[
\begin{aligned}
\E[f(\*X_t)-f(\*X_0)]=& -\E\qty[\int_0^t\eta(s)\norm{\nabla f(\*X_s)}^2\dd s]\\
& +\E\qty[\int_0^t\frac{\eta_0\eta(s)^2}{2}\inner{\nabla^2f(\*X_s),\sigma(\*X_s)\sigma(\*X_s)^\top}\dd s]\\
\leq&-\E\qty[\int_0^t\eta(s)\norm{\nabla f(\*X_s)}^2\dd s]+\qty(\frac{\eta_0LN}{2}\int_0^t\eta(s)^2\dd s)\sup_{\*x\in\bb{R}^N}\E[\norm{\*\Sigma(\*x_s)}_{\rm op}].
\end{aligned}
\]
Rearranging the terms and dividing both sides by $\int_0^t\eta(s)\dd s$, and by Proposition \ref{prop:sigma_op_bound}, we have
\[
\E\qty[\int_0^t\frac{\eta(s)}{\int_0^t\eta(s)\dd s}\norm{\nabla f(\*X_s)}^2\dd s]\leq\frac{f(\*X_0)-f_{\min}}{\int_0^t\eta(s)\dd s}+\frac{\eta_0L\sigma_0^2N\int_0^t\eta(s)^2\dd s}{2\int_0^t\eta(s)\dd s}.
\]
This completes the proof.
\end{proof}
\subsubsection{Proof of Proposition \ref{prop:adam-bound}}
\begin{proof}
Denote $\bar{\*m}_t:=\E\qty[\*m_t]$ and $\bar{\*P}_t:={\rm Cov}\qty(\*m_t\*m_t^\top)$. Then, $(\bar{\*m}_t,\bar{\*P}_t)$ satisfies the following ODE:
\[
\left\{
\begin{aligned}
\frac{\dd\bar{\*m}_t}{\dd t}=&-c_1\eta(t)\bar{\*m}_t+c_1\eta(t)\nabla f(\*X_t)\\
\frac{\dd\bar{\*P}_t}{\dd t}=&-2c_1\eta(t)\bar{\*P}_t+(c_1')^2\eta(t)^2\*\sigma(\*X_t)\*\sigma(\*X_t)^\top.
\end{aligned}
\right.
\]
Let $\*\Sigma_t:=\*\Sigma(\*X_t) =\*\sigma(\*X_t)\*\sigma(\*X_t)^\top$. The analytic solution to these equations is given by:
\[
\left\{
\begin{aligned}
\bar{\*m}_t=&\exp\qty(-c_1\int_0^t\eta(s)\dd s)\left[\int_0^t\exp\qty(c_1\int_0^s\eta(\tau)\dd\tau)c_1\eta(s)\nabla f(\*X_s)\dd s\right]\\
\bar{\*P}_t=&\exp\qty(2c_1\int_0^t\eta(s)\dd s)\left[\int_0^t\exp\qty(2c_1\int_0^s\eta(\tau)\dd\tau)(c_1')^2\eta(s)^2\*\Sigma(\*X_s)\dd s\right].\\
\end{aligned}
\right.
\]
Assumption \ref{asm:eta}.1 implies that $\sup_{\*x\in\bb{R}^N}\norm{\nabla f(\*x)}\leq\ell$. Therefore, we have that
\[
\begin{aligned}
\norm{\bar{\*m_t}}\leq&\ell\exp\qty(-c_1\int_0^t\eta(s)\dd s)\left[\int_0^t\exp\qty(c_1\int_0^s\eta(\tau)\dd\tau)c_1\eta(s)\dd s\right],\\
\norm{\bar{\*P}_t}_{\rm op}\leq&(c_1')^2\bar\sigma\exp\qty(-2c_1\int_0^t\eta(s)\dd s)\left[\int_0^t\exp\qty(2c_1\int_0^s\eta(\tau)\dd\tau)\eta(s)^2\dd s\right]\\
\leq&(c_1')^2\bar\sigma\exp\qty(-2c_1\int_0^t\eta(s)\dd s)\left[\int_0^t\exp\qty(2c_1\int_0^t\eta(\tau)\dd\tau)\eta(s)^2\dd s\right]\\
\leq&(c_1')^2\bar\sigma\int_0^t\eta(s)^2\dd s\leq(c_1')^2\bar\sigma\Gamma<\infty,
\end{aligned}
\]
where $\Gamma:=\int_0^\infty\eta(s)^2\dd s$. By L'Hopital's rule, we have that
\[
\begin{aligned}
&\lim_{t\rightarrow\infty}\exp\qty(-c_1\int_0^t\eta(s)\dd s)\left[\int_0^t\exp\qty(c_1\int_0^s\eta(\tau)\dd\tau)c_1\eta(s)\dd s\right]\\
=&\lim_{t\rightarrow\infty}\frac{c_1\exp\qty(c_1\int_0^t\eta(s)\dd s)\eta(t)}{c_1\exp\qty(c_1\int_0^t\eta(s)\dd s)\eta(t)}=1.
\end{aligned}
\]
Thus, we have $\limsup_t\norm{\bar{\*m_t}}\leq\ell$. Consequently, there exists a constant $C$ such that $\sup_t\norm{\bar{\*m}_t}\leq C<\infty$. By definition, $\E[\*m_t\*m_t^\top]=\bar{\*P}_t+\bar{\*m}_t\bar{\*m}_t^\top$. Therefore, for any $t\geq0$, we have $\norm{\E[\*m_t\*m_t^\top]}_{\rm op}\leq(c_1')^2\bar\sigma\Gamma+C^2$. Note that
\[
\E\qty[\norm{\*m_t}^2]={\rm Tr}\qty(\E[\*m_t\*m_t^\top])\leq{N}\norm{\E[\*m_t\*m_t^\top]}_{\rm op}.
\]
This proves the boundness of $\E[\norm{\*m_t}^2]$. Next, we prove the boundness of $\*v_t$. By \eqref{eq:Adam-SDE}, the solution of $\*v_t$ is given by
\[
\*v_t=\exp\qty(-\int_0^tc_2\eta(s)\dd s)\int_0^t\exp\qty(\int_0^sc_2\eta(\tau)\dd\tau)c_2\eta(s)\*d_s\dd s,
\]
where $\*d_s:={\rm diag}(\*\Sigma(\*X_s))$. By L'Hopital's rule, we have
\[
\begin{aligned}
\limsup_{t\rightarrow\infty}\norm{\*v_t}_\infty\leq&\sup_t\norm{\*d_t}_\infty\lim_{t\rightarrow\infty}\exp\qty(-\int_0^tc_2\eta(s)\dd s)\int_0^t\exp\qty(\int_0^sc_2\eta(\tau)\dd\tau)c_2\eta(s)\dd s\\
\leq&\sup_t\norm{\*d_t}_\infty=\bar\sigma,
\end{aligned}
\]
Therefore, there exists constant $V>0$ such that $\sup_{t\geq0}\norm{\*v_t}_\infty\leq V$. This completes the proof.
\end{proof}

\subsubsection{Proof of Theorem~\ref{thm:covergence-adam}}
\begin{proof}
Apply \ito's formula to $\phi_1(\*X_t,\*m_t,\*v_t):=f(\*X_t)+\frac{1}{2c_1}\inner{(\*v_t+\epsilon)^{-\frac{1}{2}}\odot\*m_t,\*m_t}$, we have
\[
\begin{aligned}
&\dd\phi_1(\*X_t,\*m_t,\*v_t)\\
=&-\eta(t)\inner{\nabla f(\*X_t),(\*v_t+\epsilon)^{-\frac{1}{2}}\odot\*m_t}\dd t-\eta(t)\inner{\*m_t-\nabla f(\*X_t),(\*v_t+\epsilon)^{-\frac{1}{2}}\odot\*m_t}\dd t\\
&+\frac{c_2\eta(t)}{4c_1}\inner{(\*v_t+\epsilon)^{-\frac{3}{2}}\odot\*m_t^2,\*v_t-{\rm diag}(\*\Sigma(\*X_t))}\dd t+\frac{(c_1')^2\eta(t)^2}{2c_1}\inner{{\rm Diag}((\*v_t+\epsilon)^{-\frac{1}{2}}),\*\Sigma(\*X_t)}\dd t\\
&+\frac{c_1'}{c_1}\eta(t)\inner{(\*v_t+\epsilon)^{-\frac{1}{2}}\odot\*m_t,\sigma(\*X_t)\dd\*W_t}\\
\leq&-\qty(1-\frac{c_2}{4c_1})\eta(t)\inner{(\*v_t+\epsilon)^{-\frac{1}{2}}\odot\*m_t,\*m_t}\dd t+\frac{(c_1')^2\eta(t)^2}{2c_1}\inner{{\rm Diag}((\*v_t+\epsilon)^{-\frac{1}{2}}),\*\Sigma(\*X_t)}\dd t\\
&+\frac{c_1'}{c_1}\eta(t)\inner{(\*v_t+\epsilon)^{-\frac{1}{2}}\odot\*m_t,\sigma(\*X_t)\dd\*W_t},
\end{aligned}
\]
where the last inequality is derived by noting that $\*v_t\geq0$ and $(\*v_t+\epsilon)^{-\frac{3}{2}}\odot\*v_t\leq(\*v_t+\epsilon)^{-\frac{1}{2}}$. Taking integral and then taking expectation, we have
\[
\begin{aligned}
&\E[\phi_1(\*X_t,\*m_t,\*v_t)-\phi_1(\*X_0,\*m_0,\*v_0)]\\
\leq&-\E\left[\int_0^t\qty(1-\frac{c_2}{4c_1})\eta(s)\inner{(\*v_s+\epsilon)^{-\frac{1}{2}}\odot\*m_s,\*m_s}\dd s\right]\\
&+\E\left[\int_0^t\frac{(c_1')^2\eta(s)^2}{2c_1}\inner{{\rm Diag}((\*v_s+\epsilon)^{-\frac{1}{2}}),\*\Sigma(\*X_s)}\dd s\right].
\end{aligned}
\]
Therefore, it holds that
\begin{equation}
\E\left[\overline{\norm{\*m_t}^2}\right]\leq\frac{\sqrt{V+\epsilon}\left(\phi_1(\*X_0,\*m_0,\*v_0)- f_{\min}\right)}{\qty(1-\frac{c_2}{4c_1})\int_0^t\eta(s)\dd s}+\frac{\frac{(c_1')^2}{2c_1} \bar\sigma\sqrt{V+\epsilon}\int_0^t\eta(s)^2\dd s}{\qty(1-\frac{c_2}{4c_1})\sqrt{\epsilon}\int_0^t\eta(s)\dd s}.
\label{eq:adam_m_est}
\end{equation}
Next, we derive the bound for the gradient. We construct a novel Lyapunov function $\phi_2(\*X_t,\*m_t,\*v_t):=f(\*X_t)-\frac{1}{c_1}\inner{\nabla f(\*X_t),(\*v_t+\epsilon)^{-\frac{1}{2}}\odot\*m_t}$ which links the noiseless gradient with the momentum. Applying \ito's formula to $\phi_2$ yields that
\[
\begin{aligned}
&\dd\phi_2(\*X_t,\*m_t,\*v_t)\\
=&-\eta(t)\inner{\nabla f(\*X_t),(\*v_t+\epsilon)^{-\frac{1}{2}}\odot\*m_t}\dd t+\frac{\eta(t)}{c_1}\inner{\nabla^2f(\*X_t)(\*v_t+\epsilon)^{-\frac{1}{2}}\odot\*m_t,(\*v_t+\epsilon)^{-\frac{1}{2}}\odot\*m_t}\dd t\\
&+\eta(t)\inner{\nabla f(\*X_t)\odot(\*v_t+\epsilon)^{-\frac{1}{2}},\*m_t-\nabla f(\*X_t)}\dd t-\frac{c_1'\eta(t)}{c_1}\inner{\nabla f(\*X_t)\odot(\*v_t+\epsilon)^{-\frac{1}{2}},\sigma(\*X_t)\dd\*W_t}\\
&-\frac{c_2\eta(t)}{2c_1}\inner{\nabla f(\*X_t)\odot\*m_t\odot(\*v_t+\epsilon)^{-\frac{3}{2}},\*v_t-{\rm diag}(\*\Sigma(\*X_t))}\dd t.
\end{aligned}
\]
Taking integral and taking expectation, we have
\[
\begin{aligned}
&\E\left[\phi_2(\*X_0,\*m_0,\*v_0)-\phi_2(\*X_t,\*m_t,\*v_t)\right]  \\
\leq&\frac{L}{c_1}\E\qty[\int_0^t\eta(s)\norm{\*m_s\odot(\*v_s+\epsilon)^{-\frac{1}{2}}}^2\dd s]-\E\qty[\int_0^t\eta(s)\inner{\nabla f(\*X_s)\odot(\*v_s+\epsilon)^{-\frac{1}{2}},\nabla f(\*X_s)}\dd s]\\
&-\frac{c_2}{2c_1}\E\qty[\int_0^t\eta(s)\inner{\nabla f(\*X_s)\odot\*m_s\odot(\*v_s+\epsilon)^{-\frac{3}{2}},\*v_s}\dd s]\\
&+\frac{c_2\bar\sigma}{2c_1}\E\qty[\int_0^t\eta(s)\inner{\nabla f(\*X_s)\odot\*m_s,(\*v_s+\epsilon)^{-\frac{3}{2}}}\dd s]\\
\leq&\frac{L}{c_1\epsilon}\E\qty[\int_0^t\eta(s)\norm{\*m_s}^2\dd s]-\frac{1}{\sqrt{V+\epsilon}}\E\qty[\int_0^t\eta(s)\norm{\nabla f(\*X_s)}^2\dd s]\\
&+\frac{1}{4\sqrt{V+\epsilon}}\E\qty[\int_0^t\eta(s)\norm{\nabla f(\*X_s)}^2\dd s]+\frac{c_2^2\sqrt{V+\epsilon}}{4c_1^2\epsilon}\E\qty[\int_0^t\eta(s)\norm{\*m_s}^2\dd s]\\
&+\frac{1}{4\sqrt{V+\epsilon}}\E\qty[\int_0^t\eta(s)\norm{\nabla f(\*X_s)}^2\dd s]+\frac{c_2^2\bar\sigma^2\sqrt{V+\epsilon}}{4\epsilon^3c_1^2}\E\qty[\int_0^t\eta(s)\norm{\*m_s}^2\dd s],\\
=&\frac{L}{c_1\epsilon}\E\qty[\int_0^t\eta(s)\norm{\*m_s}^2\dd s]-\frac{1}{2\sqrt{V+\epsilon}}\E\qty[\int_0^t\eta(s)\norm{\nabla f(\*X_s)}^2\dd s]\\
&+\qty(1+\frac{\bar\sigma^2}{\epsilon^2})\frac{c_2^2\sqrt{V+\epsilon}}{4c_1^2\epsilon}\E\qty[\int_0^t\eta(s)\norm{\*m_s}^2\dd s].
\end{aligned}
\]
The last inequality comes from Cauchy-Young's inequality:
\[
\begin{aligned}
&\frac{c_2}{2c_1}\bigg|\inner{\nabla f(\*X_s)\odot\*m_s\odot(\*v_s+\epsilon)^{-\frac{3}{2}},\*v_s}\bigg|=\frac{c_2}{2c_1}\bigg|\inner{\nabla f(\*X_s)\odot\*v_s\odot(\*v_s+\epsilon)^{-\frac{3}{2}},\*m_s}\bigg|\\
\leq&\frac{c_2}{2c_1\sqrt{\epsilon}}\bigg|\inner{\nabla f(\*X_s),\*m_s}\bigg|\leq\frac{c_2}{2c_1\sqrt{\epsilon}}\norm{\nabla f(\*X_s)}\norm{\*m_s}\leq\frac{1}{4\sqrt{V+\epsilon}}\norm{\nabla f(\*X_s)}^2+\frac{c_2^2\sqrt{V+\epsilon}}{4c_1^2\epsilon}\norm{\*m_s}^2.
\end{aligned}
\]
and similarly,
\[
\frac{c_2\bar\sigma}{2c_1}\bigg|\inner{\nabla f(\*X_s)\odot\*m_s,(\*v_s+\epsilon)^{-\frac{3}{2}}}\bigg|\leq\frac{1}{4\sqrt{V+\epsilon}}\norm{\nabla f(\*X_s)}^2+\frac{c_2^2\bar\sigma^2\sqrt{V+\epsilon}}{4\epsilon^3c_1^2}\norm{\*m_s}^2.
\]
Then, we have
\[
\begin{aligned}
\E\left[\overline{\norm{\nabla f(\*X_t)}^2}\right]\leq&\frac{2\sqrt{V+\epsilon}\left(\phi_2(\*X_0,\*m_0,\*v_0)-\min_t\E\qty[\phi_2(\*X_t,\*m_t,\*v_t)]\right)}{\int_0^t\eta(s)\dd s}\\
&+\qty(\frac{2L\sqrt{V+\epsilon}}{c_1\epsilon}+\qty(1+\frac{\bar\sigma^2}{\epsilon^2})\frac{c_2^2(V+\epsilon)}{2c_1^2\epsilon})\E\left[\overline{\norm{\*m_t}^2}\right].
\end{aligned}
\]
Note that $\bigg|\E\qty[\inner{\nabla f(\*X_t),(\*v_t+\epsilon)^{-\frac{1}{2}}\odot\*m_t}]\bigg|\leq\frac{\ell}{\sqrt{\epsilon}}\sqrt{\E\qty[\norm{\*m_t}^2]}\leq\frac{\ell M\sqrt{N}}{\sqrt{\epsilon}} $. Combined with \eqref{eq:adam_m_est}, we prove the bound for $\E\left[\overline{\norm{\nabla f(\*X_t)}^2}\right]$. This completes the proof.
\end{proof}

\subsection{\blue{Proofs for Escaping Probability}}\label{sec:supp:escape}
\subsubsection{Proof of Proposition \ref{prop:P-SGD}}
\begin{proof}
Note that $\nabla f(\*x^*)=0$, we have $\dv{\bar{\*x}_t}{t}=\*0$. So $\bar{\*x}_t$ remains constant at $\*x^*$, and $\nabla^2f(\bar{\*x}_t)\equiv\nabla^2f(\*x^*)$. Then, the ODE \eqref{eq:SGD-ODE-approx} comes from \eqref{eq:SDE-GA}. Let $\*p(t)=\myvec(\*P(t))$, $\*b=\myvec(\*\Sigma)$, where ${\rm vec}(\cdot)$ is the vectorization of a matrix by column order. Then \eqref{eq:SGD-ODE-approx} for $\*P(t)$ is equivalent to the following ODE for vector-valued function $\*p(t)$:
\[
\frac{\dd\*p(t)}{\dd t}=-\eta(t)\*Q\*p(t)+\eta_0\eta(t)^2\*b,
\]
where $\*Q=\*I\otimes\*H+\*H\otimes\*I$, where $\otimes$ is the Kronecker product. Then, we have
\[
\*p(t)=\exp\qty(-\int_0^t\eta(s)\*Q\dd s)\left(\int_0^t\exp\qty(\int_0^s\eta(\tau)\*Q\dd\tau)\eta_0\eta(s)^2\*b\dd s\right).
\]
Note that
\[
\begin{aligned}
&\exp\qty(\int_0^s\eta(\tau)\qty(\*I\otimes\*H+\*H\otimes\*I)\dd\tau){\rm vec}\qty(\*\Sigma)\\
=&\exp\qty(\*H\int_0^s\eta(\tau)\dd\tau)\otimes\exp\qty(\*H\int_0^s\eta(\tau)\dd\tau){\rm vec}\qty(\*\Sigma)\\
=&{\rm vec}\qty(\exp\qty(\*H\int_0^s\eta(\tau)\dd\tau)\*\Sigma\exp\qty(\*H\int_0^s\eta(\tau)\dd\tau)).
\end{aligned}
\]
Denote $\*A(t):=\exp\qty(-\int_0^t\eta(s)\*H\dd s)$. Then, it holds that $(\*A(t))^{-1}:=\exp\qty(\int_0^t\eta(s)\*H\dd s)$. Hence, we obtain:
\[
\begin{aligned}
\*p(t)=&\exp\qty(-\int_0^t\eta(s)(\*I\otimes\*H+\*H\otimes\*I)\dd s){\rm vec}\qty(((\*A(s))^{-1})\*\Sigma((\*A(s))^{-1})\eta_0\eta(s)^2\dd s)\\
=&\*A(t)\otimes\*A(t){\rm vec}\qty(\int_0^t((\*A(s))^{-1})\*\Sigma((\*A(s))^{-1})\eta_0\eta(s)^2\dd s)\\
=&{\rm vec}\qty(\*A(t)\qty(\int_0^t((\*A(s))^{-1})\*\Sigma((\*A(s))^{-1})\eta_0\eta(s)^2\dd s)\*A(t)).
\end{aligned}
\]
Consequently, the matrix $\*P(t)$ is expressed as:
\[
\*P(t)=\*A(t)\qty(\int_0^t((\*A(s))^{-1})\*\Sigma((\*A(s))^{-1})\eta_0\eta(s)^2\dd s)\*A(t).
\]
This completes the proof.
\end{proof}

\subsubsection{Proof of Proposition \ref{prop:P-Adam}}
We first recall the Gaussian approximations for a general SDE \citep{sarkka2019applied}. Considering the SDE with initial condition $\*x_{t_0}=\*x_0$:
\[
\dd\*x_t=\*G(\*x,t)\dd t+\*L(\*x,t)\dd\*W_t
\]
the linearization aproximation of its FPK equation yields the following differential equations for $\bar{\*x}(t)=\bb{E}[\*x_t]$ and $\*P(t)={\rm Cov}(\*x_t)$ with the initial condition $\bar{\*x}_{t_0}=\bb{E}[\*x_0]$, $\*P_{t_0}={\rm Cov}(\*x_{t_0})$:
\begin{equation}
\left\{
\begin{aligned}
\frac{\dd\bar{\*x}}{\dd t}&=\*G(\bar{\*x},t)\\
\frac{\dd\*P}{\dd t}&=\*P\*G_{\*x}^\top(\bar{\*x},t)+\*G_{\*x}(\bar{\*x},t)\*P+\*L(\bar{\*x},t)\*L(\bar{\*x},t)^\top.
\end{aligned}\right.
\label{eq:general-GA}
\tag{general-GA}
\end{equation}
\begin{proof}
Let $\*Z_0:=[\*X_0;\*m_0;\*v_0]=[\*x^*;\*0;{\rm diag}(\*\Sigma(\*x^*))]$. Then $\bar{\*z}_0=\*Z_0$ is a zero of $\*F(\*Z)$ defined in \eqref{eq:Adam-Z}. Therefore, the trajectory $\bar{\*z}_t$ remains at $\*Z_0$. Note that $\pdv{F(\*Z_0)}{\*Z}=\widehat{\*H}$. Then \eqref{eq:Adam-ODE-approx} can be derived by substituting $\bar{\*z}_t=\*Z_0$, \eqref{eq:Adam-Z} and \eqref{eq:dF-Adam-SDE} into \eqref{eq:general-GA}. The remainder of the proof follows similarly to that of Proposition \ref{prop:P-SGD}, and thus is omitted for brevity.
\end{proof}

\subsubsection{Anti-concentration for Gaussian}
Before proving the main results for the escape probability, we need a general anti-concentration inequality for Gaussian variables, as considered in other scenarios~\citep{carbery2001distributional,tu2023elementary}.
\begin{mylemma}
\label{le:anti}
    Assume $\*x \sim \mathcal{N}\qty(\bm{\mu},\*\Sigma)$, for any $\varepsilon\in(0,\Tr(\*\Sigma))$, we have
    \[
    \bb P\qty[\norm{\*x-\bm{\mu}}^2 \leq \varepsilon] \leq \sqrt{\frac{e \varepsilon}{\Tr (\*\Sigma)}}.
    \]
\end{mylemma}
\begin{proof}
    Applying Chernoff's bound, we have:
    \[
    \begin{aligned}
       & \bb P\qty[\norm{\*x-\bm{\mu}}^2\leq \varepsilon] \leq  \inf_{\lambda>0}\qty{\exp(\lambda \varepsilon)\int_{\mathbb{R}^N} \frac{1}{\sqrt{(2\pi)^N \det^*\qty(\*\Sigma)}}\exp(-\lambda\norm{\*x}^2 - \frac{1}{2}\*x^\top\pinv{\*\Sigma}\*x )\dd \*x}\\
       = &  \inf_{\lambda>0}\qty{\exp(\lambda \varepsilon)\int_{\mathbb{R}^N} \frac{1}{\sqrt{(2\pi)^N \det^*\qty(\*\Sigma)}}\exp( - \frac{1}{2}\*x^\top\qty(2\lambda\*I_N+\pinv{\*\Sigma})\*x )\dd \*x}\\
        = &  \inf_{\lambda>0}\qty{\exp(\lambda \varepsilon) \sqrt{\frac{(2\pi)^N\det\qty(\qty(2\lambda\*I_N+\pinv{\*\Sigma})^{-1})}{(2\pi)^N \det^*\qty(\*\Sigma)}}}
       =   \inf_{\lambda>0}\qty{\exp(\lambda \varepsilon) \frac{1}{\sqrt{\det \qty(2\lambda \*\Sigma + \*I_N)}}}\\
       \overset{(a)}{\leq} & \inf_{\lambda>0}\qty{\exp(\lambda \varepsilon) \frac{1}{\sqrt{1 + 2\lambda \Tr \qty(\*\Sigma)}}}
       \overset{(b)}{=}  \sqrt{\frac{\varepsilon}{\Tr(\*\Sigma)}}\exp(\frac{\Tr(\*\Sigma)-\varepsilon}{2\Tr(\*\Sigma)})
       \leq  \sqrt{\frac{e \varepsilon}{\Tr (\*\Sigma)}},
    \end{aligned}
    \]
    where $\pinv{\*\Sigma}$ is the Moore-Penrose pseudoinverse of $\*\Sigma$ and $\det^*(\cdot)$ is the pseudo-determinant. In the above, $(a)$ comes from
    \[
    \det \qty(2\lambda \*\Sigma + \*I_N) = \prod_{i=1}^N \qty(1+2\lambda \sigma_i) \geq 1+2\lambda\sum_{i=1}^N \sigma_i = 1 + 2\lambda \Tr \qty(\*\Sigma).
    \]
    where $\sigma_i$ is the $i$-th eigenvalue of $\*\Sigma$, and $(b)$ is optimized by setting $\lambda = \frac{\Tr (\*\Sigma)-\varepsilon}{2 \varepsilon \Tr (\*\Sigma)}$.
\end{proof}

\subsubsection{Proof of Theorem~\ref{thm:SGD-prob}}
\begin{proof}
Let $\*U$ be an orthogonal matrix such that $\*H=\*U\*\Lambda\*U^\top$, where $\*\Lambda$ is the diagonal matrix of eigenvalues of $\*H$. Let $\lambda_{\max}$ be the maximal eigenvalue of $\*H$. By equation \eqref{eq:SGD-P_t}, we have that
\[
\begin{aligned}
{\rm Tr}\qty(\*P(t))=&{\rm Tr}\qty(\int_0^t\exp\qty(-\*H\int_s^t\eta(\tau)\dd\tau)\*\Sigma\exp\qty(-\*H\int_s^t\eta(\tau)\dd\tau)\eta_0\eta(s)^2\dd s)\\
=&{\rm Tr}\qty(\int_0^t\exp\qty(-\*\Lambda\int_s^t\eta(\tau)\dd\tau)\*U^\top\*\Sigma\*U\exp\qty(-\*\Lambda\int_s^t\eta(\tau)\dd\tau)\eta_0\eta(s)^2\dd s)\\
\geq&\eta_0{\rm Tr}\qty(\*\Sigma)\int_0^t\exp\qty(-2\lambda_{\max}\int_s^t\eta(\tau)\dd\tau)\eta(s)^2\dd s\\
\geq&\eta_0{\rm Tr}\qty(\*\Sigma)\int_0^t\exp\qty(-2\lambda_{\max}\eta_{\max}(t-s))\eta(s)^2\dd s\\
\geq&C\eta_0{\rm Tr}\qty(\*\Sigma)\int_0^t\eta(s)^2\dd s,
\end{aligned}
\]
where $C:=\inf_{t\geq0}\frac{\int_0^t\exp\qty(-2\lambda_{\max}\eta_{\max}(t-s))\eta(s)^2\dd s}{\int_0^t\eta(s)^2\dd s}$. Next, we demonstrate that \( C > 0 \) by considering two cases. When $\int_0^\infty\eta(s)^2\dd s<\infty$, it is clear that $C>0$. When $\int_0^\infty\eta(s)^2\dd s=\infty$, by L'Hopital's rule, we have
\[
\lim_{t\rightarrow\infty}\frac{\int_0^t\exp\qty(-2\lambda_{\max}\eta_{\max}(t-s))\eta(s)^2\dd s}{\int_0^t\eta(s)^2\dd s}=\lim_{t\rightarrow\infty}\frac{\eta(t)^2}{\eta(t)^2}=1.
\]
Thus, $C>0$. The condition  $\eta(0)=\eta_{\max}$ and $\eta(T)=0$ leads to the following inequality derived from the Cauchy-Schwarz inequality:
\[
\begin{aligned}
&\left(\int_0^T\eta(s)^2\dd s\right)\left(\int_0^T\eta'(s)^2\dd s\right)\geq\left(\int_0^T\eta(s)\eta'(s)\dd s\right)^2.
\end{aligned}
\]
Integrating by parts yields:
\[
\begin{aligned}
\int_0^T\eta(s)\eta'(s)\dd s=\eta(s)^2\bigg|_0^T-\int_0^T\eta(s)\eta'(s)\dd s,
\end{aligned}
\]
which results in
\[
\int_0^T\eta(s)\eta'(s)\dd s=-\frac{\eta^2_{\max}}{2}.
\]
This result leads to the following inequality:
\[
\frac{\left(\frac{\eta^2_{\max}}{2}\right)^2}{\int_0^T\eta(s)^2\dd s}\leq\int_0^T\eta'(s)^2\dd s.
\]
Now, suppose Assumption \ref{asm:zero-mean} holds. For any fix $\delta\in(0,1)$, by Proposition \ref{prop:sigma_bound} and setting $t=\delta\Tr(\*\Sigma_g)$,  we have
\[
\begin{aligned}
&\mathbb{P}\left\{(1-\delta)\Tr(\*\Sigma_g)\leq\Tr(\*\Sigma(\*x))\leq(1+\delta)\Tr(\*\Sigma_g)\right\}\\
\geq&1-2\exp\left\{-\frac{D\delta^2(\Tr(\*\Sigma_g))^2}{4\Tr(\*\Sigma_g^2)+2\delta\Tr(\*\Sigma_g)\norm{\*\Sigma_g}_{\rm op}}\right\}.
\end{aligned}
\]
%Since $D\gg N$, we have
Denote the event $\mathcal{E}:=\{(1-\delta)\Tr(\*\Sigma_g)\leq\Tr(\*\Sigma(\*x))\leq(1+\delta)\Tr(\*\Sigma_g)\}$ and $\mathcal{A}:=\{\norm{\*X_T-\*x^*}^2 \leq \varepsilon\}$. By Proposition \ref{prop:P-SGD}, $\bb{E}[\*X_T]=\bar{\*x}_T=\*x^*$. Then, applying Lemma \ref{le:anti}, we get
\[
\begin{aligned}
&\bb P\left\{\norm{\*X_T-\*x^*}^2 \leq \varepsilon\big|\mathcal{E}\right\} \leq\sqrt{\frac{ e \varepsilon}{\Tr(\*P_t)}}
\leq\sqrt{\frac{e\varepsilon}{C\eta_0{\rm Tr}\qty(\*\Sigma)\qty(\frac{\eta_{\max}^2}{2})^2}\int_0^T\eta'(s)^2\dd s}\\
\leq&\sqrt{\frac{e\varepsilon}{C\eta_0(1-\delta){\rm Tr}\qty(\*\Sigma_g)\qty(\frac{\eta_{\max}^2}{2})^2}\int_0^T\eta'(s)^2\dd s}.
\end{aligned}
\]
Finally, we have
\[
\begin{aligned}
&\mathbb{P}\{\mathcal{A}\}=\mathbb{P}\{\mathcal{A}|\cE\}\bbP\{\cE\}+\bbP\{\cA|\cE^c\}\bbP\{\cE^c\}\leq\mathbb{P}\{\mathcal{A}|\cE\}+\bbP\{\cE^c\}\\
\leq&\sqrt{\frac{e\varepsilon}{C\eta_0(1-\delta){\rm Tr}\qty(\*\Sigma_g)\qty(\frac{\eta_{\max}^2}{2})^2}\int_0^T\eta'(s)^2\dd s}+2\exp\left\{-\frac{D\delta^2(\Tr(\*\Sigma_g))^2}{4\Tr(\*\Sigma_g^2)+2\delta\Tr(\*\Sigma_g)\norm{\*\Sigma_g}_{\rm op}}\right\}.
\end{aligned}
\]
Since $D\gg N$, the second term is sufficiently small compared with the first term. Finally, we derive the bound
\[
\bb P[\norm{\*X_T-\*x^*}^2 \leq \varepsilon ] = \order{\qty(\frac{\varepsilon  }{\eta^4_{\max} \sigma_g^2 }\int_0^T \eta^\prime(s)^2 \dd s)^{1/2}}.
\]
This completes the proof.
\end{proof}

\subsubsection{Proof of Theorem~\ref{thm:Adam-prob}}
\begin{proof}
% By Proposition \ref{prop:P-Adam}, we know that $\bb{E}[\*Z_T]=\bar{\*z}_T=[\*x^*;\*0;{\rm diag}\qty(\*\Sigma(\*x^*))]$.
The proof follows similarly to the proof of Theorem~\ref{thm:SGD-prob}, and hence we omit the details.
\end{proof}

\section{\blue{Limitations and Scope of Validity}}\label{sec:limitations}

\blue{The \abbr{} framework is developed for learning-rate schedules that start from zero, vary over time, and decay to zero. Below we describe the boundary conditions of the current formulation. First, for fixed learning-rate schedules the derivative terms $\eta'_t$ vanish, and the escape-related features reduce to zero. The law then contains only convergence and scale terms, which overlaps with the role of standard scaling laws. \abbr{} is therefore most informative for schedules with nontrivial warmup and cooldown phases. Second, the SDE escape analysis (Theorems~\ref{thm:SGD-prob} and~\ref{thm:Adam-prob}) assumes that the learning rate eventually decays to zero, which is standard in the SDE optimization literature. When the final learning rate is non-zero, the escape bounds require further extension. Generalizing the analysis to non-zero terminal LR is an interesting direction for future work. Third, the current empirical validation covers piecewise-linear, cosine, constant-plus-linear, and polygon schedule families. More complex schedule types such as cyclic schedules or warmup restarts have not been tested and remain a direction for future investigation.}

%% file: scaling.bib
@String(icml = "{International Conference on Machine Learning}")

@article{zhou2024towards,
  title={Towards understanding convergence and generalization of {AdamW}},
  author={Zhou, Pan and Xie, Xingyu and Lin, Zhouchen and Yan, Shuicheng},
  journal={IEEE Transactions on Pattern Analysis and Machine Intelligence},
  year={2024},
  publisher={IEEE}
}

@article{bi2024deepseek,
  title={{DeepSeek LLM}: Scaling open-source language models with longtermism},
  author={Bi, Xiao and Chen, Deli and Chen, Guanting and Chen, Shanhuang and Dai, Damai and Deng, Chengqi and others},
  journal={arXiv preprint arXiv:2401.02954},
  year={2024}
}

@article{deepseekai2024deepseekv2,
      title={{DeepSeek-V2}: A Strong, Economical, and Efficient Mixture-of-Experts Language Model}, 
      author={DeepSeek-AI and Aixin Liu and Bei Feng and Bin Wang and Bingxuan Wang and Bo Liu and others},
      year={2024},
      journal={arXiv preprint arXiv:2405.04434}
}

@inproceedings{gao2023scaling,
  title={Scaling laws for reward model overoptimization},
  author={Gao, Leo and Schulman, John and Hilton, Jacob},
  booktitle={International Conference on Machine Learning},
  pages={10835--10866},
  year={2023},
  organization={PMLR}
}

@inproceedings{malladi2022sdes,
  title={{On the SDEs and scaling rules for adaptive gradient algorithms}},
  author={Malladi, Sadhika and Lyu, Kaifeng and Panigrahi, Abhishek and Arora, Sanjeev},
  booktitle={Advances in Neural Information Processing Systems},
  year={2022}
}

@inproceedings{zhu2019anisotropic,
  title={The Anisotropic Noise in Stochastic Gradient Descent: Its Behavior of Escaping from Sharp Minima and Regularization Effects},
  author={Zhu, Zhanxing and Wu, Jingfeng and Yu, Bing and Wu, Lei and Ma, Jinwen},
  booktitle={International Conference on Machine Learning},
  pages={7654--7663},
  year={2019},
  organization={PMLR}
}

@inproceedings{xie2020diffusion,
  title={A Diffusion Theory For Deep Learning Dynamics: Stochastic Gradient Descent Exponentially Favors Flat Minima},
  author={Xie, Zeke and Sato, Issei and Sugiyama, Masashi},
  booktitle={International Conference on Learning Representations},
  year={2020}
}

@software{together2023redpajama,
  author = {Together Computer},
  title = {{RedPajama}: An Open Dataset for Training Large Language Models},
  year = 2023,
  url = {https://github.com/togethercomputer/RedPajama-Data}
}

@article{zhao2024longskywork,
  title={{LongSkywork}: A Training Recipe for Efficiently Extending Context Length in Large Language Models},
  author={Zhao, Liang and Wei, Tianwen and Zeng, Liang and Cheng, Cheng and Yang, Liu and Cheng, Peng and others},
  journal={arXiv preprint arXiv:2406.00605},
  year={2024}
}

@article{wei2024skywork,
  title={{Skywork-MoE}: A Deep Dive into Training Techniques for Mixture-of-Experts Language Models},
  author={Wei, Tianwen and Zhu, Bo and Zhao, Liang and Cheng, Cheng and Li, Biye and L{\"u}, Weiwei and others},
  journal={arXiv preprint arXiv:2406.06563},
  year={2024}
}

@inproceedings{gupta2023continual,
  title={Continual Pre-Training of Large Language Models: How to re-warm your model?},
  author={Gupta, Kshitij and Th{\'e}rien, Benjamin and Ibrahim, Adam and Richter, Mats Leon and Anthony, Quentin Gregory and Belilovsky, Eugene and others},
  booktitle={Workshop on Efficient Systems for Foundation Models@ ICML},
  year={2023}
}

@article{ibrahim2024simple,
  title={Simple and scalable strategies to continually pre-train large language models},
  author={Ibrahim, Adam and Th{\'e}rien, Benjamin and Gupta, Kshitij and Richter, Mats L and Anthony, Quentin and Lesort, Timoth{\'e}e and others},
  journal={arXiv preprint arXiv:2403.08763},
  year={2024}
}

@article{hagele2024scaling,
  title={Scaling Laws and Compute-Optimal Training Beyond Fixed Training Durations},
  author={H{\"a}gele, Alexander and Bakouch, Elie and Kosson, Atli and Allal, Loubna Ben and Von Werra, Leandro and Jaggi, Martin},
  journal={arXiv preprint arXiv:2405.18392},
  year={2024}
}

@article{lv2023full,
  title={Full parameter fine-tuning for large language models with limited resources},
  author={Lv, Kai and Yang, Yuqing and Liu, Tengxiao and Gao, Qinghui and Guo, Qipeng and Qiu, Xipeng},
  journal={arXiv preprint arXiv:2306.09782},
  year={2023}
}

@article{jin2023rethinking,
  title={Rethinking learning rate tuning in the era of large language models},
  author={Jin, Hongpeng and Wei, Wenqi and Wang, Xuyu and Zhang, Wenbin and Wu, Yanzhao},
  journal={arXiv preprint arXiv:2309.08859},
  year={2023}
}

@article{besiroglu2024chinchilla,
  title={{Chinchilla Scaling}: A replication attempt},
  author={Besiroglu, Tamay and Erdil, Ege and Barnett, Matthew and You, Josh},
  journal={arXiv preprint arXiv:2404.10102},
  year={2024}
}

@article{hoffmann2022training,
  title={Training compute-optimal large language models},
  author={Hoffmann, Jordan and Borgeaud, Sebastian and Mensch, Arthur and Buchatskaya, Elena and Cai, Trevor and Rutherford, Eliza and others},
  journal={arXiv preprint arXiv:2203.15556},
  year={2022}
}

@article{kaplan2020scaling,
  title={Scaling laws for neural language models},
  author={Kaplan, Jared and McCandlish, Sam and Henighan, Tom and Brown, Tom B and Chess, Benjamin and Child, Rewon and others},
  journal={arXiv preprint arXiv:2001.08361},
  year={2020}
}

@article{hu2024minicpm,
  title={MiniCPM: Unveiling the Potential of Small Language Models with Scalable Training Strategies},
  author={Hu, Shengding and Tu, Yuge and Han, Xu and He, Chaoqun and Cui, Ganqu and Long, Xiang and others},
  journal={arXiv preprint arXiv:2404.06395},
  year={2024}
}

@book{bai2010spectral,
  title={Spectral analysis of large dimensional random matrices},
  author={Bai, Zhidong and Silverstein, Jack W},
  volume={20},
  year={2010},
  publisher={Springer}
}

@article{isik2024scaling,
  title={Scaling Laws for Downstream Task Performance of Large Language Models},
  author={Isik, Berivan and Ponomareva, Natalia and Hazimeh, Hussein and Paparas, Dimitris and Vassilvitskii, Sergei and Koyejo, Sanmi},
  journal={arXiv preprint arXiv:2402.04177},
  year={2024}
}

@article{ledoux2010small,
  title={Small deviations for beta ensembles},
  author={Ledoux, Michel and Rider, Brian},
  journal={Electronic Journal of Probability},
  volume={15},
  number={41},
  pages={1319--1343},
  year={2010}
}

@article{berglund2013sharp,
  title={Sharp estimates for metastable lifetimes in parabolic {SPDEs}: Kramers’ law and beyond},
  author={Berglund, Nils and Gentz, Barbara},
  journal={Electronic Journal of Probability},
  volume={18},
  number={24},
  pages={1--58},
  year={2013}
}

@inproceedings{li2017stochastic,
  title={Stochastic modified equations and adaptive stochastic gradient algorithms},
  author={Li, Qianxiao and Tai, Cheng and Weinan, E},
  booktitle={International Conference on Machine Learning},
  pages={2101--2110},
  year={2017},
  organization={PMLR}
}

@article{bovier2015metastability,
 title={Metastability: a potential-theoretic approach},
  author={Bovier, Anton and Den Hollander, Frank},
  journal={Grundlehren der mathematischen Wissenschaften},
  volume={351},
  year={2015},
  publisher={Springer}
}

@article{dambrine2024stochastic,
  title={Stochastic Differential Equations for modeling first order optimization methods},
  author={Dambrine, Marc and Dossal, Ch and Puig, B{\'e}n{\'e}dicte and Rondepierre, Aude},
  journal={SIAM Journal on Optimization},
  volume={34},
  number={2},
  pages={1402--1426},
  year={2024},
  publisher={SIAM}
}

@article{soto2022sde,
  title={An {SDE} perspective on stochastic convex optimization},
  author={Soto, Rodrigo Maulen and Fadili, Jalal and Attouch, Hedy},
  journal={arXiv preprint arXiv:2207.02750},
  year={2022}
}

@article{kingma2014adam,
  title={Adam: A method for stochastic optimization},
  author={Kingma, Diederik P and Ba, Jimmy},
  journal={arXiv preprint arXiv:1412.6980},
  year={2014}
}

@inproceedings{simsekli2019tail,
  title={A tail-index analysis of stochastic gradient noise in deep neural networks},
  author={Simsekli, Umut and Sagun, Levent and Gurbuzbalaban, Mert},
  booktitle={International Conference on Machine Learning},
  pages={5827--5837},
  year={2019},
  organization={PMLR}
}

@inproceedings{haochen2021shape,
  title={Shape matters: Understanding the implicit bias of the noise covariance},
  author={Zhang, Haochen and Wei, Colin and Lee, Jason and Ma, Tengyu},
  booktitle={Conference on Learning Theory},
  pages={2315--2357},
  year={2021},
  organization={PMLR}
}

@article{zhou2022win,
  title={Win: Weight-Decay-Integrated {Nesterov} Acceleration for Faster Network Training},
  author={Zhou, Pan and Xie, Xingyu and Lin, Zhouchen and Toh, Kim-Chuan and Yan, Shuicheng},
  journal={Journal of Machine Learning Research},
  volume={25},
  number={83},
  pages={1--74},
  year={2024}
}

@article{xie2024adan,
  title={Adan: Adaptive {Nesterov} momentum algorithm for faster optimizing deep models},
  author={Xie, Xingyu and Zhou, Pan and Li, Huan and Lin, Zhouchen and Yan, Shuicheng},
  journal={IEEE Transactions on Pattern Analysis and Machine Intelligence},
  year={2024},
  publisher={IEEE}
}

@article{arjevani2019lower,
  title={Lower bounds for non-convex stochastic optimization},
  author={Arjevani, Yossi and Carmon, Yair and Duchi, John C and Foster, Dylan J and Srebro, Nathan and Woodworth, Blake},
  journal={Mathematical Programming},
  pages={1--50},
  year={2022},
  publisher={Springer}
}

@article{guo2021novel,
  title={A Novel Convergence Analysis for Algorithms of the {Adam} Family},
  author={Guo, Zhishuai and Xu, Yi and Yin, Wotao and Jin, Rong and Yang, Tianbao},
  journal={arXiv preprint arXiv:2112.03459},
  year={2021}
}

@inproceedings{li2022restarted,
  title={Restarted Nonconvex Accelerated Gradient Descent: No More Polylogarithmic Factor in the $ \order{\epsilon^{-7/4}}$ Complexity},
  author={Li, Huan and Lin, Zhouchen},
  booktitle={International Conference on Machine Learning},
  pages={12901--12916},
  year={2022},
  organization={PMLR}
}

@article{sarkka2013gaussian,
  title={Gaussian filtering and smoothing for continuous-discrete dynamic systems},
  author={S{\"a}rkk{\"a}, Simo and Sarmavuori, Juha},
  journal={Signal Processing},
  volume={93},
  number={2},
  pages={500--510},
  year={2013},
  publisher={Elsevier}
}

@article{sarkka2015posterior,
  title={Posterior inference on parameters of stochastic differential equations via non-linear Gaussian filtering and adaptive {MCMC}},
  author={S{\"a}rkk{\"a}, Simo and Hartikainen, Jouni and Mbalawata, Isambi Sailon and Haario, Heikki},
  journal={Statistics and Computing},
  volume={25},
  number={2},
  pages={427--437},
  year={2015},
  publisher={Springer}
}

@article{solin2021scalable,
  title={Scalable inference in {SDEs} by direct matching of the {Fokker--Planck--Kolmogorov} equation},
  author={Solin, Arno and Tamir, Ella and Verma, Prakhar},
  journal={Advances in Neural Information Processing Systems},
  volume={34},
  pages={417--429},
  year={2021}
}

@article{archambeau2007gaussian,
  title={Gaussian process approximations of stochastic differential equations},
  author={Archambeau, C{\'e}dric and Cornford, Dan and Opper, Manfred and Shawe-Taylor, John},
  journal={Journal of Machine Learning Research},
  volume={1},
  pages={1--16},
  year={2007}
}

@book{sarkka2019applied,
  title={{Applied Stochastic Differential Equations}},
  author={S{\"a}rkk{\"a}, Simo and Solin, Arno},
  volume={10},
  year={2019},
  publisher={Cambridge University Press}
}

@article{carbery2001distributional,
  title={{Distributional and $L^q$ norm inequalities for polynomials over convex bodies in $\mathbb{R}^n$}},
  author={Carbery, Anthony and Wright, James},
  journal={Mathematical Research Letters},
  volume={8},
  number={3},
  pages={233--248},
  year={2001},
  publisher={International Press of Boston}
}

@article{tu2023elementary,
  title={An elementary proof of anti-concentration for degree two non-negative Gaussian polynomials},
  author={Tu, Stephen and Boczar, Ross},
  journal={arXiv preprint arXiv:2301.05992},
  year={2023}
}

@book{gihman1972stochastic,
author="Gihman, Iosif Il'ich
and Skorohod, Anatolii Vladimirovich",
title="Stochastic Differential Equations",
bookTitle="The Theory of Stochastic Processes III",
year="1979",
publisher="Springer New York",
pages="113--219",
}

@book{kloeden1999stochastic,
  title={Numerical Solution to Stochastic Differential Equations},
  author={Kloeden, Peter E and Platen, Eckhard},
  year={1999},
  publisher={Springer}
}

@article{li2021validity,
  title={On the validity of modeling {SGD} with stochastic differential equations ({SDEs})},
  author={Li, Zhiyuan and Malladi, Sadhika and Arora, Sanjeev},
  journal={Advances in Neural Information Processing Systems},
  volume={34},
  pages={12712--12725},
  year={2021}
}

@article{li2019stochastic,
  title={Stochastic modified equations and dynamics of stochastic gradient algorithms i: Mathematical foundations},
  author={Li, Qianxiao and Tai, Cheng and Weinan, E},
  journal={Journal of Machine Learning Research},
  volume={20},
  number={40},
  pages={1--47},
  year={2019}
}

@inproceedings{reddi2018convergence,
  title={On the Convergence of Adam and Beyond},
  author={Reddi, Sashank J and Kale, Satyen and Kumar, Sanjiv},
  booktitle={International Conference on Learning Representations},
  year={2018}
}

@article{bertsekas2000gradient,
  title={Gradient convergence in gradient methods with errors},
  author={Bertsekas, Dimitri P and Tsitsiklis, John N},
  journal={SIAM Journal on Optimization},
  volume={10},
  number={3},
  pages={627--642},
  year={2000},
  publisher={SIAM}
}

@article{xiao2024adam,
  title={Adam-family methods for nonsmooth optimization with convergence guarantees},
  author={Xiao, Nachuan and Hu, Xiaoyin and Liu, Xin and Toh, Kim-Chuan},
  journal={Journal of Machine Learning Research},
  volume={25},
  number={48},
  pages={1--53},
  year={2024}
}

@article{davis2020stochastic,
	title={Stochastic subgradient method converges on tame functions},
	author={Davis, Damek and Drusvyatskiy, Dmitriy and Kakade, Sham and Lee, Jason D},
	journal={Foundations of Computational Mathematics},
	volume={20},
	number={1},
	pages={119--154},
	year={2020},
	publisher={Springer}
}

@inproceedings{khromov2024some,
  title={Some Fundamental Aspects about {Lipschitz} Continuity of Neural Networks},
  author={Khromov, Grigory and Singh, Sidak Pal},
  booktitle={The Twelfth International Conference on Learning Representations},
  year={2024}
}

@inproceedings{kim2021lipschitz,
  title={The {Lipschitz} constant of self-attention},
  author={Kim, Hyunjik and Papamakarios, George and Mnih, Andriy},
  booktitle={International Conference on Machine Learning},
  pages={5562--5571},
  year={2021},
  organization={PMLR}
}

@article{dubey2024llama,
  title={The {LLaMA} 3 Herd of Models},
  author={Dubey, Abhimanyu and Jauhri, Abhinav and Pandey, Abhinav and Kadian, Abhishek and Al-Dahle, Ahmad and Letman, Aiesha and others},
  journal={arXiv preprint arXiv:2407.21783},
  year={2024}
}

@article{grimmer2024accelerated,
  title={Accelerated Objective Gap and Gradient Norm Convergence for Gradient Descent via Long Steps},
  author={Grimmer, Benjamin and Shu, Kevin and Wang, Alex},
  journal={arXiv preprint arXiv:2403.14045},
  year={2024}
}

@article{ding2023adam,
  title={Adam-family Methods with Decoupled Weight Decay in Deep Learning},
  author={Ding, Kuangyu and Xiao, Nachuan and Toh, Kim-Chuan},
  journal={arXiv preprint arXiv:2310.08858},
  year={2023}
}

@article{xiao2023convergence,
  title={Convergence guarantees for stochastic subgradient methods in nonsmooth nonconvex optimization},
  author={Xiao, Nachuan and Hu, Xiaoyin and Toh, Kim-Chuan},
  journal={arXiv preprint arXiv:2307.10053},
  year={2023}
}

@article{shen2022single,
  title={A single-timescale analysis for stochastic approximation with multiple coupled sequences},
  author={Shen, Han and Chen, Tianyi},
  journal={Advances in Neural Information Processing Systems},
  volume={35},
  pages={17415--17429},
  year={2022}
}

@article{ke2023continual,
  title={Continual pre-training of language models},
  author={Ke, Zixuan and Shao, Yijia and Lin, Haowei and Konishi, Tatsuya and Kim, Gyuhak and Liu, Bing},
  journal={arXiv preprint arXiv:2302.03241},
  year={2023}
}

@article{guo2024efficient,
  title={Efficient Continual Pre-training by Mitigating the Stability Gap},
  author={Guo, Yiduo and Fu, Jie and Zhang, Huishuai and Zhao, Dongyan and Shen, Yikang},
  journal={arXiv preprint arXiv:2406.14833},
  year={2024}
}

@article{parmar2024reuse,
  title={Reuse, Don't Retrain: A Recipe for Continued Pretraining of Language Models},
  author={Parmar, Jupinder and Satheesh, Sanjev and Patwary, Mostofa and Shoeybi, Mohammad and Catanzaro, Bryan},
  journal={arXiv preprint arXiv:2407.07263},
  year={2024}
}

@article{jastrzkebski2017three,
  title={Three factors influencing minima in {SGD}},
  author={Jastrz{\k{e}}bski, Stanis{\l}aw and Kenton, Zachary and Arpit, Devansh and Ballas, Nicolas and Fischer, Asja and others},
  journal={arXiv preprint arXiv:1711.04623},
  year={2017}
}

@article{he2019control,
  title={Control batch size and learning rate to generalize well: Theoretical and empirical evidence},
  author={He, Fengxiang and Liu, Tongliang and Tao, Dacheng},
  journal={Advances in Neural Information Processing Systems},
  volume={32},
  year={2019}
}

@article{achiam2023gpt,
  title={{GPT-4} technical report},
  author={Achiam, Josh and Adler, Steven and Agarwal, Sandhini and Ahmad, Lama and Akkaya, Ilge and Aleman, Florencia Leoni and others},
  journal={arXiv preprint arXiv:2303.08774},
  year={2023}
}

@book{polyak1987introduction,
	author = {Polyak, B.T.},
	title = {Introduction to {Optimization}},
	publisher = {Optimization Software},
	address = {New York},
	year = {1987}
}

@inproceedings{nesterov1983method,
  title={{A method for solving the convex programming problem with convergence rate $O(1/k^2)$}},
  author={Nesterov, Yurii},
  booktitle={Dokl akad nauk Sssr},
  volume={269},
  pages={543},
  year={1983}
}

@article{su2014differential,
  title={A Differential Equation for Modeling {Nesterov’s} Accelerated Gradient Method: Theory and Insights},
  author={Su, Weijie and Boyd, Stephen and Candes, Emmanuel},
  journal={Advances in Neural Information Processing Systems},
  volume={27},
  year={2014}
}

@article{attouch2018fast,
  title={Fast convergence of inertial dynamics and algorithms with asymptotic vanishing viscosity},
  author={Attouch, Hedy and Chbani, Zaki and Peypouquet, Juan and Redont, Patrick},
  journal={Mathematical Programming},
  volume={168},
  pages={123--175},
  year={2018},
  publisher={Springer}
}

@article{may2017asymptotic,
  title={Asymptotic for a second-order evolution equation with convex potential andvanishing damping term},
  author={May, Ramzi},
  journal={Turkish Journal of Mathematics},
  volume={41},
  number={3},
  pages={681--685},
  year={2017}
}

@article{attouch2024fast,
  title={Fast convex optimization via a third-order in time evolution equation: TOGES-V an improved version of TOGES},
  author={Attouch, Hedy and Chbani, Zaki and Riahi, Hassan},
  journal={Optimization},
  volume={73},
  number={3},
  pages={575--595},
  year={2024},
  publisher={Taylor \& Francis}
}

@article{duchi2018stochastic,
  title={Stochastic methods for composite and weakly convex optimization problems},
  author={Duchi, John C and Ruan, Feng},
  journal={SIAM Journal on Optimization},
  volume={28},
  number={4},
  pages={3229--3259},
  year={2018},
  publisher={SIAM}
}

@article{ding2024stochastic,
  title={{Stochastic Bregman Subgradient Methods for Nonsmooth Nonconvex Optimization Problems}},
  author={Ding, Kuangyu and Toh, Kim-Chuan},
  journal={arXiv preprint arXiv:2404.17386},
  year={2024}
}

@article{maulen2022sde,
  title={{An SDE perspective on stochastic convex optimization}},
  author={Maulen-Soto, Rodrigo and Fadili, Jalal and Attouch, Hedy},
  journal={arXiv preprint arXiv:2207.02750},
  year={2022}
}

@article{maulen2024sde,
  title={{An SDE perspective on stochastic inertial gradient dynamics with time-dependent viscosity and geometric damping}},
  author={Maulen-Soto, Rodrigo and Fadili, Jalal and Attouch, Hedy and Ochs, Peter},
  journal={arXiv preprint arXiv:2407.04562},
  year={2024}
}

@inproceedings{muehlebach2019dynamical,
  title={A dynamical systems perspective on {Nesterov} acceleration},
  author={Muehlebach, Michael and Jordan, Michael},
  booktitle={International Conference on Machine Learning},
  pages={4656--4662},
  year={2019},
  organization={PMLR}
}

@inproceedings{ibayashi2023does,
  title={Why does SGD prefer flat minima?: Through the lens of dynamical systems},
  author={Ibayashi, Hikaru and Imaizumi, Masaaki},
  booktitle={When Machine Learning meets Dynamical Systems: Theory and Applications},
  year={2023}
}

@article{keskar2016large,
  title={On large-batch training for deep learning: Generalization gap and sharp minima},
  author={Keskar, Nitish Shirish and Mudigere, Dheevatsa and Nocedal, Jorge and Smelyanskiy, Mikhail and Tang, Ping Tak Peter},
  journal={arXiv preprint arXiv:1609.04836},
  year={2016}
}

@article{nguyen2019first,
  title={First exit time analysis of stochastic gradient descent under heavy-tailed gradient noise},
  author={Nguyen, Thanh Huy and Simsekli, Umut and Gurbuzbalaban, Mert and Richard, Ga{\"e}l},
  journal={Advances in Neural Information Processing Systems},
  volume={32},
  year={2019}
}

@article{kramers1940brownian,
  title={Brownian motion in a field of force and the diffusion model of chemical reactions},
  author={Kramers, Hendrik Anthony},
  journal={Physical},
  volume={7},
  number={4},
  pages={284--304},
  year={1940},
  publisher={Elsevier}
}

@inproceedings{mori2022power,
  title={Power-law escape rate of {SGD}},
  author={Mori, Takashi and Ziyin, Liu and Liu, Kangqiao and Ueda, Masahito},
  booktitle={International Conference on Machine Learning},
  pages={15959--15975},
  year={2022},
  organization={PMLR}
}

@article{berglund2013kramers,
  title={Kramers' law: Validity, derivations and generalisations},
  author={Berglund, Nils},
  journal={Markov Processes And Related Fields},
  volume={19},
  number={3},
  pages={459--490},
  year={2013}
}

@inproceedings{battash2024revisiting,
  title={Revisiting the Noise Model of Stochastic Gradient Descent},
  author={Battash, Barak and Wolf, Lior and Lindenbaum, Ofir},
  booktitle={International Conference on Artificial Intelligence and Statistics},
  pages={4780--4788},
  year={2024},
  organization={PMLR}
}

@article{foret2020sharpness,
  title={Sharpness-aware minimization for efficiently improving generalization},
  author={Foret, Pierre and Kleiner, Ariel and Mobahi, Hossein and Neyshabur, Behnam},
  journal={arXiv preprint arXiv:2010.01412},
  year={2020}
}

@article{shrivastava2023repofusion,
  title={RepoFusion: Training Code Models to Understand Your Repository},
  author={Shrivastava, Disha and Kocetkov, Denis and de Vries, Harm and Bahdanau, Dzmitry and Scholak, Torsten},
  journal={arXiv preprint arXiv:2306.10998},
  year={2023}
}

@inproceedings{rosenfeld2020a,
title={A Constructive Prediction of the Generalization Error Across Scales},
author={Jonathan S. Rosenfeld and Amir Rosenfeld and Yonatan Belinkov and Nir Shavit},
booktitle={International Conference on Learning Representations},
year={2020},
}

@article{hernandez2021scaling,
  title={Scaling laws for transfer},
  author={Hernandez, Danny and Kaplan, Jared and Henighan, Tom and McCandlish, Sam},
  journal={arXiv preprint arXiv:2102.01293},
  year={2021}
}

@inproceedings{
caballero2023broken,
title={Broken Neural Scaling Laws},
author={Ethan Caballero and Kshitij Gupta and Irina Rish and David Krueger},
booktitle={The Eleventh International Conference on Learning Representations },
year={2023},
}

@inproceedings{sardana2024beyond,
title={Beyond {Chinchilla}-Optimal: Accounting for Inference in Language Model Scaling Laws},
author={Nikhil Sardana and Jacob Portes and Sasha Doubov and Jonathan Frankle},
booktitle={International Conference on Machine Learning},
year={2024},
}

@inproceedings{goyal2024scaling,
  title={Scaling Laws for Data Filtering--Data Curation cannot be Compute Agnostic},
  author={Goyal, Sachin and Maini, Pratyush and Lipton, Zachary C and Raghunathan, Aditi and Kolter, J Zico},
  booktitle={Proceedings of the IEEE/CVF Conference on Computer Vision and Pattern Recognition},
  pages={22702--22711},
  year={2024}
}

@article{muennighoff2024scaling,
  title={Scaling data-constrained language models},
  author={Muennighoff, Niklas and Rush, Alexander and Barak, Boaz and Le Scao, Teven and Tazi, Nouamane and Piktus, Aleksandra and others},
  journal={Advances in Neural Information Processing Systems},
  volume={36},
  year={2024}
}

@article{xie2024loco,
  title={{LoCo}: Low-Bit Communication Adaptor for Large-scale Model Training},
  author={Xie, Xingyu and Lin, Zhijie and Toh, Kim-Chuan and Zhou, Pan},
  journal={arXiv preprint arXiv:2407.04480},
  year={2024}
}

@article{rotaru2024exact,
  title={Exact worst-case convergence rates of gradient descent: a complete analysis for all constant stepsizes over nonconvex and convex functions},
  author={Rotaru, Teodor and Glineur, Fran{\c{c}}ois and Patrinos, Panagiotis},
  journal={arXiv preprint arXiv:2406.17506},
  year={2024}
}

@article{reid2024gemini,
  title={Gemini 1.5: Unlocking multimodal understanding across millions of tokens of context},
  author={Reid, Machel and Savinov, Nikolay and Teplyashin, Denis and Lepikhin, Dmitry and Lillicrap, Timothy and Alayrac, Jean-baptiste and others},
  journal={arXiv preprint arXiv:2403.05530},
  year={2024}
}

@article{loshchilov2017decoupled,
  title={Decoupled weight decay regularization},
  author={Loshchilov, Ilya and Hutter, Frank},
  journal={arXiv preprint arXiv:1711.05101},
  year={2017}
}

@article{yang2024qwen2,
  title={Qwen2 technical report},
  author={Yang, An and Yang, Baosong and Hui, Binyuan and Zheng, Bo and Yu, Bowen and Zhou, Chang and others},
  journal={arXiv preprint arXiv:2407.10671},
  year={2024}
}

@article{bai2023qwen,
  title={Qwen technical report},
  author={Bai, Jinze and Bai, Shuai and Chu, Yunfei and Cui, Zeyu and Dang, Kai and Deng, Xiaodong and others},
  journal={arXiv preprint arXiv:2309.16609},
  year={2023}
}

@article{gess2023convergence,
  title={Convergence rates for momentum stochastic gradient descent with noise of machine learning type},
  author={Gess, Benjamin and Kassing, Sebastian},
  journal={arXiv preprint arXiv:2302.03550},
  year={2023}
}

@article{xie2022optimization,
  title={Optimization induced equilibrium networks: An explicit optimization perspective for understanding equilibrium models},
  author={Xie, Xingyu and Wang, Qiuhao and Ling, Zenan and Li, Xia and Liu, Guangcan and Lin, Zhouchen},
  journal={IEEE Transactions on Pattern Analysis and Machine Intelligence},
  volume={45},
  number={3},
  pages={3604--3616},
  year={2022},
  publisher={IEEE}
}

@inproceedings{karimi2016linear,
  title={Linear convergence of gradient and proximal-gradient methods under the polyak-{\l}ojasiewicz condition},
  author={Karimi, Hamed and Nutini, Julie and Schmidt, Mark},
  booktitle={Joint European conference on machine learning and knowledge discovery in databases},
  pages={795--811},
  year={2016},
  organization={Springer}
}

@article{lee2019first,
  title={First-order methods almost always avoid strict saddle points},
  author={Lee, Jason D and Panageas, Ioannis and Piliouras, Georgios and Simchowitz, Max and Jordan, Michael I and Recht, Benjamin},
  journal={Mathematical programming},
  volume={176},
  pages={311--337},
  year={2019},
  publisher={Springer}
}

@article{ghadimi2013stochastic,
  title={Stochastic first-and zeroth-order methods for nonconvex stochastic programming},
  author={Ghadimi, Saeed and Lan, Guanghui},
  journal={SIAM journal on optimization},
  volume={23},
  number={4},
  pages={2341--2368},
  year={2013},
  publisher={SIAM}
}

@article{zhang2022adam,
  title={Adam can converge without any modification on update rules},
  author={Zhang, Yushun and Chen, Congliang and Shi, Naichen and Sun, Ruoyu and Luo, Zhi-Quan},
  journal={Neural Information Processing Systems},
  volume={35},
  pages={28386--28399},
  year={2022}
}

@inproceedings{
tissue2025scaling,
title={Scaling Law with Learning Rate Annealing},
author={Howe Tissue and Venus Wang and Lu Wang},
booktitle={The Thirteenth International Conference on Learning Representations},
year={2025}
}

@inproceedings{luo2025multipower,
title={A Multi-Power Law for Loss Curve Prediction Across Learning Rate Schedules},
author={Kairong Luo and Haodong Wen and Shengding Hu and Zhenbo Sun and Maosong Sun and Zhiyuan Liu and Kaifeng Lyu and Wenguang Chen},
booktitle={The Thirteenth International Conference on Learning Representations},
year={2025}
}

@InProceedings{li2021expcos,
  title = 	 {A Second look at Exponential and Cosine Step Sizes: Simplicity, Adaptivity, and Performance},
  author =       {Li, Xiaoyu and Zhuang, Zhenxun and Orabona, Francesco},
  booktitle = 	 {International Conference on Machine Learning},
  pages = 	 {6553--6564},
  year = 	 {2021},
  volume = 	 {139},
  series = 	 {Proceedings of Machine Learning Research},
  publisher =    {PMLR}
  }

@article{wang2021stepdecay,
  title={On the convergence of step decay step-size for stochastic optimization},
  author={Wang, Xiaoyu and Magn{\'u}sson, Sindri and Johansson, Mikael},
  journal={Advances in Neural Information Processing Systems},
  volume={34},
  pages={14226--14238},
  year={2021}
}

@inproceedings{
bergsma2025straightzero,
title={Straight to Zero: Why Linearly Decaying the Learning Rate to Zero Works Best for {LLM}s},
author={Shane Bergsma and Nolan Simran Dey and Gurpreet Gosal and Gavia Gray and Daria Soboleva and Joel Hestness},
booktitle={The Thirteenth International Conference on Learning Representations},
year={2025}
}

@inproceedings{
chen2026exploring,
title={Exploring the Basin-Like Loss Landscape in Large Language Models},
author={Huanran Chen and Yinpeng Dong and Zeming Wei and Yao Huang and Yichi Zhang and Jun Zhu},
booktitle={The Fourteenth International Conference on Learning Representations},
year={2026}
}

@article{liu2022loss,
  title={Loss landscapes and optimization in over-parameterized non-linear systems and neural networks},
  author={Liu, Chaoyue and Zhu, Libin and Belkin, Mikhail},
  journal={Applied and Computational Harmonic Analysis},
  volume={59},
  pages={85--116},
  year={2022},
  publisher={Elsevier}
}

@inproceedings{
chen2018on,
title={On the Convergence of A Class of Adam-Type Algorithms  for Non-Convex Optimization},
author={Xiangyi Chen and Sijia Liu and Ruoyu Sun and Mingyi Hong},
booktitle={International Conference on Learning Representations},
year={2019}
}

@article{xie2026slow,
  title={Slow-Fast Inference: Training-Free Inference Acceleration via Within-Sentence Support Stability},
  author={Xie, Xingyu and Yu, Zhaochen and Liao, Yue and Wang, Tao and Toh, Kim-Chuan and Yan, Shuicheng},
  journal={arXiv preprint arXiv:2603.12038},
  year={2026}
}
